\documentclass[10pt,twocolumn,letterpaper]{article}

\PassOptionsToPackage{table}{xcolor}

\usepackage[pagenumbers]{cvpr} 

\usepackage{algorithm}
\usepackage{algorithmic}
\usepackage{upgreek}
\usepackage{comment}
\usepackage{placeins}
\usepackage{threeparttable}
\usepackage{tikz}
\usepackage{pgfplots}
\pgfplotsset{compat=1.17}
\usetikzlibrary{pgfplots.groupplots}
\usepackage{makecell}
\usepackage{mathtools}
\usepackage{bm}
\IfFileExists{dsfont.sty}{\usepackage{dsfont}}{}
\usepackage{multirow}
\usepackage{amsthm}
\numberwithin{equation}{section}
\usepackage{newfloat}
\usepackage{listings}
\usepackage{microtype}

\renewcommand{\paragraph}[1]{\vspace{.5em}\noindent\textbf{#1}}
\newcommand{\ms}[2]{#1{\scriptsize$\pm$#2}}
\newcommand{\msbest}[2]{\textbf{#1}{\scriptsize$\pm$#2}}

\definecolor{augCutout}{HTML}{4E9A57}   
\definecolor{augAutoAug}{HTML}{3B6BA5}  
\definecolor{augRandAug}{HTML}{C00000}  

\definecolor{darkblue}{RGB}{0, 0, 139}

\theoremstyle{definition}

\newtheorem{theorem}{Theorem}
\newtheorem{lemma}{Lemma}
\newtheorem{assumption}{Assumption}
\newtheorem{corollary}{Corollary}
\newtheorem{remark}{Remark}

\def\xx{\mathbf{x}}

\DeclareCaptionStyle{ruled}{labelfont=normalfont,labelsep=colon,strut=off}
\floatstyle{ruled}
\newfloat{listing}{tb}{lst}{}
\floatname{listing}{Listing}
\floatstyle{plain}

\usepackage[]{todonotes}
\definecolor{HaColor}{RGB}{127, 255, 212}

\definecolor{darkblue}{RGB}{0, 0, 139}
\usepackage[pagebackref,breaklinks,colorlinks]{hyperref}
\hypersetup{
    linkcolor=orange,
    citecolor=darkblue,
    urlcolor=blue
}

\def\paperID{*****} 
\def\confName{CVPR}
\def\confYear{2027}

\title{CAVA: Complementary-Label-Aware Data Valuation}

\author{Tan-Ha Mai\\
National Taiwan University
\and
Hsuan-Tien Lin\\
National Taiwan University
\and
Vu Nguyen\\
Amazon
}

\begin{document}
\maketitle
\begin{abstract}
\emph{Complementary-label learning} (CLL) trains a classifier from labels that name a class an instance does \emph{not} belong to. Such labels are cheap to collect from crowdworkers or vision-language models (VLMs), making them attractive for large-scale annotation. However, both sources can mistakenly assign the true class as the complementary label. Such annotation errors are particularly harmful in CLL because the model is explicitly trained to suppress the correct answer. Optimal transport (OT) data valuation identifies harmful samples by comparing training data with a clean validation set, but transferring it directly to CLL creates a semantic mismatch: label $c$ means ``not class $c$'' in the training but ``class $c$'' in the validation. We show that this makes the OT label groups noisy mixtures whose behavior depends on the noise rate. The transferred score is informative only at sufficiently high noise, becomes uninformative at the crossover rate $1/M$ for $M$ classes, and reverses below it, where practical CLL annotation typically operates. We propose \emph{CAVA}, which introduces a \emph{complementary label cost} and a \emph{label margin} that isolates the contribution of each complementary label to the transport potential. The margin keeps a consistent direction across noise rates, detects diagonal noise with a parameter-free threshold of closed-form precision, and allows detected samples to be recycled as ordinary-label supervision. On synthetic and human-annotated benchmarks, CAVA separates diagonal noise almost perfectly where directly transferred OT valuation performs worse than random. With only a fraction of the clean validation data, it also improves every CLL objective we evaluate, showing that OT-based valuation becomes reliable once the semantics of complementary supervision are modeled explicitly.
\end{abstract}

\begin{figure*}[t]
\centering
\includegraphics[width=0.95\textwidth]{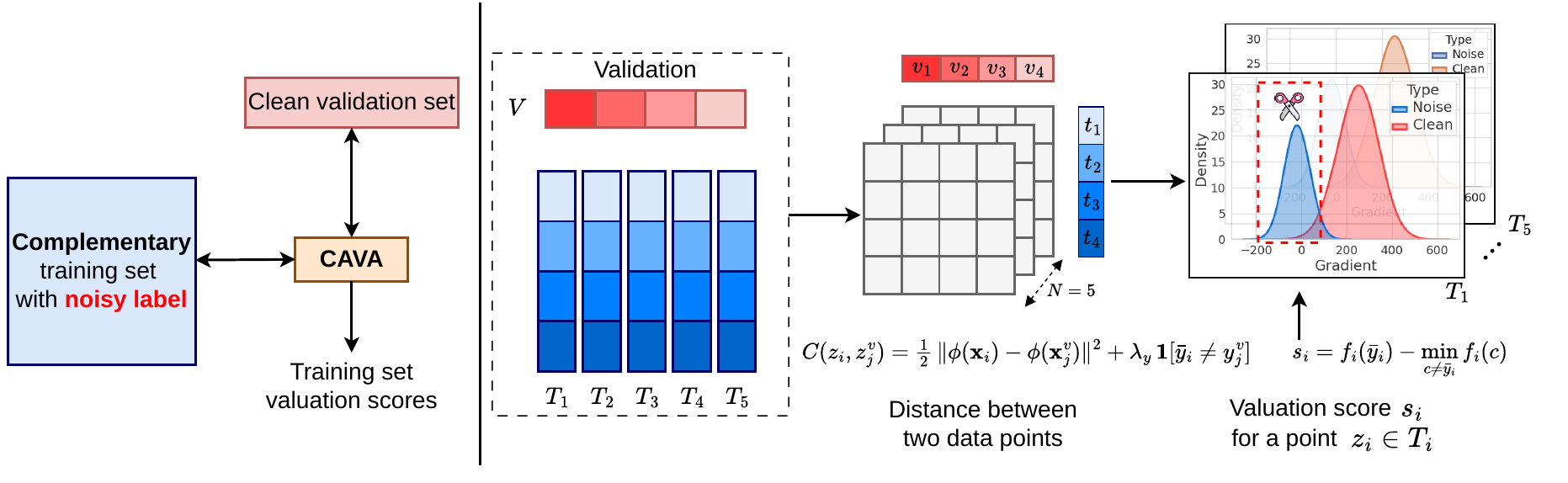}
\caption{\textbf{Overview of CAVA.} The pool is split into tasks, and each task is matched to a clean validation set by entropic OT with the complementary label cost. Each sample is scored by its label margin; samples with a negative margin are flagged as diagonal noise and then removed or recycled.}
\label{fig:cava_overview}
\vspace{-10pt}
\end{figure*}

\section{Introduction}
\label{sec:intro}
Robust vision classifiers depend on large labeled datasets, but annotation remains expensive. Complementary-label learning (CLL) reduces this burden by using a simple form of weak supervision~\citep{sugiyama2022machine}: instead of assigning the correct class, each image is labeled with a class it does \emph{not} belong to~\citep{ishida2017learning,scl2020}. Telling that a photo is \emph{not a truck} may take only a glance, whereas deciding whether an object in fog is a ship or an airplane can require closer inspection. This asymmetry makes complementary labels relatively easy to collect at scale, first from crowdworkers~\citep{wang2024climage} and increasingly from vision-language models (VLMs), which can annotate millions of images at very low cost~\citep{aclimage2025,vlm-2022,vlm-zhu2023minigpt}. CLL is therefore a practical option for large-scale annotation in label spaces of up to about $20$ classes, where current CLL methods are workable~\citep{mai2026cllsurvey}, provided that errors in these labels can be handled reliably.

This condition matters because neither humans nor VLMs are perfect annotators. Both can occasionally assign the \emph{true} class as the class an image does not belong to~\citep{wang2024climage,aclimage2025}. CLL is particularly sensitive to this \emph{diagonal noise} (Figure~\ref{fig1}, Appendix~\ref{diagonal-noise}). A clean complementary label only rules out one incorrect class, whereas a diagonal label explicitly tells the learner to suppress the correct one. Even a small amount of such noise can destabilize training, and common techniques such as augmentation and label sharing can further propagate the error~\citep{mai2025icm,lin2023ls}. A practical CLL pipeline therefore needs a denoising method that works at dataset scale, does not depend on a particular annotator, and provides some indication of how reliable its own decisions are.

Data valuation offers a natural way to do this~\citep{ghorbani2019data}. Among existing approaches, optimal-transport (OT) valuation is attractive because it is scalable and largely independent of the downstream learner~\citep{just2023lava,kesslersava,tarot2025}. OT valuation compares the training set with a small clean validation set using a cost that combines feature distance and label information, then uses the resulting dual potentials to estimate how much each training sample contributes to the transport. In LAVA~\citep{just2023lava}, for example, the label cost measures the distance between feature distributions associated with two label groups, one from training and one from validation. The difficulty in CLL is that these two sets use labels with different meanings: a training label $c$ means ``not class $c$'', while a validation label $c$ means ``class $c$''.

This semantic mismatch changes the valuation in a systematic way. Standard OT label costs implicitly treat samples carrying label $c$ as examples of class $c$, but complementary labels violate this assumption by construction. The training group labeled $c$ is mostly a mixture of the \emph{other} classes; class $c$ appears in that group only through annotation noise. As a result, the transferred label cost acquires an effective weight whose sign depends on the noise rate: it is positive above the crossover rate $1/M$ for $M$ classes, zero at $1/M$, and negative below it (Figure~\ref{fig:transfer_vs_cava}, left). The resulting valuation is therefore informative when noise is frequent, uninformative at the crossover, and reversed when noise is rare, where it can rank corrupted samples as the most trustworthy ones. This distinction is easy to miss because standard synthetic settings often operate above the crossover, where direct transfer appears to work. Real annotation errors are typically much rarer and lie below it, where the transferred score can perform worse than random. The same issue also affects valuation methods that interpret a complementary label as an ordinary class label, including KNN-Shapley~\citep{jia2019efficient} and Data-OOB~\citep{kwon2023dataoob}, which can similarly assign high value to diagonal noise (Figure~\ref{fig:cava_outperform} in Appendix~\ref{app:detect_baselines}).

We address this problem with CAVA, a complementary-label-aware data valuation method for CLL (Figure~\ref{fig:cava_overview}). CAVA keeps the OT formulation but changes the two components through which the semantic mismatch enters. First, the \emph{complementary label cost} is defined between clean classes, so noisy complementary-label groups do not determine the cost itself (Figure~\ref{fig:transfer_vs_cava}, right). Second, the \emph{label margin} measures how much the observed complementary label contributes to the transport potential of a sample, using the same OT solution (Figure~\ref{fig:main_idea}). For a diagonal-noise sample, the observed complementary label is its true class, which the clean validation set matches cheaply; its margin is negative whenever the validation set places the image in that class. Unlike the transferred score, the sign of the margin has the same meaning at every noise rate. It therefore gives an \emph{adaptive cut}: the threshold is fixed at zero, the fraction of flagged samples adapts to the unknown noise rate without a pruning ratio, and the precision of this cut has a closed form in terms of the accuracy of the induced posterior classifier. CAVA uses that precision to decide what to do with flagged samples: they can be removed from the complementary-label pool or, when the detection is sufficiently reliable, recycled as ordinary-label supervision so that the detected error itself becomes useful training signal.

Our contributions are as follows.
\begin{itemize}
\item We identify a \emph{flip law} for transferred data valuation under complementary labels: the effective label-cost weight changes sign with the noise rate, making existing OT valuation blind at $1/M$ and reversed below it, which is the regime relevant to real annotation errors.
\item We propose CAVA, an efficient denoising method for CLL, and show that the sign of its label margin has a fixed meaning across noise rates and numbers of classes. Its zero threshold needs no pruning ratio and has a closed-form precision that can be used to control the risk of recycling flagged samples.
\item On synthetic and human-annotated benchmarks, CAVA separates diagonal noise with AUC above $0.95$ on label-free representations, and almost perfectly on a supervised one, in settings where transferred valuation performs worse than random; it improves a wide range of CLL objectives using only a fraction of the validation data, and scales to a one-million-image augmented pool. We also show that the detection is the agreement of a validation-induced posterior with the given label, which a fitted classifier matches, so the contribution lies in the corrected valuation and the recycling it enables, not in the transport solve.
\end{itemize}

\section{Complementary-Label Learning}
\label{section_cll}

\paragraph{Problem setup.}
Let $\mathcal{X}$ be the input space and $\mathcal{Y}=\{1,\ldots,M\}$ the label space. Each instance $\xx\in\mathcal{X}$ has a true label $y\in\mathcal{Y}$ that is never observed during training. In complementary-label learning (CLL)~\citep{ishida2017learning}, the annotation is a \emph{complementary label} $\bar y\in\mathcal{Y}$ that names a class the instance does \emph{not} belong to, so a correct one satisfies $\bar y\neq y$. The learner receives a pool $\bar{\mathcal{D}}=\{(\xx_i,\bar y_i)\}_{i=1}^{N}$ drawn i.i.d.\ from an unknown distribution and seeks a classifier $g:\mathcal{X}\to\mathbb{R}^M$ whose prediction $\arg\max_{k}g_k(\xx)$ recovers the true label.

\paragraph{Diagonal noise.}
Since an annotation may be wrong, we call $\bar y_i$ the \emph{given label}. It is \emph{valid} if $\bar y_i\neq y_i$, whichever wrong class it names, and \emph{diagonal noise} if $\bar y_i=y_i$, the only case in which the exclusion itself is wrong. We write $\mathcal{N}=\{i:\bar y_i=y_i\}$ and $\eta=\mathbb{P}(\bar y=y)$ for the noise rate. No assumption is made on how valid labels are distributed; the analysis of Section~\ref{cava} uses the uniform model of Assumption~\ref{ass:noise_model}, which is also our synthetic protocol.

\paragraph{Learning objective.}
A CLL method minimizes the empirical counterpart of the complementary risk
\begin{equation}
\label{eq:cll_risk}
R_{\mathrm{comp}}(g)
=
\mathbb{E}_{(\xx,\bar{y})}
\bigl[
\mathcal{L}_{\mathrm{comp}}\bigl(\bar{y}, g(\xx)\bigr)
\bigr],
\end{equation}
where $\mathcal{L}_{\mathrm{comp}}$ is a complementary-label loss such as the unbiased estimator~\citep{ishida2019complementary}, SCL~\citep{scl2020}, forward correction~\citep{fwd2018}, or CPE~\citep{cpe2023}. Each of these losses decreases when predicted probability moves away from the class named by $\bar y$, so a diagonal-noise sample pushes the model away from its \emph{true} class~\citep{wang2024climage,mai2025icm}.

\paragraph{Valuation setting.}
As in OT-based data valuation~\citep{just2023lava,kesslersava}, a small clean validation set $\mathcal{D}_{v}=\{(\xx_j^v,y_j^v)\}_{j=1}^{N_v}$ with ordinary labels is available for valuation only, never for training. The task is to score every sample of $\bar{\mathcal{D}}$ from $\bar{\mathcal{D}}$, $\mathcal{D}_v$, and a fixed feature extractor $\phi$ (Section~\ref{sec:cava_valuation}), so that the ranking separates $\mathcal{N}$ from the valid samples, and then to remove the flagged samples or recycle them with $\bar y_i$ as an ordinary label. The selection observes neither $y_i$ nor $\mathcal{N}$; the only supervised component is the fixed feature extractor $\phi$, whose provenance is stated in Appendix~\ref{sec:full_exp_details}.

\begin{figure*}[t]
\centering
\includegraphics{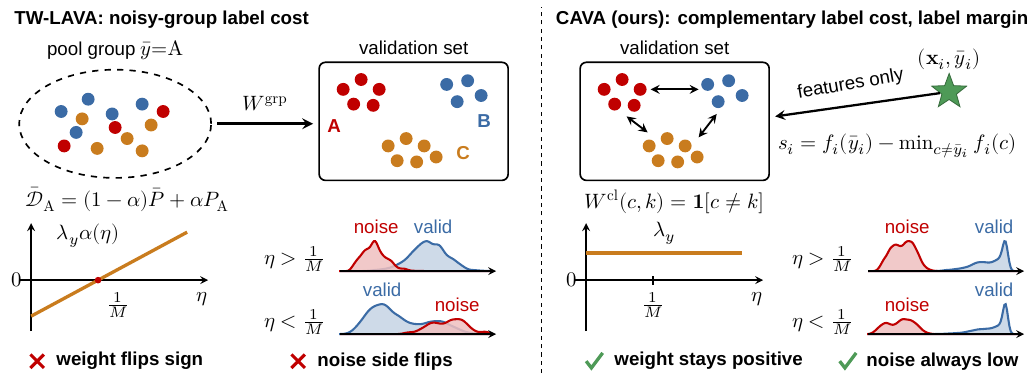}
\caption{\textbf{From TW-LAVA to CAVA.} \emph{Left}: TW-LAVA compares noisy label groups with clean classes, so its label weight $\lambda_y\alpha(\eta)$ changes sign at $\eta=1/M$ and the noise switches sides (Theorem~\ref{thm:flip}). \emph{Right}: CAVA compares clean classes only, so the sign of its margin is fixed: negative exactly when the given label is the most likely class under the transport posterior (Theorem~\ref{thm:margin}). Scores are measured on CIFAR-10 ($\eta=25\%$, top) and CLCIFAR-10 (bottom).}
\label{fig:transfer_vs_cava}
\end{figure*}

\section{CAVA: \underline{C}omplementary-Label \underline{A}ware Data \underline{Va}luation}
\label{cava}
We first describe the direct transfer of OT valuation to CLL, which we call TW-LAVA, and show why it fails (Sections~\ref{sec:cava_valuation} and~\ref{sec:cava_flip}). We then introduce the two components of CAVA, the complementary label cost and the label margin (Section~\ref{sec:cava_refmargin}), and the selection and recycling built on them (Sections~\ref{sec:cava_selection} and~\ref{sec:cava_recycling}). Proofs are in Appendix~\ref{app:refmargin_proofs}.

\subsection{TW-LAVA: OT valuation transferred to CLL}
\label{sec:cava_valuation}
TW-LAVA keeps the label cost and the score of LAVA~\citep{just2023lava} and adds a task-wise decomposition that keeps each transport problem small. The pool $\bar{\mathcal{D}}$ is split into $T$ equal-size random tasks $I_1,\ldots,I_T\subseteq\{1,\ldots,N\}$, and each task is valued against the validation set $\mathcal{D}_v$ independently. For task $t$, let $\mu_t=\sum_{i\in I_t}a_i\delta_{z_i}$ with $z_i=(\xx_i,\bar y_i)$ and $\nu=\sum_{j=1}^{N_v}b_j\delta_{z_j^v}$ with $z_j^v=(\xx_j^v,y_j^v)$, with uniform weights $a_i$ and $b_j$. The ground cost combines a feature term and a label term,
\begin{equation}
C(z_i,z_j^v)=\lambda_\xx\,C_{\mathrm{feat}}(\xx_i,\xx_j^v)+\lambda_y\,C_{\mathrm{label}}(\bar y_i,y_j^v),
\label{eq:cava_cost}
\end{equation}
with the feature cost
\begin{equation}
C_{\mathrm{feat}}(\xx_i,\xx_j^v)=\tfrac12\bigl\|\phi(\xx_i)-\phi(\xx_j^v)\bigr\|_2^2
\label{eq:cava_feat_cost}
\end{equation}
of a fixed feature extractor $\phi$, whose outputs are L2-normalized and rescaled to a median feature cost of one; we then set $\lambda_\xx=1$, so $\lambda_y>0$ is the only cost weight. LAVA compares label groups across the two sets, which under complementary labels gives the \emph{noisy-group label cost}
\begin{equation}
C_{\mathrm{label}}(\bar y_i,y_j^v)=\tfrac12W_2^2\bigl(\bar{\mathcal{D}}_{\bar y_i},\mathcal{D}^v_{y_j^v}\bigr)\eqqcolon W^{\mathrm{grp}}(\bar y_i,y_j^v),
\label{eq:cava_label_cost}
\end{equation}
where $\bar{\mathcal{D}}_{c}$ is the feature distribution of the pool samples with given label $c$ and $\mathcal{D}^v_{k}$ that of the validation samples of class $k$. Each task solves the entropic OT problem
\begin{equation}
W_{C,\varepsilon}(\mu_t,\nu)=\min_{\pi\in\Pi(a,b)}\sum_{i,j}\pi_{ij}C(z_i,z_j^v)+\varepsilon\sum_{i,j}\pi_{ij}\log\pi_{ij},
\label{eq:cava_entropic_ot}
\end{equation}
with $\varepsilon=\varepsilon_{\mathrm{OT}}>0$ and optimal pool-side and validation-side potentials $u^{(t),*}$ and $w^{(t),*}$. TW-LAVA scores sample $i$ by the calibrated potential of LAVA,
\begin{equation}
v_i^{(t)}=u_i^{(t),*}-\frac{1}{|I_t|-1}\sum_{k\in I_t,\,k\neq i}u_k^{(t),*},
\label{eq:cava_score}
\end{equation}
the change of the transport cost when sample $i$ receives more mass.

\subsection{Why TW-LAVA fails: a flip at $\eta=1/M$}
\label{sec:cava_flip}
The noisy-group label cost assumes that the samples carrying label $c$ are instances of class $c$. Under complementary labels they are not. If classes are balanced and a given label equals the true label with probability $\eta$ and is otherwise uniform over the other $M-1$ classes, the group with label $c$ is the mixture
\begin{equation}
\bar{\mathcal{D}}_{c}=(1-\alpha)\,\bar P+\alpha\,P_c,
\qquad
\alpha(\eta)=\frac{M\eta-1}{M-1},
\label{eq:group_mix}
\end{equation}
where $P_k$ is the feature distribution of class $k$ and $\bar P$ the average over classes. The group leans toward its own class when noise is frequent and away from it when noise is rare, and the label cost inherits this lean.

\begin{theorem}[Flip law]
\label{thm:flip}
\textit{Under this noise model, (i) at $\eta=1/M$ the given labels leave the ground cost for any label divergence, so no score computed from the transport can separate diagonal noise from valid samples; (ii) for a squared mean-embedding label divergence,
\begin{equation}
W^{\mathrm{grp}}(c,k)=\alpha(\eta)\,W^{\mathrm{cl}}_{D}(c,k)+r_c+t_k,
\label{eq:flip_law}
\end{equation}
where $W^{\mathrm{cl}}_{D}$ is the same divergence between clean classes and the shifts $r_c,t_k$ do not change the transport plan.}
\end{theorem}
At the level of the transport plan, TW-LAVA therefore uses a clean-class label cost with the \emph{effective} weight $\lambda_y\alpha(\eta)$: positive above $\eta=1/M$, zero at it, and negative below, where it rewards a given label that agrees with the validation class and reverses the score (the scores themselves coincide when the clean classes have equal radii; Appendix~\ref{app:refmargin_proofs}). Part (ii) is exact for mean-embedding divergences; for the Sinkhorn class divergence of the pipeline we observe the same sign change with a smaller slope (Remark~\ref{rem:refmargin_scope}). Synthetic benchmarks sit above the crossover ($\alpha=1/6$ at $\eta=25\%$, $M=10$), and human annotation errors sit below it ($\alpha<0$ on CLCIFAR-10, whose labels are not uniform, so the law applies there only qualitatively). On CIFAR-10 the TW-LAVA score crosses chance exactly at $\eta=1/M$ and is worse than random below it (Figure~\ref{fig:flip}, Appendix~\ref{app:refmargin_proofs}).

\begin{figure}[t]
\centering
\includegraphics[scale=0.9]{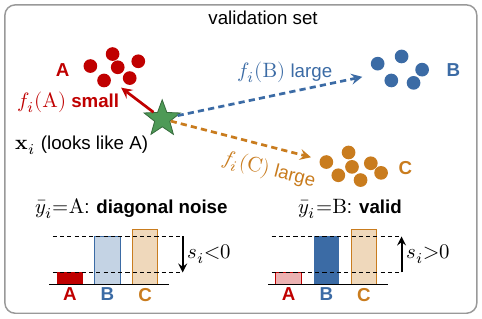}
\caption{\textbf{The label margin.} The image looks like class A, so it is cheap to transport under label A and costly under B or C. The bars show $f_i(c)$, and the margin $s_i$ compares the given label (solid bar) with the cheapest other label. Given A, $s_i<0$ and the sample is flagged as diagonal noise; given B, $s_i>0$ and the label is valid.}
\label{fig:main_idea}
\end{figure}

\subsection{Complementary label cost and label margin}
\label{sec:cava_refmargin}
CAVA keeps the transport problem of TW-LAVA and changes the two ingredients that the flip runs through (Figure~\ref{fig:transfer_vs_cava}, right).

\textbf{Complementary label cost.} The label cost is defined between clean classes, so the noisy groups never enter the cost and the given label is only looked up:
\begin{equation}
C_{\mathrm{label}}(\bar y_i,y_j^v)=W^{\mathrm{cl}}(\bar y_i,y_j^v),
\qquad
W^{\mathrm{cl}}(c,k)=\mathbf{1}[c\neq k].
\label{eq:ref_label_cost}
\end{equation}
The class-distance variant $W^{\mathrm{cl}}_{D}$ of Theorem~\ref{thm:flip} performs within $0.01$ AUC of this $0/1$ cost (Appendix~\ref{app:refmargin_proofs}), and the $0/1$ cost makes the analysis exact.

\textbf{Label margin.} The calibrated score of Eq.~\eqref{eq:cava_score} measures the effect of the \emph{mass} of a sample, so a typical image with a wrong label and an atypical image with a right label can score alike. We measure instead what the \emph{label} contributes. With the validation potential $w=w^{(t),*}$ of the problem under Eq.~\eqref{eq:ref_label_cost}, the potential of sample $i$ under label $c$ is
\begin{equation}
f_i(c)=-\varepsilon\log\sum_j b_j\exp\Bigl(\tfrac{w_j-C_{\mathrm{feat}}(\xx_i,\xx_j^v)-\lambda_y\mathbf{1}[c\neq y^v_j]}{\varepsilon}\Bigr),
\label{eq:cava_fc}
\end{equation}
with $f_i(\bar y_i)=u_i^{(t),*}$, so no further transport problem is solved. The \emph{label margin} compares the given label with its best alternative,
\begin{equation}
s_i=f_i(\bar y_i)-\min_{c\neq\bar y_i}f_i(c).
\label{eq:cava_margin}
\end{equation}
A diagonal-noise sample names its own class, the one label the validation set serves cheaply, so its margin is negative whenever the transport places the image in that class (Figure~\ref{fig:main_idea}).

\begin{theorem}[Direction of the label margin]
\label{thm:margin}
\textit{For the $0/1$ complementary label cost and any validation potential $w$,
\begin{equation}
s_i=\psi\bigl(p_i(\bar y_i)\bigr)-\psi\Bigl(\max_{c\neq\bar y_i}p_i(c)\Bigr),
\label{eq:margin_closed}
\end{equation}
where $p_i(k)\propto\sum_{j:y^v_j=k}b_j\,e^{(w_j-C_{\mathrm{feat}}(\xx_i,\xx_j^v))/\varepsilon}$ is the class posterior induced by the transport and $\psi$ is strictly decreasing. Hence $s_i\le0$ if and only if $\bar y_i$ is the most likely class under $p_i$, for every noise rate and every $M$.}
\end{theorem}
The theorem fixes the sign of the margin relative to the posterior $p_i$, so $s_i\le0$ is a cut that needs no pruning ratio; how often diagonal noise is flagged is the accuracy of the posterior classifier on it (Theorem~\ref{thm:detect}), which depends on the features and the validation set, not on $\eta$ or $M$. The closed form also makes the score cheap: one softmax over the validation set gives $p_i$ and then $s_i$. As $\varepsilon\to0$ the margin reduces to a nearest-neighbor label check, while a finite $\varepsilon$ keeps the transport weighting of the whole validation set.

\subsection{Task-wise selection}
\label{sec:cava_selection}
CAVA selects within each task, which keeps every transport problem small and adapts the cut to each task. The \emph{fixed cut} flags the fraction $p$ of each task with the lowest margins, as prior OT valuation methods prune a fixed fraction~\citep{just2023lava,kesslersava}. The \emph{adaptive cut} needs no budget: its threshold is fixed at zero by Theorem~\ref{thm:margin}, so it flags exactly the samples whose given label is the most likely class, and the fraction it removes adapts to the noise rate of each task without knowing $\eta$:
\begin{equation}
I_{t}^{\mathrm{rem}}=\{i\in I_t:\ s_i\le0\},
\qquad
I^{\mathrm{keep}}=\bigcup_{t=1}^{T}\bigl(I_t\setminus I_{t}^{\mathrm{rem}}\bigr).
\label{eq:strategy2_zero}
\end{equation}
TW-LAVA has no such cut, because the sign of its score depends on whether $\eta$ lies above or below $1/M$, which is unknown in practice; it is therefore used with the fixed cut. Without recycling, flagged samples are discarded and the classifier is trained on the retained pool,
\begin{equation}
\mathcal{L}_{\mathrm{sel}}=\frac{1}{|I^{\mathrm{keep}}|}\sum_{i\in I^{\mathrm{keep}}}\mathcal{L}_{\mathrm{comp}}\bigl(\bar y_i,g(\xx_i)\bigr),
\label{eq:cava_selection_objective}
\end{equation}
where $\mathcal{L}_{\mathrm{comp}}$ is any CLL loss; for the fixed cut, $I^{\mathrm{keep}}$ is the complement of the flagged tails.

\subsection{Recycling gated by precision}
\label{sec:cava_recycling}
A diagonal-noise label satisfies $\bar y_i=y_i$, so it is wrong as an exclusion but correct as an ordinary label. CAVA can therefore recycle the flagged set $I^{\mathrm{rec}}=\bigcup_t I_t^{\mathrm{rem}}$ into a cross-entropy branch and train
\begin{equation}
\begin{aligned}
\mathcal{L}_{\mathrm{CAVA}}={}&\frac{\kappa}{|I^{\mathrm{rec}}|}\sum_{i\in I^{\mathrm{rec}}}\mathcal{L}_{\mathrm{CE}}\bigl(\bar y_i,g(\xx_i)\bigr)\\
&+\frac{1-\kappa}{|I^{\mathrm{keep}}|}\sum_{i\in I^{\mathrm{keep}}}\mathcal{L}_{\mathrm{comp}}\bigl(\bar y_i,g(\xx_i)\bigr),
\end{aligned}
\label{eq:cava_mixed_objective}
\end{equation}
with weight $\kappa\in[0,1]$. A falsely recycled valid sample injects a wrong ordinary label, so recycling helps only when the flagged set is pure. The adaptive cut makes this purity predictable: by Theorem~\ref{thm:margin}, a sample is flagged exactly when the posterior classifier $h_i=\arg\max_k p_i(k)$ agrees with its given label.

\begin{theorem}[Precision-gated recycling]
\label{thm:detect}
\textit{Let $F=\{i:s_i\le0\}$ and let $e$ be the error rate of $h$ against the true labels. Under the noise model of Section~\ref{sec:cava_flip}, if the errors of $h$ are independent of the label draw ($h$ depends on the given labels only through the validation potential $w$), the precision of $F$ is
\begin{equation}
\pi=\frac{\eta\,(1-e)}{\eta\,(1-e)+(1-\eta)\,e/(M-1)},
\label{eq:detect_precision}
\end{equation}
which exceeds $\eta$ whenever $e\le1-1/M$. Conditionally on the selection, training the cross-entropy branch on $F$ has excess true-label risk at most $2B(1-\pi)+2B|\pi-(1-e)|+\Delta_{\mathrm{gen}}(|F|,\delta)$ with probability $1-\delta$, where $B$ bounds the loss, the middle term is the distribution shift of the flagged set, and $\Delta_{\mathrm{gen}}$ vanishes as $|F|$ grows (Theorem~\ref{thm:recycle_truelabel}, Remark~\ref{rem:selection}).}
\end{theorem}
A valid sample is flagged only when $h$ errs \emph{and} its wrong guess equals the given label, an extra factor $1/(M-1)$, which is why the flagged set is far purer than the base rate. Neither the bias term $2B(1-\pi)$ nor the shift term, which is about $e$ when the flagged set is pure, shrinks with more data, so the precision and the recall of the posterior classifier decide the mode: recycling when it is high, selection only (Eq.~\eqref{eq:cava_selection_objective}) otherwise. Eq.~\eqref{eq:detect_precision} needs only $e$, which can be estimated from the validation set; in our experiments we report the precision measured with the true labels after selection (Appendix~\ref{app:branch_metrics}).

\section{Experiments}
\label{experiments}
\subsection{Experimental setup}
\label{sec:experiment_setup}

\begin{table*}[t]
\centering
\caption{
\textbf{Main results: test accuracy (\%) with AutoAugment.} All methods score the pool against the same $2{,}000$ validation images ($20$ per class on CLMIN), and accuracy is measured after the final epoch on the remaining test images. The fixed cut removes $p=25\%$ of each task on the synthetic datasets and $5\%$ ($3\%$ on CLCIFAR-20) on the human-labeled ones; the adaptive cut removes every sample with $s_i\le0$ and needs no ratio. Mean$\pm$std over three seeds; best per column and block in \textbf{bold}.
}
\label{tab:main_results}
\small
\setlength{\tabcolsep}{5.5pt}
\begin{tabular}{@{}ll ccc cccc@{}}
\toprule
& & \multicolumn{3}{c}{\textbf{Synthetic} ($\eta=25\%$)} & \multicolumn{4}{c}{\textbf{Real-world}} \\
\cmidrule(lr){3-5} \cmidrule(l){6-9}
\textbf{Method} & \textbf{Selection} & MNIST & CIFAR-10 & CIFAR-20 & CLCIFAR-10 & CLCIFAR-20 & CLMIN-10 & CLMIN-20 \\
\midrule
\multicolumn{9}{@{}l}{\textit{(a) Selection only: SCL-NL on the retained pool, Eq.~\eqref{eq:cava_selection_objective}}} \\
\midrule
No pruning & -- & 0.14{\scriptsize$\pm$0.09} & 1.19{\scriptsize$\pm$0.02} & 0.49{\scriptsize$\pm$0.10} & 42.36{\scriptsize$\pm$0.91} & 7.71{\scriptsize$\pm$0.08} & 14.89{\scriptsize$\pm$1.02} & 4.55{\scriptsize$\pm$1.18} \\
TW-LAVA & Fixed $p$ & 98.21{\scriptsize$\pm$0.29} & 76.04{\scriptsize$\pm$0.25} & 14.43{\scriptsize$\pm$1.00} & 47.46{\scriptsize$\pm$0.17} & 6.90{\scriptsize$\pm$0.11} & 12.56{\scriptsize$\pm$0.84} & 4.28{\scriptsize$\pm$0.35} \\
\rowcolor{green!12} CAVA (ours) & Fixed $p$ & 99.20{\scriptsize$\pm$0.10} & \textbf{79.26}{\scriptsize$\pm$0.25} & \textbf{17.70}{\scriptsize$\pm$1.38} & 79.63{\scriptsize$\pm$0.33} & \textbf{10.78}{\scriptsize$\pm$1.25} & \textbf{27.44}{\scriptsize$\pm$2.69} & \textbf{10.39}{\scriptsize$\pm$0.67} \\
\rowcolor{green!12} CAVA (ours) & Adaptive cut & \textbf{99.47}{\scriptsize$\pm$0.04} & 79.12{\scriptsize$\pm$0.13} & 12.54{\scriptsize$\pm$2.68} & \textbf{80.00}{\scriptsize$\pm$0.59} & 10.27{\scriptsize$\pm$1.15} & \textbf{27.44}{\scriptsize$\pm$1.84} & 9.56{\scriptsize$\pm$1.13} \\
\midrule
\multicolumn{9}{@{}l}{\textit{(b) Selection with recycling: CE+SCL-NL, Eq.~\eqref{eq:cava_mixed_objective}}} \\
\midrule
TW-LAVA & Fixed $p$ & 97.68{\scriptsize$\pm$0.14} & 88.75{\scriptsize$\pm$0.40} & 59.03{\scriptsize$\pm$0.77} & 38.71{\scriptsize$\pm$0.85} & 10.24{\scriptsize$\pm$0.44} & 23.33{\scriptsize$\pm$0.67} & 5.33{\scriptsize$\pm$0.73} \\
\rowcolor{green!12} CAVA (ours) & Fixed $p$ & 99.36{\scriptsize$\pm$0.07} & 91.57{\scriptsize$\pm$0.14} & \textbf{75.34}{\scriptsize$\pm$0.17} & 81.84{\scriptsize$\pm$0.15} & 47.71{\scriptsize$\pm$0.38} & \textbf{52.78}{\scriptsize$\pm$0.19} & 35.22{\scriptsize$\pm$1.11} \\
\rowcolor{green!12} CAVA (ours) & Adaptive cut & \textbf{99.66}{\scriptsize$\pm$0.03} & \textbf{91.97}{\scriptsize$\pm$0.14} & 73.05{\scriptsize$\pm$0.32} & \textbf{85.69}{\scriptsize$\pm$0.29} & \textbf{48.86}{\scriptsize$\pm$0.66} & 52.33{\scriptsize$\pm$1.45} & \textbf{36.11}{\scriptsize$\pm$1.19} \\
\bottomrule
\end{tabular}
\end{table*}

We evaluate CAVA on MNIST~\citep{mnist}, CIFAR-10, CIFAR-20~\citep{cifar10}, the real-world CLImage datasets CLCIFAR-10/20 and CLMIN-10/20~\citep{wang2024climage}, an enlarged CIFAR-10-1M setting, and three further domains, FOOD15, RESISC45-20, and Places365-10, built from Food-101~\citep{bossard2014food}, NWPU-RESISC45~\citep{cheng2017remote}, and Places365~\citep{zhou2018places} (Appendix~\ref{clarifying_cifar10}). All label spaces have at most $20$ classes, the scale at which current CLL methods are workable: a complementary label excludes a single class, so its supervision weakens as the number of classes grows~\citep{mai2026cllsurvey}. 
For synthetic experiments, each training example receives an observed complementary label $\bar y_i$: with probability $\eta$, we set $\bar y_i=y_i$ to create diagonal complementary-label noise; otherwise, we sample a valid complementary label uniformly from $\mathcal{Y}\setminus\{y_i\}$. 
The test set always uses clean ground-truth labels. Unless otherwise stated, every valuation method uses the same validation set of $N_v=2{,}000$ clean images, a class-stratified subset of the test split. The validation images are used only to score the training samples: they never enter model training, model selection, or the training objective. Test accuracy is measured after the final training epoch on the remaining $8{,}000$ test images, which the valuation never sees, so no quantity is selected on the evaluation data. CLMIN-10 and CLMIN-20 contain $64\times64$ ImageNet images, which we resize to $32\times32$; their test splits hold only $50$ images per class, so $20$ per class form the validation set and the other $30$ per class the evaluation set. This assumption is shared by the valuation baselines and is examined in Appendix~\ref{app:reference_assumption}. All valuation methods share one fixed feature extractor. The main tables use the supervised embedder of LAVA, which was trained with the clean labels of the training split and is therefore an upper bound on feature quality (Appendix~\ref{sec:full_exp_details}); Appendix~\ref{app:representations} repeats the detection and training experiments with two extractors that have seen no label of the pool, CLIP features and a network trained on the validation images only.

We compare against No Pruning and the transferred OT valuation TW-LAVA, and at small noise rates also against Random Pruning, SAVA~\citep{kesslersava}, KNN-Shapley~\citep{jia2019efficient}, Data-OOB~\citep{kwon2023dataoob}, and small-loss selection~\citep{han2018co}; LAVA itself enters the runtime and memory comparison (Appendix~\ref{app:computational_efficiency}). For CLL objectives, we evaluate FWD~\citep{fwd2018}, CPE~\citep{cpe2023}, OP~\citep{orderpreserving2023}, SCARCE~\citep{conu2023}, SCL-NL, and SCL-EXP~\citep{scl2020}. All models use ResNet-18~\citep{he2015deep} and are trained under the same protocol across compared methods. Full hyperparameters, seeds, single-run settings for expensive scalability experiments, and additional dataset details are provided in Appendix~\ref{Appendix_C}. Table~\ref{tab:method_names} fixes the names of the three OT valuation methods: LAVA is the original full-set method, TW-LAVA is its task-wise transfer to complementary labels (Section~\ref{sec:cava_valuation}), and CAVA replaces the label cost and the score (Section~\ref{sec:cava_refmargin}).

\begin{table}[t]
\centering
\caption{The three OT valuation methods compared in this paper. TW-LAVA differs from LAVA only in the task-wise decomposition; CAVA differs from TW-LAVA only in the label cost and the score.}
\label{tab:method_names}
\small
\setlength{\tabcolsep}{3.5pt}
\begin{tabular}{@{}llll@{}}
\toprule
Method & Valuation scope & Label cost & Score \\
\midrule
LAVA~\citep{just2023lava} & full set ($\le$20K) & noisy-group & calibrated \\
TW-LAVA & task-wise, all data & noisy-group & calibrated \\
CAVA (ours) & task-wise, all data & complementary & label margin \\
\bottomrule
\end{tabular}
\end{table}

\subsection{Results and analysis}

\begin{figure*}[t]
\centering
\includegraphics[scale=0.9]{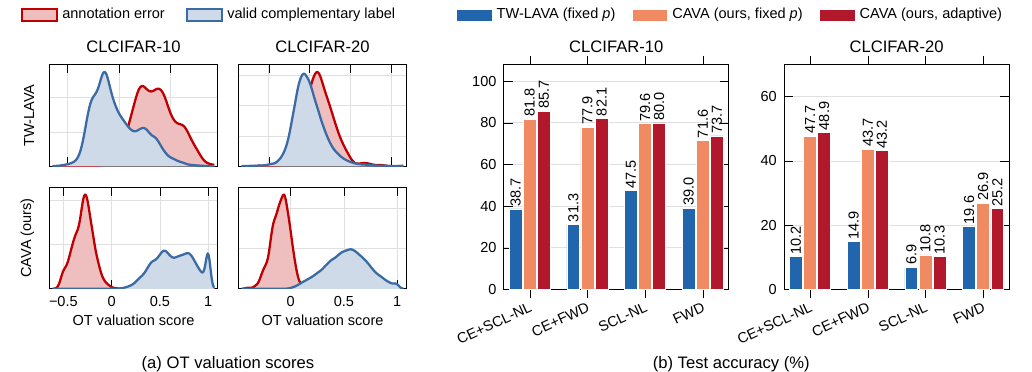}
\caption{\textbf{Real-world complementary labels.} (a) Score distributions: under TW-LAVA (top), annotation errors overlap with valid labels and sit on the wrong side; under CAVA (bottom), they separate on the low side. (b) Test accuracy of TW-LAVA and CAVA at the fixed cut of Table~\ref{tab:main_results} ($5\%$ on CLCIFAR-10, $3\%$ on CLCIFAR-20) and of CAVA with the adaptive cut, mean of three seeds; all use the same $2{,}000$ validation images.}
\label{fig:realworld_compare}
\end{figure*}

\begin{figure*}[t]
\centering
\includegraphics[scale=0.9]{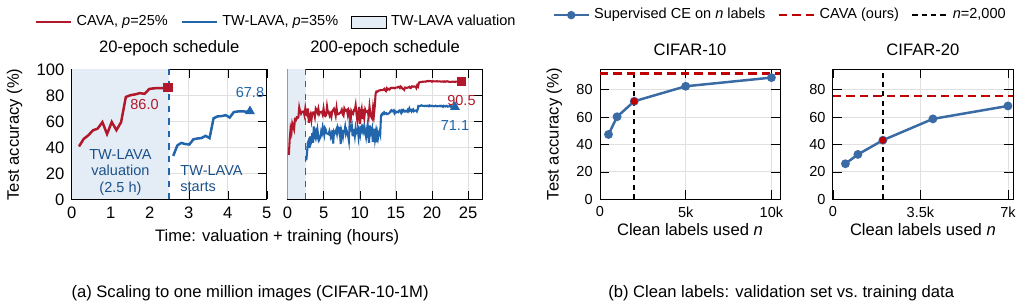}
\caption{\textbf{Scale and the value of clean labels.} (a) CIFAR-10-1M ($\eta=25\%$, SCL-NL): test accuracy after every epoch against time (valuation plus training, measured on one dedicated GPU), each method at its best cut (CAVA $p=25\%$, TW-LAVA $p=35\%$), for the 20-epoch and the 200-epoch schedule. Shaded: the valuation of TW-LAVA, so its training starts only after $2.5$ hours (dashed line); CAVA needs about 4 minutes. All cuts are in Table~\ref{tab:cifar10_1m_scalability}. (b) Supervised CE trained from scratch on $n$ clean labels (three seeds) against CAVA, which uses $2{,}000$ clean labels only as its validation set ($\eta=25\%$); CAVA uses the supervised embedder of the main tables, and results on other fixed representations are in Table~\ref{tab:reference_budget}.}
\label{fig:scale_budget}
\end{figure*}

\begin{figure}[t]
\centering
\includegraphics{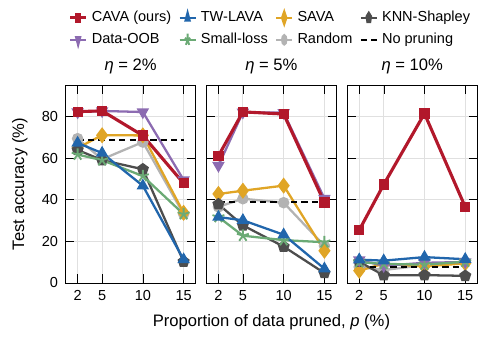}
\caption{\textbf{Small noise ratios.} Final-epoch test accuracy (held-out test images) of SCL-NL on the retained pool against the pruning ratio $p$, on CIFAR-10 with small diagonal-noise ratios $\eta$; CAVA, TW-LAVA, and no pruning are means over three seeds, the other baselines single runs (seed 1126); same validation set for all methods; TW-LAVA keeps its original OT settings here (Appendix~\ref{sec:full_exp_details}). Each baseline is used as published, keeping its highest-valued or lowest-loss samples. The dashed line trains on the full pool.}
\label{fig:smallnoise}
\end{figure}

\begin{figure}[t]
\centering
\includegraphics{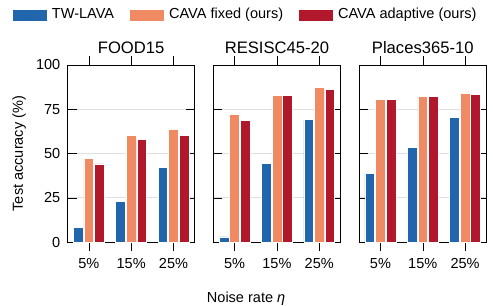}
\caption{\textbf{Beyond CIFAR.} Test accuracy with recycling (CE+SCL-NL) on FOOD15 (Food-101~\citep{bossard2014food}), RESISC45-20 (NWPU-RESISC45~\citep{cheng2017remote}), and Places365-10 (Places365~\citep{zhou2018places}) with synthetic diagonal noise; protocol of Table~\ref{tab:main_results}, fixed cut $p=\eta$, seed 1126. All rows, including selection only, are in Table~\ref{tab:newds}.}
\label{fig:newds}
\end{figure}

\paragraph{Main results.}
Table~\ref{tab:main_results} compares CAVA with the TW-LAVA baseline under one validation set. With the same fixed cut, CAVA is higher than TW-LAVA in all $14$ cells, and without selection SCL-NL collapses on the synthetic datasets (at most $1.19\%$). The gap is largest on the human-annotated datasets, whose error rates lie below $1/M$: there, TW-LAVA does not beat training on the full pool on CLCIFAR-20 and both CLMIN datasets, as the flip law predicts, while CAVA with recycling reaches $85.69\%$, $48.86\%$, $52.78\%$, and $36.11\%$ on CLCIFAR-10, CLCIFAR-20, CLMIN-10, and CLMIN-20, against at most $38.71\%$, $10.24\%$, $23.33\%$, and $5.33\%$ for TW-LAVA with recycling. The adaptive cut, which needs no pruning ratio, is best or tied in eight of the $14$ CAVA comparisons and within one point of the fixed cut except on CIFAR-20, where it flags fewer noisy samples.

\paragraph{Beyond CIFAR.}
Figure~\ref{fig:newds} repeats the comparison on three further domains, food, aerial scenes, and places, at three noise rates (Table~\ref{tab:newds} in Appendix~\ref{app:newds} lists all rows). With recycling, CAVA is far above TW-LAVA at every noise rate, with either cut, and the gap is largest at $\eta=5\%$, near the crossover $1/M$, where TW-LAVA reaches only $8.27\%$, $3.00\%$, and $38.95\%$ against $47.33\%$, $71.93\%$, and $80.49\%$ for CAVA. Without recycling, CAVA with the fixed cut is above TW-LAVA in every setting, by $0.4$ points on RESISC45-20 at $\eta=25\%$ and by more elsewhere, while its adaptive cut falls below TW-LAVA on FOOD15 and RESISC45-20 at $\eta=25\%$. At $\eta=5\%$, the selection of TW-LAVA stays within $4$ points of training on the full pool, whereas CAVA gains $14$ to $30$ points.

\paragraph{Small noise ratios.}
Real complementary labels are mostly correct: both CLCIFAR datasets have noise rates below $1/M$, the regime in which the noisy-group label cost of TW-LAVA reverses. Figure~\ref{fig:smallnoise} reproduces this regime on CIFAR-10 with $\eta\in\{2\%,5\%,10\%\}$ and trains SCL-NL on the retained pool. All methods use the same feature extractor, the same $2{,}000$ validation images, and the same task pools; each baseline is used as published, keeping the samples it values most (SAVA~\citep{kesslersava}, KNN-Shapley~\citep{jia2019efficient}, Data-OOB~\citep{kwon2023dataoob}) or the samples with the smallest loss~\citep{han2018co}. CAVA peaks at the cut $p=\eta$ at $\eta=5\%$ and $10\%$ ($82.2\%$ and $81.4\%$) and at $p=5\%$ at $\eta=2\%$ ($82.6\%$; $82.3\%$ at $p=2\%$), against $68.5\%$, $38.7\%$, and $7.9\%$ without pruning. TW-LAVA removes valid samples first, so pruning lowers its accuracy at $\eta=2\%$ and $5\%$ and adds at most $4.6$ points at $\eta=10\%$. Among the baselines, only Data-OOB keeps up with CAVA, and only below the crossover: it reaches $82.6\%$ and $82.0\%$ at $\eta=2\%$ and $5\%$, where its natural direction is the correct one (Figure~\ref{fig:cava_outperform} in Appendix~\ref{app:detect_baselines}), but $11.0\%$ at $\eta=10\%=1/M$, where its score is blind. Random pruning, KNN-Shapley, and small-loss selection stay near or below no pruning at every noise rate, SAVA gains at most $8.1$ points at $\eta=5\%$, and at $\eta=10\%$ no baseline exceeds $12.4\%$, while CAVA reaches $81.4\%$.


\paragraph{Selection without a pruning ratio.}
In practice the noise rate is unknown, so a fixed ratio $p$ is hard to tune. The adaptive cut of Eq.~\eqref{eq:strategy2_zero} needs no ratio: it flags a different fraction in each task, and on MNIST, CIFAR-10, CIFAR-20, and the two CLCIFAR datasets its flagged set reaches $96.0\%$ to $99.98\%$ precision and $87.2\%$ to $99.9\%$ recall (Tables~\ref{tab:representations} and~\ref{tab:mnist-ablation}), with a single OT solve per task. A score whose direction depends on $\eta$, such as that of TW-LAVA, admits no such threshold. In training, the adaptive cut performs on par with the best fixed cut without choosing a ratio, except on CIFAR-20, and with recycling it is the best setting on five of the seven datasets (Table~\ref{tab:main_results}).

\paragraph{Why existing valuation fails, and what fixes it.}
Figure~\ref{fig:cava_outperform} in Appendix~\ref{app:detect_baselines} scores every method on the same pool and tasks, at the small noise rates of real complementary labels. KNN-Shapley~\citep{jia2019efficient} credits a sample whose label agrees with its validation neighbors, and a diagonal label agrees by construction, so it ranks the noise as the most valuable data on every dataset. Data-OOB~\citep{kwon2023dataoob} and small-loss selection~\citep{han2018co} learn from the given labels themselves, so their direction follows the flip law: they are correct below the crossover $1/M$ (CIFAR-10 at $\eta=5\%$ and both CLCIFAR datasets), near chance on CIFAR-20 at $\eta=5\%=1/M$, and reversed above it (both datasets at $\eta=15\%$). TW-LAVA follows the same law with the opposite sign: it is reversed on CIFAR-10 at $\eta=5\%$ and on both human-annotated datasets, near chance on CIFAR-20 at $\eta=5\%$, and correct only above $1/M$, at $\eta=15\%$, as Theorem~\ref{thm:flip} predicts for the synthetic data and consistently with it on the human labels. Every score that reads a complementary label as an ordinary one therefore has a direction that depends on whether $\eta$ lies above or below $1/M$, which is unknown in practice. CAVA reaches the same detection inside the OT valuation framework, where the cut at zero and the precision of Theorem~\ref{thm:detect} come with the score at no extra cost. The transport solve itself is not what detects the noise: Theorem~\ref{thm:margin} holds for any validation potential, and on CLIP features the margin with the potential set to zero is within $0.006$ AUC of CAVA on each of the five datasets where we computed it, and the marginal constraint of the transport lowers the AUC when the class marginal of the pool and of the validation set differ (Appendix~\ref{app:representations}). The gain over existing valuation comes from the complementary label cost, which removes the flip, and from the margin, which isolates the label; the transport problem is the framework in which both are defined, not a better detector.


Recycling detected diagonal-noise samples further improves both methods on CIFAR-10 and CIFAR-20 and CAVA on all three synthetic datasets (Table~\ref{tab:main_results}, block (b) against block (a)); for TW-LAVA with the fixed cut on CIFAR-10, accuracy rises from $76.04\%$ to $88.75\%$. Detected diagonal-noise samples thus provide useful ordinary-label supervision when they are recycled with sufficient precision. 

\paragraph{Real-world complementary labels.}
CLCIFAR-10 and CLCIFAR-20 contain $3.98\%$ and $2.32\%$ annotation errors, both below $1/M$, which is the regime where the label cost of TW-LAVA reverses (Theorem~\ref{thm:flip}). Figure~\ref{fig:realworld_compare}(a) shows the consequence. With the noisy-group label cost, annotation errors and valid labels overlap, and at $p=3\%$ only $43.5\%$ (CLCIFAR-10) and $9.9\%$ (CLCIFAR-20) of the recycled samples are true errors (Tables~\ref{tab:clcifar10-ablation} and~\ref{tab:clcifar20-ablation}). With the complementary label cost and the label margin, the two groups separate almost completely under the same feature extractor and validation set: the recycled branch at $p=3\%$ is $100\%$ true errors on CLCIFAR-10 and $75.9\%$ on CLCIFAR-20, where it already contains $98.0\%$ of all errors, and the adaptive cut on CLCIFAR-10 recycles $1{,}990$ samples at $99.7\%$ precision. This selection quality carries over to training (Figure~\ref{fig:realworld_compare}(b)). CE+FWD rises from $31.28\%$ to $77.87\%$ on CLCIFAR-10 ($82.14\%$ with the adaptive cut) and from $14.89\%$ to $43.67\%$ on CLCIFAR-20, and every objective improves on both datasets, including SCL-NL, which needs no transition matrix and reaches $80.00\%$ on CLCIFAR-10 with the adaptive cut, above known-$Q$ FWD. The gain is largest for the recycling objectives, whose CE branch depends directly on detector precision (Theorem~\ref{thm:detect}). The margin adds no OT solve: it reuses the validation potential of the same transport problem and evaluates $M$ soft-minima per sample. Full results are given in Appendix~\ref{sec:ablation-clcifar}.

\paragraph{Is the validation set better spent on training?}
Trained from scratch with supervised CE, the $2{,}000$ clean labels reach $71.52\%$ on CIFAR-10 and $43.08\%$ on CIFAR-20, against $91.57\%$ and $75.34\%$ for CAVA, which uses them only for scoring (Figure~\ref{fig:scale_budget}(b)). CAVA keeps this advantage on a fixed representation that has seen no label of the pool: with CLIP features it reaches $87.91\%$ and $67.37\%$ (Appendix~\ref{app:reference_assumption}), and the corresponding results on CLIP features for every dataset are in Appendix~\ref{app:representations}.

\paragraph{Scalability on the large-scale CIFAR-10-1M benchmark.}
To evaluate CAVA beyond the regime where full-set OT valuation is feasible, we enlarge CIFAR-10 to one million training samples (CIFAR-10-1M) via label-preserving augmentation under the same diagonal-noise protocol ($\eta=25\%$). Task-wise decomposition splits the valuation into five tasks of 200{,}000 samples, and CAVA scores all one million samples in about 4 minutes. CAVA is best at the cut $p=25\%$ that equals the noise rate and reaches $86.04\%$ with SCL-NL after 20 epochs, against $0.12\%$ without pruning. As Figure~\ref{fig:scale_budget}(a) shows, it finishes this training before the valuation of TW-LAVA does, and TW-LAVA at its best cut $p=35\%$ reaches $67.76\%$ under the same 20-epoch schedule and $71.05\%$ after 200 epochs, against $90.51\%$ for CAVA. Full results are in Appendix~\ref{large-scale_dataset}.

\section{Conclusion}
\label{conclusion}
We have introduced CAVA, an OT-based data valuation framework for CLL. Starting from the task-wise transfer of LAVA, we showed that its noisy-group label cost carries an effective label weight $\lambda_y\alpha(\eta)$ that changes sign at the noise rate $\eta=1/M$: the direction of its score depends on the noise rate, and the method is blind at the crossover, which is the regime of human-annotated complementary labels. CAVA replaces this cost by a complementary label cost defined between clean classes and scores each sample by a label margin, the part of its transport potential that the given label contributes. The sign of the margin has the same meaning on every dataset, flags diagonal noise with an adaptive cut that needs no pruning ratio, and comes with a closed-form precision that gates the recycling of detected samples as ordinary-label supervision. On synthetic and real-world benchmarks, CAVA with recycling is the best of the compared valuation methods on every dataset, and its gains are largest on real complementary labels, where the transferred method fails. The detection itself is matched by a classifier fit on the validation set, and the transport potential adds nothing to it; what the analysis contributes is the reason transferred valuation fails, a cost under which it does not, and a recycling rule gated by a precision that can be estimated from the validation set.


\section{Limitations}
\label{limitations}
Our work has several limitations. First, we evaluate CAVA primarily on relatively small-scale CLL benchmarks; extending OT-based valuation to truly large-scale complementary-label datasets remains an open challenge~\citep{libcll_2024,aclimage2025}. Second, our framework assumes access to a clean validation set representative of the downstream task, a common assumption in data valuation research that may not always hold in practice; Appendix~\ref{app:reference_assumption} quantifies it. Third, recycling is most directly applicable to diagonal complementary-label noise, where a violated exclusion label reveals ordinary-label information. Addressing scalability, reducing reliance on clean validation sets, and generalizing recycling beyond diagonal noise are important future directions.

{
    \small
    \bibliographystyle{ieeenat_fullname}
    \bibliography{main}
}

\clearpage
\setcounter{page}{1}
\appendix
\setcounter{section}{0}
\setcounter{subsection}{0}
\setcounter{subsubsection}{0}
\renewcommand{\thesection}{\Alph{section}}
\renewcommand{\thesubsection}{\thesection.\arabic{subsection}}
\renewcommand{\thesubsubsection}{\thesubsection.\arabic{subsubsection}}

\onecolumn
\begin{center}
    {\Large \bf \thetitle \\[0.5em] Supplementary Material \par}
    \vskip 1.0em
\end{center}
\section{Background}
\label{app:background}
\subsection{Optimal Transport for Data Valuation}
\label{Appendix_A}

\label{sec:prelim_ot}
OT provides a principled framework for data valuation by quantifying how individual data points contribute to aligning empirical distributions in a task-aware and geometry-aware manner \citep{villani2008optimal, peyre2019computational,data-lin24t}. Let
\begin{equation} 
    \mu = \sum_{i=1}^{n} a_i \delta_{z_i}, \qquad 
    \nu = \sum_{j=1}^{m} b_j \delta_{z'_j}, 
\end{equation}
denote empirical source and target distributions, where $a\in\Delta_n$ and $b\in\Delta_m$ are probability weights on the simplex. Here, $z_i$ and $z'_j$ denote generic source and target samples. In CAVA, these correspond to complementary-labeled training samples and ordinary-labeled validation samples, respectively.
Given a ground cost $C(z_i,z'_j)$ that encodes task-relevant dissimilarity, the OT problem is defined as 
\begin{equation}
\label{eq:ot} 
    W_{C}(\mu,\nu) = \min_{\pi\in\Pi(a,b)} \sum_{i=1}^{n} \sum_{j=1}^{m} \pi_{ij} C(z_i,z'_j), 
\end{equation}
where the transport polytope is 
\begin{equation} 
    \Pi(a,b) := \Bigl\{ \pi\in\mathbb{R}_{+}^{n\times m} \;:\; \pi\mathbf{1}=a,\; \pi^{\top}\mathbf{1}=b \Bigr\}. 
\end{equation}
Each entry $\pi_{ij}$ represents how much probability mass from source sample $z_i$ is transported to target sample $z'_j$, providing a natural mechanism for attributing importance to individual data points.

To improve scalability and numerical stability, entropic regularization is commonly used, yielding the Sinkhorn-regularized OT objective \citep{cuturi2013sinkhorn}: 
\begin{equation} 
\label{eq:sinkhorn} 
    W_{C,\varepsilon_{\mathrm{OT}}}(\mu,\nu) 
    = \min_{\pi\in\Pi(a,b)} \sum_{i=1}^{n} \sum_{j=1}^{m} \pi_{ij} C(z_i,z'_j) 
    + \varepsilon_{\mathrm{OT}} \sum_{i=1}^{n} \sum_{j=1}^{m} \pi_{ij}\log \pi_{ij}, 
\end{equation} 
where $\varepsilon_{\mathrm{OT}}>0$ controls the smoothness of the transport plan. This formulation can be computed efficiently via Sinkhorn iterations and admits dual potentials whose derivatives with respect to the probability weights provide interpretable data values \citep{genevay2018learning}.

Let $u^*\in\mathbb{R}^{n}$ and $w^*\in\mathbb{R}^{m}$ denote the optimal source-side and target-side dual potentials, respectively. By sensitivity analysis of the OT objective, the derivative of the regularized OT value with respect to the source weights is given by the source-side dual potential: 
\begin{equation} 
    \nabla_{a} W_{C,\varepsilon_{\mathrm{OT}}}(\mu,\nu) = u^*. 
    \label{dual_potential}
\end{equation} 
Thus, $u_i^*$ measures how the OT discrepancy changes when the probability mass of source sample $z_i$ is perturbed.

However, OT dual potentials are unique only up to an additive constant because one of the marginal constraints in the transport polytope is redundant. Following the calibrated OT valuation convention used in LAVA \citep{just2023lava}, we remove this degree of freedom by using a zero-sum calibrated score. For source sample $z_i$, the calibrated OT sensitivity score is defined as 
\begin{equation} 
\label{eq_lava_calibrated_gradient} 
v_i = u_i^* - \frac{1}{n-1} \sum_{\substack{k=1 \\ k\neq i}}^{n} u_k^*. \end{equation} 
This calibrated score estimates how much source sample $z_i$ contributes to the OT discrepancy between the source and target distributions. A large positive score indicates that increasing the mass of $z_i$ would increase the OT discrepancy, suggesting that the sample is poorly aligned with the target distribution. Conversely, a smaller score indicates better alignment. 

In data valuation, these calibrated OT sensitivity scores can be used to identify samples that increase the mismatch between training and validation distributions. Removing or downweighting samples with large positive scores can reduce the OT discrepancy and improve alignment with the validation distribution, as demonstrated in LAVA \citep{just2023lava}. CAVA keeps this transport problem and replaces two of its ingredients for complementary-label learning, the label cost and the score (Section~\ref{sec:cava_refmargin}); the calibrated score $v_i$ is the score of the TW-LAVA baseline (Eq.~\eqref{eq:cava_score}). Specifically, the generic source sample $z_i$ becomes a complementary-labeled training sample $(\xx_i,\bar y_i)$, the target sample $z'_j$ becomes an ordinary-labeled validation sample $(\xx_j^v,y_j^v)$, and the generic cost $C$ is the ground cost of Eq.~\eqref{eq:cava_cost}. 

Despite its effectiveness, Sinkhorn-based OT still requires storing and operating on the pairwise cost matrix, leading to $\mathcal{O}(nm)$ memory complexity, or $\mathcal{O}(N^2)$ when both empirical distributions have size $N$. Recent works have addressed this limitation through minibatch or hierarchical OT approximations \citep{kesslersava,zhu2025kairos}. CAVA instead uses a task-wise decomposition, solving smaller OT problems between each complementary-label task and the clean validation set, which improves computational feasibility while preserving the distributional interpretation needed for data valuation. 

Compared to Shapley-value-based data valuation methods \citep{ghorbani2019data, jia2019efficient}, OT-based approaches avoid combinatorial subset enumeration and instead leverage distributional geometry to implicitly capture higher-order interactions. This makes OT particularly suitable for large-scale and continuous data settings, where valuation must account for domain shift, covariate imbalance, and sample diversity. Recent work has demonstrated the effectiveness of OT-based data valuation for dataset pruning, reweighting, curriculum learning, and active data acquisition \citep{wang2023data, zhao2024optimal,just2023lava,kesslersava,chi2025shapleyguided}, positioning OT as a scalable and principled foundation for data-centric machine learning.

\subsection{Diagonal Complementary-Label Noise}
\label{diagonal-noise}
Figure~\ref{fig1} shows how diagonal noise appears in the empirical transition matrix $P(\bar Y\mid Y)$: valid complementary labels fill the off-diagonal entries, whereas a label that names the true class places mass on the diagonal and asks the learner to suppress the correct class.

\begin{figure}[htb]
\centering
\includegraphics[width=0.5\linewidth]{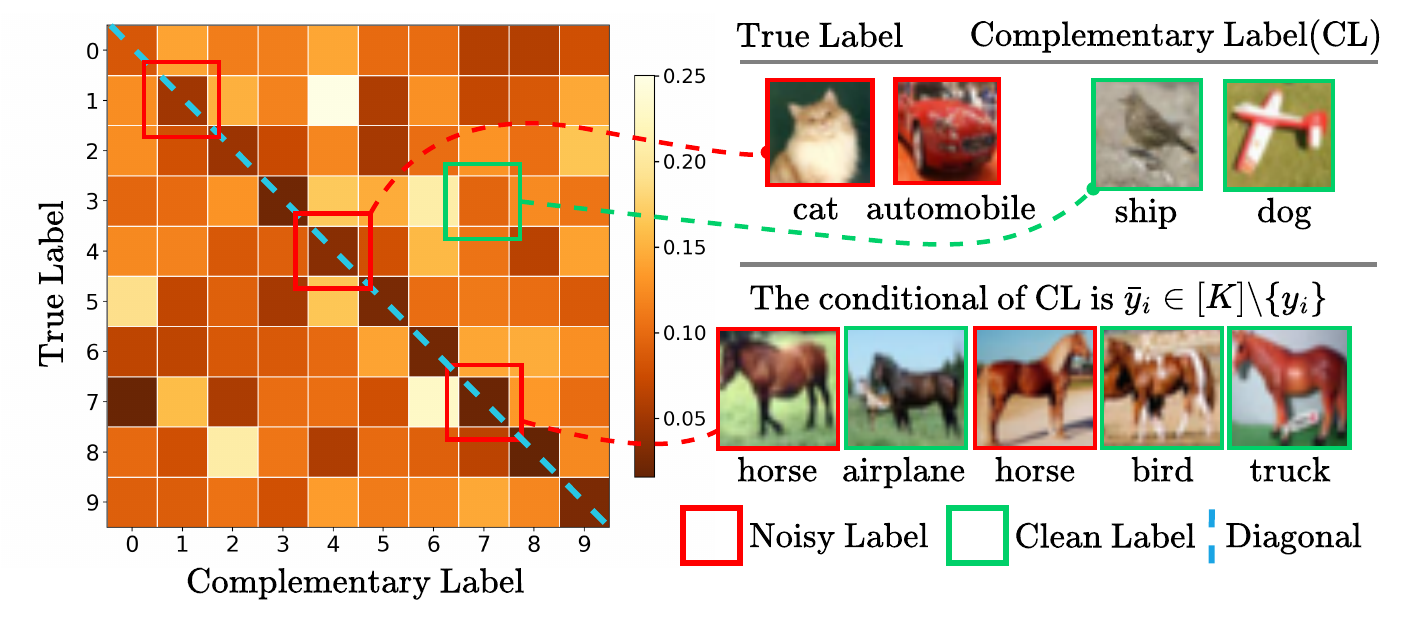}
\caption{
Complementary-label noise in CLL. When the complementary label coincides with the true label, i.e., $\bar{y}_i=y_i$, the empirical transition matrix contains diagonal mass, creating contradictory exclusion constraints.}
\label{fig1}
\end{figure}

\subsection{Notation}
\label{app:notation}
Table~\ref{tab:cava_notation} collects the main symbols.

\begin{table}[ht]
\centering
\caption{Notation.}
\label{tab:cava_notation}
\small
\begin{tabular}{@{}p{0.36\textwidth}p{0.60\textwidth}@{}}
\toprule
\textbf{Symbol} & \textbf{Meaning} \\
\midrule
$\mathcal{X},\ \mathcal{Y}=\{1,\ldots,M\}$ & Input space and label space of $M$ classes \\
$\bar{\mathcal{D}}=\{(\xx_i,\bar y_i)\}_{i=1}^{N}$, $y_i$ & Pool with given labels, and the unobserved true labels \\
$\mathcal{N}=\{i:\bar y_i=y_i\}$, $\eta=\mathbb{P}(\bar y=y)$ & Diagonal-noise set and noise rate \\
$\mathcal{D}_v=\{(\xx^v_j,y^v_j)\}_{j=1}^{N_v}$ & Clean validation set with ordinary labels \\
$I_t$, $T$ & Task $t$ (a subset of the pool) and the number of tasks \\
$\phi$, $C_{\mathrm{feat}}$ & Fixed feature extractor and feature cost (Eq.~\eqref{eq:cava_feat_cost}) \\
$\lambda_y$, $\varepsilon=\varepsilon_{\mathrm{OT}}$ & Label weight and entropic regularization \\
$W^{\mathrm{grp}}$ & Noisy-group label cost of TW-LAVA (Eq.~\eqref{eq:cava_label_cost}) \\
$W^{\mathrm{cl}}$, $W^{\mathrm{cl}}_{D}$ & Complementary label cost ($0/1$) and its class-distance variant \\
$\alpha(\eta)=(M\eta-1)/(M-1)$ & Lean of a noisy label group; effective label weight $\lambda_y\alpha(\eta)$ (Theorem~\ref{thm:flip}) \\
$u^{(t),*}$, $w^{(t),*}$ & Optimal pool-side and validation-side potentials of task $t$ \\
$v_i^{(t)}$ & Calibrated score of TW-LAVA (Eq.~\eqref{eq:cava_score}) \\
$f_i(c)$, $p_i(k)$, $s_i$ & Potential under label $c$, transport posterior, label margin (Eqs.~\eqref{eq:cava_fc}--\eqref{eq:cava_margin}) \\
$F=\{i:s_i\le0\}$, $h_i=\arg\max_k p_i(k)$ & Adaptive-cut flagged set and posterior classifier \\
$p$ & Pruning ratio of the fixed cut \\
$I^{\mathrm{rec}}$, $I^{\mathrm{keep}}$, $\kappa$ & Recycled and retained sets, recycling weight (Eq.~\eqref{eq:cava_mixed_objective}) \\
$\pi_{\mathrm{det}}$, $B$, $\Delta_{\mathrm{gen}}$ & Precision of the recycled set, loss bound, generalization term (Theorem~\ref{thm:recycle_truelabel}) \\
\bottomrule
\end{tabular}
\end{table}

\FloatBarrier
\section{Theory}
\label{Appendix_B}
This section proves the results of Section~\ref{cava}. Appendix~\ref{app:refmargin_proofs} states and proves the flip law, the direction of the label margin, and the precision of the adaptive cut; Appendix~\ref{app:recycle} proves the recycling bound that turns this precision into a risk guarantee.


\subsection{Analysis of the Complementary Label Cost and the Label Margin}
\label{app:refmargin_proofs}

This subsection gives the full statements and the proofs of the three results stated in short form with the method: the flip law (Theorem~\ref{thm:flip}), the decomposition and direction of the margin (Theorem~\ref{thm:margin}), and the detection-to-recycling bound (Theorem~\ref{thm:detect}). Throughout, $d_{ij}=C_{\mathrm{feat}}(\xx_i,\xx^v_j)$ denotes the feature cost of Eq.~\eqref{eq:cava_feat_cost}, the validation weights are $b_j=1/N_v$, the feature weight is $\lambda_\xx=1$ (Section~\ref{sec:cava_valuation}), and the entropic problem of Eq.~\eqref{eq:cava_entropic_ot} with ground cost $C_{ij}$ has optimal potentials $(u,w)$, written $u^{(t),*}$ and $w^{(t),*}$ in the main text, satisfying
\begin{equation}
u_i=-\varepsilon_{\mathrm{OT}}\log\sum_j b_j\exp\bigl((w_j-C_{ij})/\varepsilon_{\mathrm{OT}}\bigr),
\label{eq:app_f_softmin}
\end{equation}
so the counterfactual potentials of Eq.~\eqref{eq:cava_fc} satisfy $f_i(\bar y_i)=u_i$.

\begin{assumption}[Uniform diagonal-noise model]
\label{ass:noise_model}
Classes are balanced. Given the true label $y$, the observed label is $\bar y=y$ with probability $\eta$ and otherwise uniform on the other $M-1$ labels, independently of $\xx$. This is the synthetic protocol of Section~\ref{experiments}. Write $P_k$ for the feature distribution of class $k$, $\bar P=\frac1M\sum_kP_k$, and $\alpha(\eta)=(M\eta-1)/(M-1)$.
\end{assumption}

\begin{lemma}[Group composition and shifts]
\label{lem:groupshift}
\textit{(a) Under Assumption~\ref{ass:noise_model}, the feature distribution of the pool samples with given label $c$ is $\bar{\mathcal{D}}_{c}=(1-\alpha)\bar P+\alpha P_c$ with $\alpha=\alpha(\eta)$.
(b) If $C'_{ij}=C_{ij}+r_i+t_j$ for a row term $r$ and a column term $t$, the entropic problems with costs $C$ and $C'$ have the same optimal plan, and $(u_i+r_i,\,w_j+t_j)$ are optimal potentials for $C'$. If $r$ is constant, the calibrated score of Eq.~\eqref{eq:cava_score} and the margin of Eq.~\eqref{eq:cava_margin} are unchanged.}
\end{lemma}

\begin{proof}
We compute the law of the given label, the posterior of the true label, and assemble the mixture; part (b) is a direct check of the optimality conditions.

\textbf{Step 1. Marginal of the given label.}
By Assumption~\ref{ass:noise_model} and balanced classes,
\begin{equation}
\mathbb{P}(\bar y=c)=\frac1M\,\eta+\sum_{k\neq c}\frac1M\,\frac{1-\eta}{M-1}=\frac1M .
\label{eq:app_marginal}
\end{equation}

\textbf{Step 2. Posterior of the true label.}
By the Bayes rule and Eq.~\eqref{eq:app_marginal},
\begin{equation}
\mathbb{P}(y=c\mid\bar y=c)=\eta,\qquad \mathbb{P}(y=k\mid\bar y=c)=\frac{1-\eta}{M-1}\quad(k\neq c).
\label{eq:app_posterior_y}
\end{equation}

\textbf{Step 3. Mixture form.}
The label draw is independent of $\xx$ given $y$, so by Eq.~\eqref{eq:app_posterior_y}
\[
\bar{\mathcal{D}}_{c}=\eta\,P_c+\frac{1-\eta}{M-1}\sum_{k\neq c}P_k
=\Bigl[\eta-\frac{1-\eta}{M-1}\Bigr]P_c+\frac{M(1-\eta)}{M-1}\,\bar P .
\]
The coefficient of $P_c$ is $(M\eta-1)/(M-1)=\alpha$ and the coefficient of $\bar P$ is $1-\alpha$, which proves (a).

\textbf{Step 4. Same plan under shifts.}
The optimal plan for $C$ can be rewritten with the shifted potentials:
\[
\pi_{ij}=a_ib_j\exp\!\Bigl(\tfrac{u_i+w_j-C_{ij}}{\varepsilon_{\mathrm{OT}}}\Bigr)
=a_ib_j\exp\!\Bigl(\tfrac{(u_i+r_i)+(w_j+t_j)-C'_{ij}}{\varepsilon_{\mathrm{OT}}}\Bigr).
\]
It has marginals $a$ and $b$, and the entropic optimum for $C'$ is the unique plan of this Gibbs form with these marginals, so $\pi$ is also optimal for $C'$ with potentials $(u+r,\,w+t)$.

\textbf{Step 5. Scores under shifts.}
The calibrated score changes by $r_i-\frac{1}{|I_t|-1}\sum_{k\neq i}r_k$, which vanishes when $r$ is constant. For the margin, the shifted potential $w'_j=w_j+t_j$ absorbs the column term, so the counterfactual potentials satisfy $f'_i(c)=f_i(c)+r_i$ for every $c$, and the difference in Eq.~\eqref{eq:cava_margin} is unchanged. This proves (b).
\end{proof}

\paragraph{Full statement of Theorem~\ref{thm:flip}.}
{\itshape Consider population groups under Assumption~\ref{ass:noise_model}.
\begin{enumerate}[label=(\roman*), leftmargin=2.2em, itemsep=2pt, topsep=3pt]
\item For any label divergence: at $\eta=1/M$, $\bar{\mathcal{D}}_{c}=\bar P$ for every $c$, the ground cost no longer involves $\bar y_i$, and no score computed from the transport problem can distinguish diagonal noise from valid samples.
\item If the label divergence is a squared mean-embedding distance, then
\[
W^{\mathrm{grp}}(c,k)=\alpha(\eta)\,W^{\mathrm{cl}}_{D}(c,k)+r_c+t_k
\]
for row and column shifts $r_c,t_k$, where $W^{\mathrm{cl}}_{D}(c,k)$ is the same divergence computed between clean classes, the class-distance variant of the complementary label cost of Section~\ref{sec:cava_refmargin}. The TW-LAVA problem therefore has the same optimal plan as the problem under the complementary label cost with label weight $\lambda_y\,\alpha(\eta)$, and with equal class radii the same calibrated score and label margin.
\end{enumerate}}

\begin{proof}[\textbf{Proof of Theorem \ref{thm:flip}}]
We prove part (i) from Lemma~\ref{lem:groupshift}(a), obtain the shift identity of part (ii) by expanding two squared norms, and conclude with Lemma~\ref{lem:groupshift}(b).

\textbf{Step 1. Part (i): the given label leaves the cost.}
By Lemma~\ref{lem:groupshift}(a) with $\alpha=0$, $\bar{\mathcal{D}}_{c}=\bar P$ for every $c$. Hence
\[
W^{\mathrm{grp}}(c,k)=D(\bar P,P_k),\qquad C_{ij}=d_{ij}+\lambda_y\,D(\bar P,P_{y^v_j}),
\]
and $\bar y_i$ does not appear in the ground cost. The plan, the potentials, and every score computed from the transport problem are therefore functions of the features and of the validation set only. Under Assumption~\ref{ass:noise_model} the indicators $\mathbf{1}[\bar y_i=y_i]$ are i.i.d.\ Bernoulli$(\eta)$ and independent of the features given the true labels, so every score has the same distribution on diagonal-noise and on valid samples.

\textbf{Step 2. Mean embedding of the group.}
The mean embedding is linear in the distribution. Applying it to Lemma~\ref{lem:groupshift}(a), with $u_k=\mu_{P_k}-\mu_{\bar P}$,
\begin{equation}
\mu_{\bar{\mathcal{D}}_{c}}-\mu_{\bar P}=\alpha\,(\mu_{P_c}-\mu_{\bar P})=\alpha\,u_c .
\label{eq:app_mean_group}
\end{equation}

\textbf{Step 3. Expanding the squares.}
Writing $\mu_{\bar{\mathcal{D}}_{c}}-\mu_{P_k}=(\mu_{\bar{\mathcal{D}}_{c}}-\mu_{\bar P})-(\mu_{P_k}-\mu_{\bar P})$ and using Eq.~\eqref{eq:app_mean_group},
\begin{align}
W^{\mathrm{grp}}(c,k)&=\|\alpha u_c-u_k\|^2=\alpha^2\|u_c\|^2-2\alpha\langle u_c,u_k\rangle+\|u_k\|^2,
\label{eq:app_sq1}\\
\alpha\,W^{\mathrm{cl}}_{D}(c,k)&=\alpha\|u_c-u_k\|^2=\alpha\|u_c\|^2-2\alpha\langle u_c,u_k\rangle+\alpha\|u_k\|^2 .
\label{eq:app_sq2}
\end{align}
Subtracting Eq.~\eqref{eq:app_sq2} from Eq.~\eqref{eq:app_sq1},
\[
W^{\mathrm{grp}}(c,k)-\alpha\,W^{\mathrm{cl}}_{D}(c,k)=-\alpha(1-\alpha)\|u_c\|^2+(1-\alpha)\|u_k\|^2,
\]
a term depending on the row only plus a term depending on the column only.

\textbf{Step 4. Plans and scores.}
In the ground cost the two extra terms become $r_i=-\lambda_y\alpha(1-\alpha)\|u_{\bar y_i}\|^2$ and $t_j=\lambda_y(1-\alpha)\|u_{y^v_j}\|^2$. By Lemma~\ref{lem:groupshift}(b) the TW-LAVA problem has the same optimal plan as the cost $d_{ij}+\lambda_y\alpha\,W^{\mathrm{cl}}_{D}(\bar y_i,y^v_j)$, which is the problem under the complementary label cost with label weight $\lambda_y\alpha(\eta)$. If $\|u_k\|=r$ for all $k$ (equal class radii), then $r_i$ is constant and the second part of Lemma~\ref{lem:groupshift}(b) gives equality of the calibrated scores and of the margins.
\end{proof}

Figure~\ref{fig:flip} shows the flip of Theorem~\ref{thm:flip} as measured on CIFAR-10.

\begin{figure}[htb]
\centering
\includegraphics{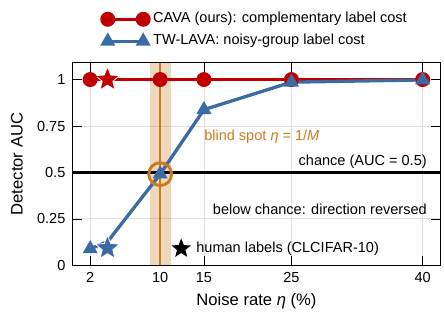}
\caption{\textbf{The direction flip, measured.} Detector AUC against $\eta$ on CIFAR-10; both scores use the same features, $\varepsilon_{\mathrm{OT}}=0.01$, $\lambda_y=1$, and the same $N_v=2{,}000$ validation images. An AUC of $0.5$ is chance. The noisy-group label cost is at chance exactly at the blind spot $\eta=1/M$ (orange) predicted by Theorem~\ref{thm:flip}, and below chance for smaller $\eta$, where it ranks the noise in the wrong direction. The complementary label cost stays at the top for every $\eta$. Stars: human-labeled CLCIFAR-10 ($3.98\%$ errors).}
\label{fig:flip}
\end{figure}

\paragraph{Full statement of Theorem~\ref{thm:margin}.}
{\itshape For the $0/1$ complementary label cost, any validation potential $w$, any pool, any noise rate and any $M$, define the OT posterior and the function $\psi$ by
\[
p_i(k)\propto\sum_{j:\,y^v_j=k}b_j\,e^{(w_j-C_{\mathrm{feat}}(\xx_i,\xx_j^v))/\varepsilon_{\mathrm{OT}}},
\qquad
\psi(p)=-\varepsilon_{\mathrm{OT}}\log\bigl(\rho+(1-\rho)p\bigr),
\quad \rho=e^{-\lambda_y/\varepsilon_{\mathrm{OT}}}.
\]
Then:
\begin{enumerate}[label=(\roman*), leftmargin=2.2em, itemsep=2pt, topsep=3pt]
\item $f_i(c)=T_i+\psi(p_i(c))$, where the typicality term $T_i$ does not depend on $c$.
\item $s_i=\psi(p_i(\bar y_i))-\psi(\max_{c\neq\bar y_i}p_i(c))$ and $|s_i|\le\lambda_y$.
\item $s_i\le0$ if and only if $\bar y_i\in\arg\max_k p_i(k)$.
\item If the label weight is replaced by any real $\lambda'$, the label part of $f_i(\bar y_i)$ is strictly decreasing in $p_i(\bar y_i)$ for $\lambda'>0$, constant at $\lambda'=0$, and strictly increasing for $\lambda'<0$.
\end{enumerate}}

\begin{proof}[\textbf{Proof of Theorem \ref{thm:margin}}]
We split the sum that defines $f_i(c)$ by validation class, factor out the normalizer, and read off the four statements. Recall $p_i(k)=\frac{1}{Z_i}\sum_{j:y^v_j=k}b_j e^{(w_j-d_{ij})/\varepsilon_{\mathrm{OT}}}$ with $Z_i=\sum_j b_j e^{(w_j-d_{ij})/\varepsilon_{\mathrm{OT}}}$, $T_i=-\varepsilon_{\mathrm{OT}}\log Z_i$, and $\rho=e^{-\lambda_y/\varepsilon_{\mathrm{OT}}}$.

\textbf{Step 1. Splitting the sum.}
With $W^{\mathrm{cl}}(c,k)=\mathbf{1}[c\neq k]$, the validation samples of class $c$ pay no label cost and all others acquire the factor $\rho$:
\begin{equation}
\sum_j b_j\,e^{(w_j-d_{ij}-\lambda_y\mathbf{1}[c\neq y^v_j])/\varepsilon_{\mathrm{OT}}}
=\sum_{j:\,y^v_j=c}b_j\,e^{(w_j-d_{ij})/\varepsilon_{\mathrm{OT}}}
+\rho\sum_{j:\,y^v_j\neq c}b_j\,e^{(w_j-d_{ij})/\varepsilon_{\mathrm{OT}}}.
\label{eq:app_split}
\end{equation}

\textbf{Step 2. Factoring the normalizer.}
The first sum in Eq.~\eqref{eq:app_split} equals $Z_i\,p_i(c)$ and the second equals $Z_i\,(1-p_i(c))$, so
\begin{equation}
\sum_j b_j\,e^{(w_j-d_{ij}-\lambda_y\mathbf{1}[c\neq y^v_j])/\varepsilon_{\mathrm{OT}}}=Z_i\bigl[\rho+(1-\rho)\,p_i(c)\bigr].
\label{eq:app_factored}
\end{equation}
Taking $-\varepsilon_{\mathrm{OT}}\log$ of Eq.~\eqref{eq:app_factored} gives $f_i(c)=T_i+\psi(p_i(c))$, which is (i).

\textbf{Step 3. The margin.}
Since $\psi$ is strictly decreasing, the cheapest alternative label is the one with the largest posterior:
\[
\min_{c\neq\bar y_i}f_i(c)=T_i+\psi\Bigl(\max_{c\neq\bar y_i}p_i(c)\Bigr).
\]
Subtracting this from $f_i(\bar y_i)=T_i+\psi(p_i(\bar y_i))$ cancels $T_i$ and yields the closed form of (ii). Because $0\le\psi\le\lambda_y$ on $[0,1]$, $|s_i|\le\lambda_y$.

\textbf{Step 4. Direction.}
By monotonicity of $\psi$,
\[
s_i\le0
\iff \psi\bigl(p_i(\bar y_i)\bigr)\le\psi\Bigl(\max_{c\neq\bar y_i}p_i(c)\Bigr)
\iff p_i(\bar y_i)\ge\max_{c\neq\bar y_i}p_i(c),
\]
which is (iii). No property of the pool, of $\eta$, or of $M$ was used.

\textbf{Step 5. Sign of the weight.}
The computation leading to Eq.~\eqref{eq:app_factored} does not use the sign of the weight. With $\rho'=e^{-\lambda'/\varepsilon_{\mathrm{OT}}}$,
\[
f_i(\bar y_i)=T_i-\varepsilon_{\mathrm{OT}}\log\bigl(\rho'+(1-\rho')\,p_i(\bar y_i)\bigr).
\]
The map $p\mapsto\rho'+(1-\rho')p$ is increasing when $\rho'<1$, constant when $\rho'=1$, and decreasing when $\rho'>1$, and $\rho'<1$ exactly when $\lambda'>0$. Composing with $-\varepsilon_{\mathrm{OT}}\log$ reverses the monotonicity, which gives (iv).
\end{proof}

\paragraph{Full statement of Theorem~\ref{thm:detect}.}
{\itshape Let $h_i=\arg\max_k p_i(k)$ be the classifier induced by the OT posterior of Theorem~\ref{thm:margin} (assumed unique). Let $\mathcal N=\{i:\bar y_i=y_i\}$ be the diagonal-noise set of the pool of size $N$, $\hat\eta=|\mathcal N|/N$, $F=\{i:s_i\le0\}$, $e_{\mathcal N}$ the error rate of $h$ on $\mathcal N$, and $q_{V}$ the fraction of valid samples with $h_i=\bar y_i$. With the $0/1$ complementary label cost:
\begin{enumerate}[label=(\roman*), leftmargin=2.2em, itemsep=2pt, topsep=3pt]
\item The recall and the precision of $F$ are
\[
\mathrm{recall}=1-e_{\mathcal N},
\qquad
\mathrm{precision}=\frac{\hat\eta\,(1-e_{\mathcal N})}{\hat\eta\,(1-e_{\mathcal N})+(1-\hat\eta)\,q_{V}} .
\]
\item Under the uniform noise model with rate $\eta$, if $h$ has overall error $e$ and is conditionally independent of the given label, then $q_{V}=e/(M-1)$ and the flagged-set purity
\[
\pi=\frac{\eta\,(1-e)}{\eta\,(1-e)+(1-\eta)\,e/(M-1)}
\]
exceeds the base rate $\eta$ whenever $e\le1-1/M$.
\item Training the CE branch on $F$ satisfies, with probability at least $1-\delta$,
\[
R_{\mathrm{true}}(\widehat g)-R_{\mathrm{true}}(g^\star)\le 2B\,(1-\mathrm{precision})+\Delta_{\mathrm{gen}}(|F|,\delta),
\]
by inserting the precision of (i) into Theorem~\ref{thm:recycle_truelabel}.
\end{enumerate}}

\begin{proof}[\textbf{Proof of Theorem \ref{thm:detect}}]
We express the flagging rule through the posterior classifier $h$, count the flagged samples on the noise set and on the valid set, specialize to uniform complementary labels, and insert the precision into Theorem~\ref{thm:recycle_truelabel}. Let $V$ denote the complement of $\mathcal N$.

\textbf{Step 1. Flagging and the classifier.}
By Theorem~\ref{thm:margin}(iii) and uniqueness of the maximizer,
\begin{equation}
i\in F\iff h_i=\bar y_i .
\label{eq:app_flag}
\end{equation}

\textbf{Step 2. Count on the noise set.}
For $i\in \mathcal N$ we have $\bar y_i=y_i$, so by Eq.~\eqref{eq:app_flag} the sample is flagged exactly when $h_i=y_i$. Hence $|F\cap \mathcal N|=|\mathcal N|\,(1-e_{\mathcal N})$, and the recall is $1-e_{\mathcal N}$.

\textbf{Step 3. Count on the valid set.}
For $i\in V$, Eq.~\eqref{eq:app_flag} reads $h_i=\bar y_i$, so $|F\cap V|=|V|\,q_{V}$. Since $\bar y_i\neq y_i$ on $V$, such a sample is an error of $h$.

\textbf{Step 4. Precision.}
Dividing the two counts by $N$,
\[
\mathrm{precision}=\frac{|F\cap \mathcal N|}{|F\cap \mathcal N|+|F\cap V|}
=\frac{\hat\eta(1-e_{\mathcal N})}{\hat\eta(1-e_{\mathcal N})+(1-\hat\eta)\,q_{V}},
\]
which completes (i).

\textbf{Step 5. Uniform complementary labels.}
Given $(\xx_i,y_i)$ and $i\in V$, the label $\bar y_i$ is uniform on the $M-1$ labels other than $y_i$ and independent of $h_i$, so
\[
\mathbb{P}\bigl(h_i=\bar y_i\mid i\in V,\,\xx_i,\,y_i\bigr)=\frac{\mathbf{1}[h_i\neq y_i]}{M-1}.
\]
Averaging over $(\xx_i,y_i)$ gives $q_{V}=e/(M-1)$. Given $i\in \mathcal N$, the Bernoulli draw is independent of $h_i$, so $\mathbb{P}(i\in F\mid i\in \mathcal N)=\mathbb{P}(h_i=y_i)=1-e$. The Bayes rule with prior $\eta$ and these two conditional probabilities gives the purity $\pi$ of (ii).

\textbf{Step 6. Enrichment condition.}
$\pi\ge\eta$ holds if and only if
$1-e\ge\eta(1-e)+(1-\eta)\,e/(M-1)$, that is,
$(1-\eta)(1-e)\ge(1-\eta)\,e/(M-1)$. Dividing by $(1-\eta)>0$ and rearranging gives $e\le1-1/M$, which completes (ii).

\textbf{Step 7. Recycling.}
Theorem~\ref{thm:recycle_truelabel} bounds the excess risk of the empirical minimizer on any recycled set of size $m$ and precision $\pi_{\mathrm{det}}$ by $2B(1-\pi_{\mathrm{det}})+\Delta_{\mathrm{gen}}(m,\delta)$. Taking the set to be $F$ and inserting the precision of (i) gives (iii).
\end{proof}

\begin{remark}[Scope]
\label{rem:refmargin_scope}
Theorem~\ref{thm:flip}(ii) covers mean-embedding divergences (squared distance between class means, or squared MMD). The pipeline uses a Sinkhorn divergence for the class-distance variant $W^{\mathrm{cl}}_{D}$, for which we observe the same sign change with a measured proportionality factor of about $0.32\,\alpha(\eta)$; we state the exact law for mean embeddings and report the Sinkhorn case numerically. Theorems~\ref{thm:margin} and~\ref{thm:detect} are exact for the default $0/1$ complementary label cost, which performed within $0.01$ AUC of the class-distance variant in every test.
\end{remark}

\subsection{Recycling Risk Bound}
\label{app:recycle}

This subsection proves the recycling bound used by Theorem~\ref{thm:detect}(iii). It shows that the benefit of recycling is governed by a single quantity, the \emph{precision} of the flagged set.

\paragraph{Diagonal-noise model and detector.}
Under the diagonal-noise setting, a training sample is either
\emph{clean} ($\bar y_i \neq y_i$, a valid exclusion) or \emph{noisy}
($\bar y_i=y_i$, a violated exclusion). CAVA flags a recycled set
$\widehat{\mathcal N}=I^{\mathrm{rec}}$ (Section~\ref{sec:cava_recycling}) and trains it with
cross-entropy treating $\bar y$ as an ordinary label. Write
$m=|\widehat{\mathcal N}|$ and define the detector \emph{precision}
\begin{equation}
\pi_{\mathrm{det}} \;:=\;\frac{|\{i\in\widehat{\mathcal N}:\bar y_i=y_i\}|}{m}.
\end{equation}
For a correctly flagged (noisy) sample $\bar y_i=y_i$, so the CE term is supervised by the \emph{true} label; for a falsely flagged (clean) sample $\bar y_i\neq y_i$, the CE term is supervised by a wrong label. Thus the recycled branch is exactly an ordinary-label dataset corrupted at rate $(1-\pi_{\mathrm{det}})$.

\begin{assumption}[Bounded loss and capacity]
\label{ass:bdd}
The per-example cross-entropy is bounded, $\mathcal{L}_{\mathrm{CE}}(y,g(\xx))\in[0,B]$ (e.g.\ via logit clipping, temperature scaling, or label smoothing), and the hypothesis class $\mathcal G$ has Rademacher complexity $\mathfrak R_m(\mathcal G)$ on $m$ samples.
\end{assumption}

\begin{theorem}[Recycling excess risk]
\label{thm:recycle_truelabel}
\textit{Let $\widehat g\in\arg\min_{g\in\mathcal G}\widehat R_{\mathrm{rec}}(g)$ minimize the \emph{observed} recycled risk
$\widehat R_{\mathrm{rec}}(g)=\tfrac1m\sum_{i\in\widehat{\mathcal N}}\mathcal{L}_{\mathrm{CE}}\big(\bar y_i,g(\xx_i)\big)$ over the recycled set $\widehat{\mathcal N}=I^{\mathrm{rec}}$ of size $m$ and precision $\pi_{\mathrm{det}}$, and let
$g^{\star}\in\arg\min_{g\in\mathcal G}R_{\mathrm{true}}(g)$ minimize the true-label risk $R_{\mathrm{true}}(g)=\mathbb E_{(\xx,y)}[\mathcal{L}_{\mathrm{CE}}(y,g(\xx))]$. Under Assumption~\ref{ass:bdd}, with probability at least $1-\delta$,}
\begin{equation}
\label{eq:recycle_truelabel}
R_{\mathrm{true}}(\widehat g)-R_{\mathrm{true}}(g^{\star})
\le
\underbrace{2(1-\pi_{\mathrm{det}})\,B}_{\text{noise bias (irreducible)}}
+
\underbrace{4B\,\mathfrak R_m(\mathcal G)+2B\sqrt{\tfrac{\ln(2/\delta)}{2m}}}_{\Delta_{\mathrm{gen}}(m,\delta)\;\to\;0}.
\end{equation}
\end{theorem}

\begin{proof}
[\textbf{Proof of Theorem \ref{thm:recycle_truelabel}}]
All statements below hold under Assumption~\ref{ass:bdd}, with probability at least $1-\delta$.

\textbf{Step 1. The observed risk tracks the clean risk.}
Let $\widehat R_{\mathrm{clean}}(g)=\tfrac1m\sum_{i\in\widehat{\mathcal N}}
\mathcal{L}_{\mathrm{CE}}(y_i,g(\xx_i))$ denote the empirical risk under the (unobserved) true labels. For a correctly flagged sample, $\bar y_i=y_i$, so the corresponding summands of $\widehat R_{\mathrm{rec}}$ and $\widehat R_{\mathrm{clean}}$ agree. For each of the $(1-\pi_{\mathrm{det}})m$ falsely flagged samples, the two summands differ by at most $B$, since both losses lie in $[0,B]$. Averaging over the $m$ samples gives, for all $g$,
\begin{equation}
\Big|\widehat R_{\mathrm{rec}}(g)-\widehat R_{\mathrm{clean}}(g)\Big|
\le (1-\pi_{\mathrm{det}})B.
\label{eq:obsclean}
\end{equation}

\textbf{Step 2. Uniform convergence on the clean risk.}
Standard uniform convergence \citep{mohri2018foundations} gives, with probability at least $1-\delta$,
\begin{equation}
\sup_{g\in\mathcal G}\Big(R_{\mathrm{true}}(g)-\widehat R_{\mathrm{clean}}(g)\Big)
\le \Delta_{\mathrm{gen}}(m,\delta),
\end{equation}
together with the analogous lower-tail bound.

\textbf{Step 3. Chaining the two estimates.}
Combining Steps 1 and 2,
\begin{align*}
R_{\mathrm{true}}(\widehat g)
&\le \widehat R_{\mathrm{clean}}(\widehat g)+\Delta_{\mathrm{gen}}
 \le \widehat R_{\mathrm{rec}}(\widehat g)+(1-\pi_{\mathrm{det}})B+\Delta_{\mathrm{gen}}
 \le \widehat R_{\mathrm{rec}}(g^{\star})+(1-\pi_{\mathrm{det}})B+\Delta_{\mathrm{gen}}\\
&\le \widehat R_{\mathrm{clean}}(g^{\star})+2(1-\pi_{\mathrm{det}})B+\Delta_{\mathrm{gen}}
 \le R_{\mathrm{true}}(g^{\star})+2(1-\pi_{\mathrm{det}})B+2\Delta_{\mathrm{gen}},
\end{align*}
where the first and fifth inequalities use Step~2, the second and fourth use Eq.~\eqref{eq:obsclean}, and the third uses that $\widehat g$ minimizes $\widehat R_{\mathrm{rec}}$. Rearranging completes the proof.
\end{proof}

Theorem~\ref{thm:recycle_truelabel} isolates the recycled branch by analyzing the minimizer of $\widehat R_{\mathrm{rec}}$ alone. The next corollary extends the same argument to the mixed objective actually optimized in Eq.~\eqref{eq:cava_mixed_objective}.

\begin{corollary}[Recycling within the mixed objective]
\label{cor:mixed}
Extend Assumption~\ref{ass:bdd} by $\mathcal{L}_{\mathrm{comp}}(\bar y,g(\xx))\in[0,B]$. Let $\widehat R_{\mathrm{keep}}$ denote the empirical complementary risk on the $n_k=|I^{\mathrm{keep}}|$ retained samples, let
$\widehat g\in\arg\min_{g\in\mathcal G}\;\kappa\,\widehat R_{\mathrm{rec}}(g)+(1-\kappa)\,\widehat R_{\mathrm{keep}}(g)$
minimize the empirical mixed objective, and define the population mixed risk
$R_{\mathrm{mix}}(g)=\kappa\,R_{\mathrm{true}}(g)+(1-\kappa)\,R_{\mathrm{comp}}(g)$ with
$R_{\mathrm{comp}}(g)=\mathbb E_{\xx_t\sim\rho_t}[\mathcal{L}_{\mathrm{comp}}(\bar y_t,g(\xx_t))]$, together with
$\Delta_{\mathrm{mix}}=\kappa\,\Delta_{\mathrm{gen}}(m,\delta/2)+(1-\kappa)\,\Delta_{\mathrm{gen}}^{\mathrm{comp}}(n_k,\delta/2)$,
where $\Delta_{\mathrm{gen}}^{\mathrm{comp}}$ is the analogous uniform-convergence term for the complementary loss class. Then, with probability at least $1-\delta$, for every $g^{\star}\in\mathcal G$,
\begin{equation}
R_{\mathrm{mix}}(\widehat g)-R_{\mathrm{mix}}(g^{\star})
\;\le\;
2\kappa\,(1-\pi_{\mathrm{det}})\,B
\;+\;
2\,\Delta_{\mathrm{mix}}.
\end{equation}
\end{corollary}

\begin{proof}
\textbf{Step 1. Bias of the observed mixture.}
For every $g$, Eq.~\eqref{eq:obsclean} gives
$\bigl|\widehat R_{\mathrm{mix}}(g)-\kappa\widehat R_{\mathrm{clean}}(g)-(1-\kappa)\widehat R_{\mathrm{keep}}(g)\bigr|
=\kappa\bigl|\widehat R_{\mathrm{rec}}(g)-\widehat R_{\mathrm{clean}}(g)\bigr|
\le\kappa(1-\pi_{\mathrm{det}})B$.

\textbf{Step 2. Uniform convergence of the unbiased mixture.}
A union bound over the two branches gives, with probability at least $1-\delta$,
$\sup_{g\in\mathcal G}\bigl|\kappa\widehat R_{\mathrm{clean}}(g)+(1-\kappa)\widehat R_{\mathrm{keep}}(g)-R_{\mathrm{mix}}(g)\bigr|\le\Delta_{\mathrm{mix}}$.

\textbf{Step 3. Chaining.}
Combining Steps 1 and 2 exactly as in the proof of Theorem~\ref{thm:recycle_truelabel},
\begin{equation*}
R_{\mathrm{mix}}(\widehat g)
\le\widehat R_{\mathrm{mix}}(\widehat g)+\kappa(1-\pi_{\mathrm{det}})B+\Delta_{\mathrm{mix}}
\le\widehat R_{\mathrm{mix}}(g^{\star})+\kappa(1-\pi_{\mathrm{det}})B+\Delta_{\mathrm{mix}}
\le R_{\mathrm{mix}}(g^{\star})+2\kappa(1-\pi_{\mathrm{det}})B+2\Delta_{\mathrm{mix}},
\end{equation*}
where the middle inequality uses that $\widehat g$ minimizes the empirical mixed objective.
\end{proof}

Setting $\kappa=1$ recovers Theorem~\ref{thm:recycle_truelabel}. In the mixed objective the noise bias enters only through the recycled branch, scaled by the weight $\kappa$.

\begin{remark}[Data-dependent selection]
\label{rem:selection}
The recycled set $I^{\mathrm{rec}}$ is obtained by thresholding scores computed on the same pool, so its elements are not an i.i.d.\ sample drawn independently of the selection, and Theorem~\ref{thm:recycle_truelabel} and Corollary~\ref{cor:mixed} should be read conditionally on the realized selection, with $\pi_{\mathrm{det}}$ its realized precision; this is exactly the quantity we measure post hoc and report (Appendix~\ref{post-hoc_analysis}). An unconditional statement follows from a standard sample-splitting variant: compute the scores and the cuts on one half of the pool and apply recycling only to the other half. Conditionally on the first half, the selection region is fixed, the selected samples of the second half form an i.i.d.\ draw from it, and both results apply verbatim with $m$ the split size. We verified this device empirically with the TW-LAVA configuration (noisy-group score, fixed cut) on synthetic CIFAR-10 ($\eta=25\%$, $N_v=2{,}000$, seed $1126$): scoring each half of the pool with the entropic $c$-transform of the dual potentials fitted on the other half gives detector AUC $0.962$ and recycled-branch precision $83.5\%$, compared with $0.971$ and $86.8\%$ for full-pool scoring, and end-to-end training on the split-fitted selection reaches $85.16\%$ against $86.33\%$, so the unconditional guarantee costs about one accuracy point in practice. Because the per-task rules are one-dimensional thresholds, letting a fitted threshold be chosen on the evaluation half adds only the complexity of threshold classes, a union of families of VC dimension one that inflates $\Delta_{\mathrm{gen}}$ by a logarithmic factor; the adaptive cut of CAVA fits no threshold, so it adds nothing. Both devices change constants, not the structure of the bounds.
\end{remark}

The bound separates two effects: the estimation term $\Delta_{\mathrm{gen}}$ shrinks as the recycled set grows, whereas false positives leave a bias floor $2(1-\pi_{\mathrm{det}})B$ that more data cannot remove. Recycling therefore helps only when the flagged set is pure, which is why CAVA recycles the adaptive-cut set, whose precision Theorem~\ref{thm:detect} gives in closed form, and why the selection-only objective is the default when that precision cannot be checked.

\section{Experimental Setup}
\label{Appendix_C}
This section gives the datasets, the training and valuation settings, the metrics used to evaluate selections, and the baselines.

\subsection{Datasets}
\label{clarifying_cifar10}

\paragraph{Synthetic benchmarks.} MNIST~\citep{mnist}, CIFAR-10, and CIFAR-20~\citep{cifar10} (CIFAR-100 grouped into its $20$ superclasses) receive synthetic diagonal noise: each training label equals the true class with probability $\eta$ and is otherwise drawn uniformly from the other $M-1$ classes (Assumption~\ref{ass:noise_model}); $\eta=25\%$ unless stated otherwise.

\paragraph{Human-annotated benchmarks.} CLCIFAR-10 and CLCIFAR-20~\citep{wang2024climage} pair the CIFAR training images with crowdsourced complementary labels; we use the first of the three labels per image, of which $3.98\%$ ($1{,}990$ images) and $2.32\%$ ($1{,}161$) name the true class. CLMIN-10 and CLMIN-20 from the same collection contain $5{,}000$ and $10{,}000$ ImageNet training images of size $64\times64$, resized to $32\times32$, with $4.80\%$ and $3.28\%$ such annotation errors. All four have error rates below $1/M$, the regime in which the label cost of TW-LAVA reverses (Theorem~\ref{thm:flip}).

\paragraph{Further domains.} Three benchmarks outside CIFAR and ImageNet receive the same synthetic diagonal noise ($\eta\in\{5,15,25\}\%$). FOOD15 uses the first $15$ of $20$ dishes we take from Food-101~\citep{bossard2014food} (hamburger, pizza, sushi, ramen, steak, ice cream, french fries, caesar salad, fried rice, chocolate cake, donuts, dumplings, pancakes, spaghetti bolognese, tacos) with the official split ($750$ training and $250$ test images per class); the training images of Food-101 were left uncleaned by its authors and may carry a small amount of label noise of their own, while the test images were cleaned. RESISC45-20 uses $20$ of the $45$ aerial scene classes of NWPU-RESISC45~\citep{cheng2017remote} (airport, beach, bridge, church, desert, forest, harbor, island, parking lot, stadium, circular farmland, cloud, freeway, golf course, lake, mountain, railway, roundabout, ship, storage tank); the dataset has no official split, so the $700$ images per class are split into $500$ training and $200$ test images with seed $1126$. Places365-10 uses $10$ scene classes of Places365-Standard~\citep{zhou2018places} (airfield, beach, bedroom, bridge, forest path, kitchen, mountain, skyscraper, football stadium, outdoor swimming pool), with $4{,}000$ training and $1{,}000$ test images per class drawn from the official training set with seed $1126$, since the official validation split holds only $100$ images per class. All images are center-cropped and resized to $32\times32$; at $64\times64$ and $128\times128$ the accuracy did not improve in a pilot. Each dataset has its own feature extractor of the same architecture, trained on the clean labels of its training split and taken after the final epoch. The validation set holds $50$ test images per class ($200$ for Places365-10), and accuracy is measured on the remaining test images.

\paragraph{Scale.} CIFAR-10-1M enlarges the CIFAR-10 training set to one million samples by label-preserving augmentation (Appendix~\ref{large-scale_dataset}).

\paragraph{Scope.} Benchmarks with at most $20$ classes are standard in CLL, and current CLL methods are workable only at this scale~\citep{ishida2017learning,mcl2020,scl2020,wang2024climage,mai2026cllsurvey}, because a complementary label excludes a single class and the supervision weakens as the label space grows; extending CLL to larger label spaces is beyond the scope of this work.

\subsection{Training and Valuation Settings}
\label{sec:full_exp_details}

For all image classification experiments, we use ResNet-18 as the backbone network. Models are trained for 200 epochs with batch size 512 using stochastic gradient descent with momentum 0.9 and weight decay $5\times10^{-4}$. The initial learning rate is set to 0.1 and decayed to 0.01 and 0.001 at epochs 100 and 150, respectively. We set random seeds to control data loading and training randomness. Table~\ref{tab:main_results} and every table or figure whose caption says so report the mean $\pm$ standard deviation over three independent runs with seeds $1126$, $42$, and $2023$, the same seeds for all compared methods; the other results are single runs with seed $1126$. Unless a caption says otherwise, every method in a table uses the same feature extractor, the same $2{,}000$ validation images, the same tasks, and AutoAugment, and accuracy is measured after the final epoch on the held-out test images. The recycling weight is $\kappa=0.5$ in all experiments. For computationally expensive scalability experiments such as CIFAR-10-1M, we report single-run results and explicitly mark them as such.

\paragraph{OT valuation settings.}
Following LAVA and SAVA~\citep{just2023lava,kesslersava}, the feature extractor $\phi$ is a fixed pretrained network shared by every valuation method and baseline, so no method gains an advantage from it. We use the embedder released with LAVA for CIFAR-10 (which CLCIFAR-10 reuses) and networks trained the same way for the other datasets: a supervised PreAct-ResNet-18 trained with the clean ordinary labels of the training split, which is the same image set as the complementary-label pool (the LAVA network reaches $99.87\%$ on the CIFAR-10 training split and $93.07\%$ on its test split; LAVA describes its embedder as trained on the validation set, which this released network is not). CLMIN and the further domains have their own such networks, and CIFAR-10-1M an augmentation-matched one. We use its class logits (the penultimate $512$-dimensional features for CIFAR-20 and CLCIFAR-20). In the complementary-label setting such an embedder is an upper bound on feature quality, since it has seen the true label of every pool image, and it is the only component of the pipeline that has: the selection, the thresholds and the training never use a true label. The detection quality of the label margin is bounded by how well $\phi$ separates the classes (Theorem~\ref{thm:detect}(i)), so weaker, label-free features lower the absolute numbers of every method without changing which method is better; Appendix~\ref{app:detect_baselines} reports all detectors on the same features.
Representations are L2-normalized and rescaled so that the median feature cost between pool and validation samples equals one. CAVA uses the ground cost of Eq.~\eqref{eq:cava_cost} with $\lambda_\xx=\lambda_y=1$, the $0/1$ complementary label cost, entropic regularization $\varepsilon_{\mathrm{OT}}=0.01$, uniform marginals, $T=5$ random tasks of equal size, and a fixed class-stratified validation set of $N_v=2{,}000$ test images (drawn once with seed 1126 and shared by all seeds and methods), used only to score the training samples and never for training. On CLMIN-10 and CLMIN-20 the test split holds only $50$ images per class, so $20$ per class form the validation set ($200$ and $400$ images) and the other $30$ per class the evaluation set. TW-LAVA uses the same feature extractor, $\lambda_\xx=\lambda_y=1$ and the noisy-group label cost of Eq.~\eqref{eq:cava_label_cost}; it is computed at the CAVA settings ($\varepsilon_{\mathrm{OT}}=0.01$, $N_v=2{,}000$). The small-noise experiments (Figure~\ref{fig:smallnoise} and Table~\ref{tab:smallnoise}) are the one exception: there TW-LAVA keeps the OT settings of its original implementation ($\varepsilon_{\mathrm{OT}}=0.1$, label costs given by the OT distance between label groups), while still using the same $2{,}000$ validation images as every other method. The direct score comparisons of Figures~\ref{fig:transfer_vs_cava} and~\ref{fig:flip} recompute the noisy-group score at the CAVA settings (same normalized features, $\varepsilon_{\mathrm{OT}}=0.01$, $N_v=2{,}000$), so that only the label cost and the score differ.

\paragraph{Transition matrix for FWD and CPE.}
Forward correction and CPE take a class transition matrix $Q$. As in their original papers and in the CLImage benchmark~\citep{fwd2018,cpe2023,wang2024climage}, $Q$ is the empirical transition matrix of the training pool (the known-$Q$ protocol); it is the same for every method in a column, so it affects the objective and not the selection. The empirical $Q$ of a noisy pool contains the diagonal noise rates, so these two objectives receive information that SCL-NL and SCL-EXP do not; our main results use SCL-NL, which needs no transition matrix, and CAVA improves it over TW-LAVA by $3.9$ to $43.1$ points on the CLCIFAR datasets, against $7.3$ to $46.6$ points for FWD (Table~\ref{tab:main_results}, Figure~\ref{fig:realworld_compare}).

\subsection{Branch-Quality Metrics}
\label{app:branch_metrics}
Several tables evaluate the two branches of Eq.~\eqref{eq:cava_mixed_objective} with the true labels, which are available for the synthetic benchmarks and for the annotated CLCIFAR datasets. In the actual CLL training setup, the true label $y_i$ is not observed, and CAVA only uses $(\xx_i,\bar y_i)$ together with the clean validation set. 
However, in the synthetic benchmark, the original labels are known to the evaluator because diagonal complementary-label noise is artificially generated. 
We therefore use $y_i$ only after CAVA has completed selection and recycling, and only to evaluate the quality of the recycled and retained subsets. 
The true labels are never used during training, OT valuation, threshold estimation, pruning, or recycling, apart from the fixed feature extractor described in Appendix~\ref{sec:full_exp_details}.

The two selected subsets affect different branches of Eq.~(\ref{eq:cava_mixed_objective}). 
The recycled set $I^{\mathrm{rec}}$ is trained with the CE loss by treating the observed complementary label $\bar y_i$ as an ordinary label. 
If a recycled sample satisfies $\bar y_i=y_i$, then the CE branch receives correct ordinary-label supervision. 
If a recycled sample satisfies $\bar y_i\neq y_i$, then the CE branch receives an incorrect ordinary label. 
Similarly, the retained set $I^{\mathrm{keep}}$ is trained with SCL-NL as the complementary-label branch, so any diagonal-noise sample remaining in $I^{\mathrm{keep}}$ still imposes a contradictory exclusion constraint. 
Therefore, the quality of $I^{\mathrm{rec}}$ affects the CE branch, while the quality of $I^{\mathrm{keep}}$ affects the SCL-NL branch.

Let $I^{\mathrm{noise}}=\{i:\bar y_i=y_i\}$ be the true diagonal-noise set, which is used only for this post-hoc evaluation. $I^{\mathrm{rec}}$ is the recycled set sent to the CE branch (the flagged tail for the fixed cut, or the adaptive-cut set $\{i:s_i\le0\}$ of CAVA), and $I^{\mathrm{keep}}$ is the retained set sent to the complementary-label branch. Table~\ref{tab:branch_metrics} lists the metrics; the first three describe the CE branch as a detector of diagonal noise, and the last two give the share of wrong supervision in each branch.

\begin{table}[htb]
\centering
\caption{Branch-quality metrics. Higher is better for precision, recall, and F1; lower is better for the two contamination rates.}
\label{tab:branch_metrics}
\small
\renewcommand{\arraystretch}{1.35}
\begin{tabular}{@{}lll@{}}
\toprule
\textbf{Metric} & \textbf{Definition} & \textbf{What it measures} \\
\midrule
CE Precision & $|I^{\mathrm{rec}}\cap I^{\mathrm{noise}}|\,/\,|I^{\mathrm{rec}}|$ & recycled labels that are correct ordinary labels \\
CE Recall & $|I^{\mathrm{rec}}\cap I^{\mathrm{noise}}|\,/\,|I^{\mathrm{noise}}|$ & diagonal noise recovered into the CE branch \\
F1 & harmonic mean of the two & balance between precision and recall \\
CE Contam. & $1-\text{CE Precision}$ & wrong ordinary labels in the CE branch \\
CLL Contam. & $|I^{\mathrm{keep}}\cap I^{\mathrm{noise}}|\,/\,|I^{\mathrm{keep}}|$ & diagonal noise left in the complementary-label branch \\
\bottomrule
\end{tabular}
\end{table}

\subsection{Detection Baselines}
\label{app:detect_baselines}
Table~\ref{tab:detect_baselines} gives the values plotted in Figure~\ref{fig:cava_outperform}, which shows the synthetic datasets at $\eta=5\%$ and $\eta=15\%$, and adds $\eta=25\%$. At $\eta=5\%$, below the crossover $1/M$ of CIFAR-10 and at the crossover of CIFAR-20, TW-LAVA is reversed and at chance, respectively; above both crossovers, Data-OOB and small-loss selection are reversed instead. All methods are evaluated on the same pool and five tasks (seed $1126$), all valuation methods use the same feature extractor and the same $2{,}000$ validation images, and each score is evaluated in its natural direction, with valid samples expected to score high. KNN-Shapley~\citep{jia2019efficient} uses the exact recursion with the $k$-nearest-neighbor accuracy ($k=10$) on the validation set as utility and $\bar y_i$ as the training label. Data-OOB~\citep{kwon2023dataoob} bags $50$ logistic regressions fit on $(\xx_i,\bar y_i)$ and scores each sample by the out-of-bag probability of $\bar y_i$. Small-loss selection~\citep{han2018co} needs no clean data; it warms up a ResNet-18 from scratch with SCL-NL on the full pool for $20$ epochs and scores each sample by its negative loss averaged over the last $5$ epochs. TW-LAVA is the calibrated score of Eq.~\eqref{eq:cava_score} at the CAVA OT settings. Every score in this table is measured on the first task; over all five tasks, TW-LAVA gives $0.996$, $0.945$, $0.076$, and $0.149$ at $\eta=25\%$ on CIFAR-10 and CIFAR-20 and on CLCIFAR-10 and CLCIFAR-20, the values used in the rest of the appendix.

\begin{figure}[htb]
\centering
\includegraphics{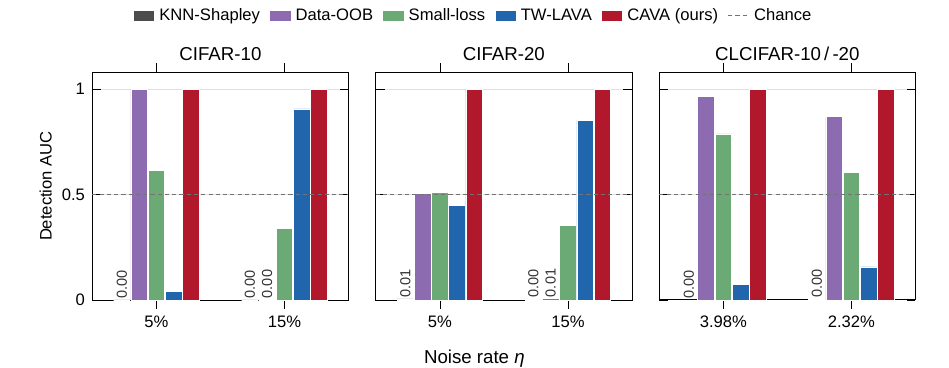}
\caption{\textbf{Detection of diagonal noise.} Detection AUC of each score in its natural direction; the synthetic datasets are shown at $\eta=5\%$ and $15\%$, around the crossover $1/M$ ($10\%$ for CIFAR-10, $5\%$ for CIFAR-20). The gray dashed line marks chance ($0.5$): a bar below it ranks the noise above valid data, and bars with an AUC near $0$ show their value.}
\label{fig:cava_outperform}
\end{figure}

\begin{table}[htb]
\centering
\caption{Detection AUC of diagonal noise in the natural direction of each score (valid samples high; $0.5$ is chance, values near $0$ rank the noise as the most valuable data). Same pool and tasks for every method (seed $1126$).}
\label{tab:detect_baselines}
\small
\begin{tabular}{@{}lcccccccc@{}}
\toprule
& \multicolumn{3}{c}{CIFAR-10} & \multicolumn{3}{c}{CIFAR-20} & CLCIFAR-10 & CLCIFAR-20 \\
\cmidrule(lr){2-4}\cmidrule(lr){5-7}
Noise rate $\eta$ & $25\%$ & $15\%$ & $5\%$ & $25\%$ & $15\%$ & $5\%$ & $3.98\%$ & $2.32\%$ \\
\midrule
KNN-Shapley & 0.000 & 0.000 & 0.000 & 0.002 & 0.003 & 0.005 & 0.000 & 0.004 \\
Data-OOB & 0.000 & 0.001 & 1.000 & 0.002 & 0.006 & 0.506 & 0.965 & 0.869 \\
Small-loss & 0.217 & 0.338 & 0.614 & 0.313 & 0.354 & 0.509 & 0.787 & 0.603 \\
TW-LAVA & 0.995 & 0.906 & 0.041 & 0.940 & 0.852 & 0.447 & 0.072 & 0.156 \\
\midrule
\rowcolor{green!12} CAVA (ours) & 1.000 & 1.000 & 1.000 & 0.997 & 0.998 & 0.999 & 1.000 & 0.999 \\
\bottomrule
\end{tabular}
\end{table}

\subsection{Target-Label-Free Representations}
\label{app:representations}
The supervised embedder of the main tables has seen the true label of every pool image (Appendix~\ref{sec:full_exp_details}), so it bounds the quality any representation can give the valuation. This subsection repeats the detection experiments with two extractors that have seen no label of the pool, under the protocol of Table~\ref{tab:main_results}: \emph{CLIP}, the image encoder of CLIP ViT-B/32~\citep{radford2021clip} applied to the images upsampled to $224\times224$, and \emph{validation only}, a PreAct-ResNet-18 trained from scratch on the $2{,}000$ validation images ($200$ or $400$ on CLMIN) and nothing else, the protocol LAVA describes for its own embedder, with penultimate features and no model selection. Both are L2-normalized and median-rescaled as before, and all methods in a row share the representation, the validation set and the five tasks. Seeds $42$ and $2023$ change the AUC of CAVA by at most $0.005$ and that of TW-LAVA by at most $0.011$ on every dataset of Table~\ref{tab:representations}, so the table reports seed $1126$.

\paragraph{What the transport contributes.} Theorem~\ref{thm:margin} holds for any validation potential, including $w=0$, where $p_i$ is a kernel classifier over the validation set with bandwidth $\varepsilon_{\mathrm{OT}}$. The column ``$w=0$'' is therefore the exact ablation of the OT solve. Table~\ref{tab:representations} shows that it is within $0.006$ AUC of CAVA on both representations and every dataset where we computed it: the potentials, which enforce that every validation image receives the same mass, do not improve detection. Two further settings test the one effect the marginal constraint can have. With $N_v\in\{50,100,200,500\}$ on CIFAR-10 (CLIP), CAVA reaches AUC $0.978$, $0.984$, $0.987$ and $0.990$, and $w=0$ is within $0.001$ of it at $N_v\ge100$ and higher at $N_v=50$ ($0.981$). When the class marginals of the pool and of the validation set differ, the constraint forces mass onto the wrong classes: thinning five validation classes to $20\%$ lowers CAVA to $0.920$ while $w=0$ stays at $0.985$, and thinning five pool classes lowers CAVA to $0.983$ against $0.994$ for $w=0$. The transport is thus the framework in which the complementary label cost and the margin are defined, and with mismatched class marginals the margin without the marginal constraint ($w=0$) is preferable.

\begin{table*}[htb]
\centering
\caption{Detection of diagonal noise on two representations that have never seen a label of the pool (seed $1126$, all five tasks). AUC treats valid samples as positive; ``adaptive cut'' gives the precision and recall of $\{s_i\le0\}$, and ``$w=0$'' is the CAVA margin without the validation potential. The supervised embedder of the main tables is shown for reference; ``--'': not run.}
\label{tab:representations}
\small
\setlength{\tabcolsep}{4.5pt}
\begin{tabular}{@{}llcccc@{}}
\toprule
& & \multicolumn{3}{c}{\textbf{CAVA (ours)}} & \\
\cmidrule(lr){3-5}
\textbf{Features} & \textbf{Dataset} & AUC & adaptive cut P / R & $w=0$ & \textbf{TW-LAVA} \\
\midrule
\multirow{7}{*}{CLIP}
& MNIST & 0.990 & 96.7 / 88.5 & -- & 0.771 \\
& CIFAR-10 & 0.994 & 96.6 / 88.1 & 0.992 & 0.773 \\
& CIFAR-20 & 0.971 & 94.5 / 66.3 & 0.969 & 0.711 \\
& CLCIFAR-10 & 0.996 & 87.1 / 86.4 & 0.995 & 0.758 \\
& CLCIFAR-20 & 0.981 & 75.4 / 59.6 & 0.976 & 0.621 \\
& CLMIN-10 & 0.968 & 73.0 / 70.8 & 0.965 & 0.540 \\
& CLMIN-20 & 0.949 & 66.7 / 55.5 & -- & 0.514 \\
\midrule
\multirow{7}{*}{validation only}
& MNIST & 0.999 & 99.4 / 98.1 & -- & 0.985 \\
& CIFAR-10 & 0.940 & 86.7 / 68.5 & 0.941 & 0.751 \\
& CIFAR-20 & 0.853 & 81.6 / 37.8 & 0.855 & 0.624 \\
& CLCIFAR-10 & 0.943 & 52.9 / 66.6 & 0.942 & 0.742 \\
& CLCIFAR-20 & 0.861 & 28.7 / 32.6 & 0.857 & 0.595 \\
& CLMIN-10 & 0.797 & 25.0 / 36.7 & 0.794 & 0.585 \\
& CLMIN-20 & 0.754 & 20.4 / 22.6 & -- & 0.499 \\
\midrule
\multirow{4}{*}{supervised (oracle)}
& CIFAR-10 & 1.000 & 99.98 / 99.8 & -- & 0.996 \\
& CIFAR-20 & 0.997 & 98.8 / 87.2 & -- & 0.945 \\
& CLCIFAR-10 & 1.000 & 99.7 / 99.7 & -- & 0.076 \\
& CLCIFAR-20 & 0.999 & 96.0 / 87.9 & -- & 0.149 \\
\bottomrule
\end{tabular}
\end{table*}

\paragraph{The flip law and the two components on label-free features.} Table~\ref{tab:cost_score_ablation} repeats the sweep of Figure~\ref{fig:flip} on CLIP features for the four combinations of label cost and score. The noisy-group cost flips with either score: with the calibrated potential (TW-LAVA) the AUC rises from $0.34$ at $\eta=2\%$ through chance at $\eta=1/M$ to $0.77$ at $25\%$, and with the margin from $0.04$ to $0.98$. The complementary label cost does not flip with either score; the calibrated potential under it stays at $0.97$ to $0.98$ and the margin at $0.994$ to $0.995$, the difference being the typicality term $T_i$ that the margin cancels (Theorem~\ref{thm:margin}(i)). The cost therefore removes the flip and the margin adds the last points; CIFAR-20 behaves the same way with a crossover at $5\%$.

\begin{table}[htb]
\centering
\caption{Detection AUC of the four combinations of label cost and score on CLIP features (seed $1126$, $N_v=2{,}000$) against the noise rate; the crossover is $\eta=1/M$. The supervised embedder gives the same pattern (Figure~\ref{fig:flip}).}
\label{tab:cost_score_ablation}
\small
\setlength{\tabcolsep}{4pt}
\begin{tabular}{@{}llcccc@{}}
\toprule
& & \multicolumn{2}{c}{\textbf{Noisy-group cost}} & \multicolumn{2}{c}{\textbf{Complementary cost}} \\
\cmidrule(lr){3-4} \cmidrule(l){5-6}
\textbf{Dataset} & $\eta$ & calibrated (TW-LAVA) & margin & calibrated & margin (CAVA) \\
\midrule
\multirow{5}{*}{CIFAR-10} & 2\% & 0.342 & 0.041 & 0.983 & 0.995 \\
& 5\% & 0.400 & 0.063 & 0.981 & 0.995 \\
& 10\% & 0.503 & 0.508 & 0.979 & 0.994 \\
& 15\% & 0.610 & 0.952 & 0.977 & 0.994 \\
& 25\% & 0.773 & 0.982 & 0.974 & 0.994 \\
\midrule
\multirow{4}{*}{CIFAR-20} & 2.5\% & 0.457 & 0.250 & 0.946 & 0.974 \\
& 5\% & 0.489 & 0.488 & 0.940 & 0.972 \\
& 15\% & 0.613 & 0.926 & 0.939 & 0.974 \\
& 25\% & 0.711 & 0.946 & 0.933 & 0.971 \\
\bottomrule
\end{tabular}
\end{table}

\paragraph{Estimating the precision without true labels.} Theorem~\ref{thm:detect}(ii) needs the error $e$ of the posterior classifier and the noise rate $\eta$. On CLIP features we estimate $e$ by leave-one-out prediction of $h$ on the validation set and $\eta$ from the flagged fraction (Appendix~\ref{app:refmargin_proofs}). Table~\ref{tab:precision_estimate} compares the estimate with the precision measured post hoc. On the synthetic benchmarks, which satisfy Assumption~\ref{ass:noise_model}, the estimate is within $1.5$ points. On the human-annotated benchmarks it underestimates the precision by $15$ to $57$ points, because their valid labels are not uniform and their errors concentrate on ambiguous images, so the uniform model overstates how often a valid sample is flagged; the error is on the safe side for gating recycling. The estimate is not available for the validation-only embedder, which fits the validation images and makes the leave-one-out error vanish.

\begin{table}[htb]
\centering
\caption{Estimating the precision of the adaptive cut without true labels (CLIP features, seed $1126$): leave-one-out error $\hat e$ of the posterior classifier, estimated noise rate $\hat\eta$, and estimated precision $\hat\pi$ (Eq.~\eqref{eq:detect_precision}), next to the measured precision and recall.}
\label{tab:precision_estimate}
\small
\setlength{\tabcolsep}{5pt}
\begin{tabular}{@{}lccccccc@{}}
\toprule
\textbf{Dataset} & $\hat e$ & $\hat\eta$ & $\eta$ & $\hat\pi$ (\%) & measured (\%) & recall (\%) \\
\midrule
MNIST & 0.093 & 0.244 & 0.250 & 96.6 & 96.7 & 88.5 \\
CIFAR-10 & 0.107 & 0.245 & 0.250 & 96.1 & 96.6 & 88.1 \\
CIFAR-20 & 0.304 & 0.234 & 0.250 & 93.0 & 94.5 & 66.3 \\
CLCIFAR-10 & 0.104 & 0.032 & 0.040 & 71.7 & 87.1 & 86.4 \\
CLCIFAR-20 & 0.285 & 0.005 & 0.023 & 18.6 & 75.4 & 59.6 \\
CLMIN-10 & 0.248 & 0.026 & 0.048 & 42.4 & 73.0 & 70.8 \\
CLMIN-20 & 0.349 & 0.014 & 0.033 & 33.8 & 66.7 & 55.5 \\
\bottomrule
\end{tabular}
\end{table}

\paragraph{Training on label-free features.} Table~\ref{tab:representations_train} reports Table~\ref{tab:main_results} on CLIP features and on the validation-only features, with three seeds each. On both representations and with both objectives, CAVA is above TW-LAVA on all seven datasets; with recycling on the validation-only features it reaches $75.22\%$ on CIFAR-10 and $49.21\%$ on CIFAR-20, against $50.61\%$ and $22.27\%$ for TW-LAVA.
\begin{table*}[htb]
\centering
\caption{Test accuracy (\%) of Table~\ref{tab:main_results} on CLIP features and on the validation-only features (mean $\pm$ std over seeds 1126, 42, 2023). Best per column, block and representation in \textbf{bold}.}
\label{tab:representations_train}
\small
\setlength{\tabcolsep}{2.9pt}
\begin{tabular}{@{}lll ccc cccc@{}}
\toprule
& & & MNIST & CIFAR-10 & CIFAR-20 & CLCIFAR-10 & CLCIFAR-20 & CLMIN-10 & CLMIN-20 \\
\midrule
\multicolumn{10}{@{}l}{\textit{(a) Selection only: SCL-NL on the retained pool}} \\
\midrule
CLIP & TW-LAVA & Fixed $p$ & 26.25{\scriptsize$\pm$1.31} & 8.41{\scriptsize$\pm$1.86} & 3.10{\scriptsize$\pm$0.31} & 42.79{\scriptsize$\pm$1.79} & 7.69{\scriptsize$\pm$0.08} & 15.22{\scriptsize$\pm$2.17} & 4.50{\scriptsize$\pm$0.17} \\
CLIP & CAVA (ours) & Fixed $p$ & \textbf{90.59}{\scriptsize$\pm$0.21} & \textbf{74.11}{\scriptsize$\pm$0.32} & \textbf{12.05}{\scriptsize$\pm$1.05} & \textbf{79.06}{\scriptsize$\pm$0.67} & \textbf{10.64}{\scriptsize$\pm$0.89} & \textbf{27.11}{\scriptsize$\pm$1.90} & \textbf{10.66}{\scriptsize$\pm$0.76} \\
CLIP & CAVA (ours) & Adaptive cut & 84.58{\scriptsize$\pm$0.35} & 64.19{\scriptsize$\pm$3.36} & 7.01{\scriptsize$\pm$1.74} & 77.94{\scriptsize$\pm$0.43} & 8.48{\scriptsize$\pm$0.11} & 25.11{\scriptsize$\pm$4.43} & 7.83{\scriptsize$\pm$2.92} \\
validation only & TW-LAVA & Fixed $p$ & 97.91{\scriptsize$\pm$0.10} & 29.90{\scriptsize$\pm$2.29} & 2.71{\scriptsize$\pm$0.12} & 43.43{\scriptsize$\pm$0.72} & 7.68{\scriptsize$\pm$0.15} & 16.22{\scriptsize$\pm$3.02} & 4.39{\scriptsize$\pm$0.59} \\
validation only & CAVA (ours) & Fixed $p$ & \textbf{98.83}{\scriptsize$\pm$0.06} & \textbf{66.25}{\scriptsize$\pm$0.31} & \textbf{18.57}{\scriptsize$\pm$1.40} & 73.33{\scriptsize$\pm$0.08} & \textbf{10.41}{\scriptsize$\pm$1.77} & 21.56{\scriptsize$\pm$2.52} & \textbf{12.83}{\scriptsize$\pm$0.67} \\
validation only & CAVA (ours) & Adaptive cut & 98.74{\scriptsize$\pm$0.09} & 56.61{\scriptsize$\pm$0.76} & 4.18{\scriptsize$\pm$2.27} & \textbf{73.90}{\scriptsize$\pm$0.34} & 8.57{\scriptsize$\pm$0.34} & \textbf{22.78}{\scriptsize$\pm$3.56} & 9.78{\scriptsize$\pm$0.51} \\
\midrule
\multicolumn{10}{@{}l}{\textit{(b) Selection with recycling: CE+SCL-NL}} \\
\midrule
CLIP & TW-LAVA & Fixed $p$ & 56.42{\scriptsize$\pm$2.59} & 46.68{\scriptsize$\pm$0.58} & 28.85{\scriptsize$\pm$0.96} & 28.18{\scriptsize$\pm$2.10} & 5.24{\scriptsize$\pm$0.29} & 16.11{\scriptsize$\pm$1.65} & 6.45{\scriptsize$\pm$1.11} \\
CLIP & CAVA (ours) & Fixed $p$ & \textbf{94.93}{\scriptsize$\pm$0.13} & \textbf{87.91}{\scriptsize$\pm$0.55} & \textbf{67.37}{\scriptsize$\pm$0.24} & 80.59{\scriptsize$\pm$0.23} & \textbf{42.12}{\scriptsize$\pm$0.77} & \textbf{50.44}{\scriptsize$\pm$1.26} & \textbf{29.55}{\scriptsize$\pm$0.86} \\
CLIP & CAVA (ours) & Adaptive cut & 94.31{\scriptsize$\pm$0.03} & 87.24{\scriptsize$\pm$0.28} & 62.67{\scriptsize$\pm$0.34} & \textbf{81.84}{\scriptsize$\pm$0.29} & 41.66{\scriptsize$\pm$0.90} & 49.89{\scriptsize$\pm$0.70} & 29.39{\scriptsize$\pm$0.26} \\
validation only & TW-LAVA & Fixed $p$ & 97.78{\scriptsize$\pm$0.47} & 50.61{\scriptsize$\pm$1.32} & 22.27{\scriptsize$\pm$0.56} & 34.47{\scriptsize$\pm$1.22} & 5.85{\scriptsize$\pm$0.60} & 19.11{\scriptsize$\pm$1.17} & 4.78{\scriptsize$\pm$1.77} \\
validation only & CAVA (ours) & Fixed $p$ & 99.29{\scriptsize$\pm$0.05} & \textbf{75.22}{\scriptsize$\pm$0.67} & \textbf{49.21}{\scriptsize$\pm$0.85} & \textbf{72.53}{\scriptsize$\pm$0.14} & \textbf{35.07}{\scriptsize$\pm$0.47} & 40.11{\scriptsize$\pm$1.07} & 22.22{\scriptsize$\pm$1.27} \\
validation only & CAVA (ours) & Adaptive cut & \textbf{99.33}{\scriptsize$\pm$0.05} & 73.50{\scriptsize$\pm$0.31} & 41.81{\scriptsize$\pm$0.48} & 72.46{\scriptsize$\pm$0.13} & 34.06{\scriptsize$\pm$0.32} & \textbf{43.89}{\scriptsize$\pm$0.70} & \textbf{22.55}{\scriptsize$\pm$0.95} \\
\bottomrule
\end{tabular}
\end{table*}

\paragraph{Training without the transport solve.} Table~\ref{tab:repr_train_detectors} trains on the selections of Table~\ref{tab:representations} with CLIP features. The accuracies follow the detection results: with recycling, CAVA without the OT solve ($w=0$) is within about one point of CAVA on every dataset except CLCIFAR-20 ($3.0$ points), and both are far above TW-LAVA. Table~\ref{tab:repr_train_components} separates the two components of CAVA in training. With the noisy-group label cost, the margin helps on CIFAR-10, above the crossover, but not on CLCIFAR-10, below it ($41.21\%$ and $31.08\%$ against $79.06\%$ and $80.59\%$ for CAVA). The complementary label cost alone removes the reversal on CLCIFAR-10 ($71.42\%$ and $73.26\%$), and the margin adds $7.1$ to $15.3$ points on top of it. With fewer validation images (Table~\ref{tab:repr_train_nv}), CAVA reaches $82.92\%$ and $84.36\%$ at $N_v=50$ and $100$, against $87.91\%$ at $N_v=2{,}000$.

\begin{table*}[htb]
\centering
\caption{Training on the selections of Table~\ref{tab:representations} with CLIP features (cuts of Table~\ref{tab:main_results}; mean $\pm$ std over three seeds). ``$w=0$'' is the CAVA margin without the validation potential. Best per column and block in \textbf{bold}.}
\label{tab:repr_train_detectors}
\small
\setlength{\tabcolsep}{3.5pt}
\begin{tabular}{@{}ll ccccc@{}}
\toprule
\textbf{Selection} & & CIFAR-10 & CIFAR-20 & CLCIFAR-10 & CLCIFAR-20 & CLMIN-10 \\
\midrule
\multicolumn{7}{@{}l}{\textit{(a) Selection only: SCL-NL on the retained pool}} \\
\midrule
TW-LAVA & Fixed $p$ & 8.41{\scriptsize$\pm$1.86} & 3.10{\scriptsize$\pm$0.31} & 42.79{\scriptsize$\pm$1.79} & 7.69{\scriptsize$\pm$0.08} & 15.22{\scriptsize$\pm$2.17} \\
\cellcolor{green!12}CAVA (ours) & \cellcolor{green!12}Fixed $p$ & \cellcolor{green!12}\textbf{74.11}{\scriptsize$\pm$0.32} & \cellcolor{green!12}12.05{\scriptsize$\pm$1.05} & \cellcolor{green!12}\textbf{79.06}{\scriptsize$\pm$0.67} & \cellcolor{green!12}\textbf{10.64}{\scriptsize$\pm$0.89} & \cellcolor{green!12}27.11{\scriptsize$\pm$1.90} \\
\cellcolor{green!12}CAVA (ours) & \cellcolor{green!12}Adaptive cut & \cellcolor{green!12}64.19{\scriptsize$\pm$3.36} & \cellcolor{green!12}7.01{\scriptsize$\pm$1.74} & \cellcolor{green!12}77.94{\scriptsize$\pm$0.43} & \cellcolor{green!12}8.48{\scriptsize$\pm$0.11} & \cellcolor{green!12}25.11{\scriptsize$\pm$4.43} \\
CAVA, $w=0$ & Fixed $p$ & 73.20{\scriptsize$\pm$0.06} & \textbf{13.98}{\scriptsize$\pm$0.81} & 78.27{\scriptsize$\pm$0.37} & 8.75{\scriptsize$\pm$0.16} & 25.44{\scriptsize$\pm$2.14} \\
CAVA, $w=0$ & Adaptive cut & 64.29{\scriptsize$\pm$0.65} & 8.77{\scriptsize$\pm$0.51} & 76.96{\scriptsize$\pm$0.88} & 9.39{\scriptsize$\pm$1.01} & \textbf{29.00}{\scriptsize$\pm$1.76} \\
\midrule
\multicolumn{7}{@{}l}{\textit{(b) Selection with recycling: CE+SCL-NL}} \\
\midrule
TW-LAVA & Fixed $p$ & 46.68{\scriptsize$\pm$0.58} & 28.85{\scriptsize$\pm$0.96} & 28.18{\scriptsize$\pm$2.10} & 5.24{\scriptsize$\pm$0.29} & 16.11{\scriptsize$\pm$1.65} \\
\cellcolor{green!12}CAVA (ours) & \cellcolor{green!12}Fixed $p$ & \cellcolor{green!12}\textbf{87.91}{\scriptsize$\pm$0.55} & \cellcolor{green!12}\textbf{67.37}{\scriptsize$\pm$0.24} & \cellcolor{green!12}80.59{\scriptsize$\pm$0.23} & \cellcolor{green!12}\textbf{42.12}{\scriptsize$\pm$0.77} & \cellcolor{green!12}50.44{\scriptsize$\pm$1.26} \\
\cellcolor{green!12}CAVA (ours) & \cellcolor{green!12}Adaptive cut & \cellcolor{green!12}87.24{\scriptsize$\pm$0.28} & \cellcolor{green!12}62.67{\scriptsize$\pm$0.34} & \cellcolor{green!12}\textbf{81.84}{\scriptsize$\pm$0.29} & \cellcolor{green!12}41.66{\scriptsize$\pm$0.90} & \cellcolor{green!12}49.89{\scriptsize$\pm$0.70} \\
CAVA, $w=0$ & Fixed $p$ & 87.38{\scriptsize$\pm$0.44} & 66.71{\scriptsize$\pm$0.28} & 80.06{\scriptsize$\pm$0.20} & 39.08{\scriptsize$\pm$0.86} & \textbf{51.56}{\scriptsize$\pm$1.17} \\
CAVA, $w=0$ & Adaptive cut & 87.01{\scriptsize$\pm$0.38} & 63.15{\scriptsize$\pm$0.61} & 80.99{\scriptsize$\pm$0.13} & 40.58{\scriptsize$\pm$0.63} & 50.00{\scriptsize$\pm$0.88} \\
\bottomrule
\end{tabular}
\end{table*}

\begin{table}[htb]
\centering
\caption{The two components of CAVA in training (CLIP features, fixed cut; mean $\pm$ std over three seeds). Table~\ref{tab:cost_score_ablation} gives the detection AUC of the same four combinations on CIFAR-10. Best per column in \textbf{bold}.}
\label{tab:repr_train_components}
\small
\setlength{\tabcolsep}{3.5pt}
\begin{tabular}{@{}ll cc cc@{}}
\toprule
& & \multicolumn{2}{c}{CIFAR-10 ($\eta=25\%$)} & \multicolumn{2}{c}{CLCIFAR-10 ($\eta=3.98\%$)} \\
\cmidrule(lr){3-4} \cmidrule(l){5-6}
\textbf{Label cost} & \textbf{Score} & SCL-NL & CE+SCL-NL & SCL-NL & CE+SCL-NL \\
\midrule
noisy-group & calibrated (TW-LAVA) & 8.41{\scriptsize$\pm$1.86} & 46.68{\scriptsize$\pm$0.58} & 42.79{\scriptsize$\pm$1.79} & 28.18{\scriptsize$\pm$2.10} \\
noisy-group & margin & 67.57{\scriptsize$\pm$1.31} & 85.46{\scriptsize$\pm$0.08} & 41.21{\scriptsize$\pm$4.00} & 31.08{\scriptsize$\pm$1.89} \\
complementary & calibrated & 58.79{\scriptsize$\pm$2.12} & 80.81{\scriptsize$\pm$0.32} & 71.42{\scriptsize$\pm$0.16} & 73.26{\scriptsize$\pm$0.49} \\
\cellcolor{green!12}complementary & \cellcolor{green!12}margin (CAVA) & \cellcolor{green!12}\textbf{74.11}{\scriptsize$\pm$0.32} & \cellcolor{green!12}\textbf{87.91}{\scriptsize$\pm$0.55} & \cellcolor{green!12}\textbf{79.06}{\scriptsize$\pm$0.67} & \cellcolor{green!12}\textbf{80.59}{\scriptsize$\pm$0.23} \\
\bottomrule
\end{tabular}
\end{table}

\begin{table}[htb]
\centering
\caption{Training with small validation sets (CLIP features, CIFAR-10, CE+SCL-NL, fixed cut $p=25\%$, mean $\pm$ std over three seeds, final-epoch accuracy on the test images outside the validation set). Best per column in \textbf{bold}.}
\label{tab:repr_train_nv}
\small
\setlength{\tabcolsep}{3.5pt}
\begin{tabular}{@{}l ccccc@{}}
\toprule
$N_v$ & 50 & 100 & 200 & 500 & 2{,}000 \\
\midrule
\cellcolor{green!12}CAVA (ours) & \cellcolor{green!12}\textbf{82.92}{\scriptsize$\pm$0.17} & \cellcolor{green!12}\textbf{84.36}{\scriptsize$\pm$0.06} & \cellcolor{green!12}\textbf{85.47}{\scriptsize$\pm$0.10} & \cellcolor{green!12}\textbf{86.61}{\scriptsize$\pm$0.15} & \cellcolor{green!12}\textbf{87.91}{\scriptsize$\pm$0.55} \\
CAVA, $w=0$ & 82.76{\scriptsize$\pm$0.21} & 83.74{\scriptsize$\pm$0.11} & 85.37{\scriptsize$\pm$0.41} & 86.08{\scriptsize$\pm$0.34} & 87.38{\scriptsize$\pm$0.44} \\
\bottomrule
\end{tabular}
\end{table}

\section{Additional Results}
\label{app:additional_results}

\subsection{Synthetic Benchmarks: Branch Composition and Complementary-Label Bias}
\label{sec:ablation-mnist}
\label{sec:ablation-cifar20}
This subsection analyzes the fixed cuts on MNIST, CIFAR-10, and CIFAR-20 with $\eta=25\%$ (AutoAugment). For every cut, Tables~\ref{tab:mnist-ablation}, \ref{tab:detector-compare}, and~\ref{tab:cifar20-ablation} report the composition of the two branches of Eq.~\eqref{eq:cava_mixed_objective}, measured by CE Contam.\ and CLL Contam.\ (Table~\ref{tab:branch_metrics}), and the test accuracy with and without recycling. CAVA rows follow the protocol of the main paper; TW-LAVA rows use the original TW-LAVA protocol with $10{,}000$ validation images.

\paragraph{Pruning is a prerequisite for complementary losses.}
Without selection, SCL-NL collapses to $1.66\%$ on MNIST, $5.10\%$ on CIFAR-10, and $2.25\%$ on CIFAR-20 (Table~\ref{tab:main_results}): the $25\%$ diagonal-noise samples teach ``not the true class'', which a loss that treats every label as an exclusion cannot absorb. Every pruned configuration improves on this baseline by an order of magnitude.

\paragraph{CAVA cuts at the noise rate.}
Because the label margin ranks the noise at the low end of every task (Figure~\ref{fig:ot-scores}), the cut $p=\eta$ gives an almost oracle split on all three datasets: the recycled branch contains $0.70\%$, $0.54\%$, and $4.44\%$ valid labels on MNIST, CIFAR-10, and CIFAR-20, the retained pool at most $1.48\%$ diagonal noise, and $p=\eta$ is the best cut for the recycling objective CE+SCL-NL (Table~\ref{tab:detector-compare}). The adaptive cut flags nearly the same set without knowing $\eta$ (last CAVA row of Tables~\ref{tab:mnist-ablation} and~\ref{tab:cifar20-ablation}). TW-LAVA degrades with its detector: as its AUC falls from $0.9994$ to $0.9833$ and $0.9087$, the valid labels in its recycled branch at $p=\eta$ rise from $1.24\%$ to $9.84\%$ and $24.34\%$, which is why its best cut moves away from $\eta$.

\paragraph{Cutting beyond the noise rate biases the complementary labels.}
A cut larger than $\eta$ can only move valid labels into the CE branch, close to the floor $(p-\eta)/p$ of $28.6\%$ and $44.4\%$ at $p=35\%$ and $45\%$. It also changes the labels that remain: the lowest-ranked valid samples are those whose given label is a plausible class for the image, so over-pruning removes confusable pairs such as $4{\leftrightarrow}9$ and $1{\leftrightarrow}7$ and leaves a transition matrix with near-empty cells ($14/90$ cells at $p=45\%$ on MNIST versus $0/90$ at $p=25\%$, and $38/380$ versus $5/380$ on CIFAR-20; Figures~\ref{fig:cll-bias} and~\ref{fig:cll-bias-cifar20}). SCL-NL assumes uniform transitions and degrades accordingly (CAVA: $99.35\%\to79.50\%$ on MNIST and $81.67\%\to55.13\%$ on CIFAR-10 from $p=25\%$ to $45\%$), whereas FWD, which estimates the transition matrix explicitly, is affected less. With TW-LAVA on CIFAR-20 the ordering is reversed, because the contamination left by its weaker detector outweighs this bias, so larger cuts help its selection-only training.

\begin{table}[htb]
\centering
\caption{Detector quality and branch composition across datasets under the identical synthetic protocol ($\eta=25\%$, AutoAug, combined CE+SCL-NL loss), for the baseline TW-LAVA and for CAVA with the same feature extractor. The detection AUC is $0.9994$, $0.9833$, $0.9087$ for TW-LAVA and $1.0000$, $1.0000$, $0.9971$ for CAVA on MNIST, CIFAR-10, CIFAR-20. CE Contam. denotes the fraction of the recycled CE branch that consists of valid complementary samples $(\bar y_i\neq y_i)$, which become false-positive ordinary labels when recycled. CLL Contam. denotes the fraction of the retained complementary-label pool that still consists of diagonal-noise samples $(\bar y_i=y_i)$. Accuracy denotes the CE+SCL-NL test accuracy: a single run (seed 1126) for TW-LAVA, which uses all 10{,}000 test images as the validation set, and the mean over seeds 1126, 42, and 2023 for CAVA, which uses a fixed validation set of 2{,}000 test images the selection statistics are from seed 1126. For $p>\eta$ the CE branch necessarily contains valid samples, so its contamination is at least $1-\eta/p$ for any detector.
}
\label{tab:detector-compare}
\footnotesize
\setlength{\tabcolsep}{1.8pt}
\begin{tabular}{ll|rrc|rrc|rrc}
\toprule
& & \multicolumn{3}{c|}{MNIST}
& \multicolumn{3}{c|}{CIFAR-10}
& \multicolumn{3}{c}{CIFAR-20} \\
\textbf{Method} & \textbf{Cut} & \textbf{CE Con.} & \textbf{CLL Con.} & \textbf{Acc.} & \textbf{CE Con.} & \textbf{CLL Con.} & \textbf{Acc.} & \textbf{CE Con.} & \textbf{CLL Con.} & \textbf{Acc.} \\
\midrule
TW-LAVA & $p{=}0.25$ & 1.24\% & 0.41\% & 99.53 & 9.84\% & 3.28\% & 86.34 & 24.34\% & 8.11\% & 64.10 \\
 & $p{=}0.35$ & 28.73\% & 0.09\% & 99.15 & 30.03\% & 0.78\% & 85.31 & 36.81\% & 4.44\% & 63.78 \\
 & $p{=}0.45$ & 44.54\% & 0.08\% & 99.20 & 44.99\% & 0.44\% & 84.13 & 47.70\% & 2.66\% & 61.87 \\
\midrule
CAVA (ours) & $p{=}0.25$ & 0.70\% & 0.23\% & \textbf{99.57} & 0.54\% & 0.18\% & \textbf{92.45} & 4.44\% & 1.48\% & \textbf{76.45} \\
 & $p{=}0.35$ & 28.57\% & 0.00\% & 98.47 & 28.57\% & 0.00\% & 85.18 & 28.69\% & 0.06\% & 70.44 \\
 & $p{=}0.45$ & 44.44\% & 0.00\% & 98.60 & 44.44\% & 0.00\% & 82.32 & 44.47\% & 0.02\% & 64.15 \\
\bottomrule
\end{tabular}
\end{table}

\begin{table}[htb]
\centering
\caption{
MNIST ablation under the synthetic diagonal-noise protocol with noise ratio $\eta=25\%$ and AutoAug. 
CLL Contam. denotes the fraction of the retained complementary-label pool that still consists of diagonal-noise samples $(\bar y_i=y_i)$. CE Contam. denotes the fraction of the recycled CE branch that consists of valid complementary samples $(\bar y_i\neq y_i)$, for which treating $\bar y_i$ as an ordinary label is incorrect. FWD uses the transition matrix computed from the synthetic corruption record, which includes the diagonal. The upper block is the baseline TW-LAVA and the lower block is CAVA with the same feature extractor. Static $p$ denotes the fixed cut. For CAVA, the adaptive cut $s_i\le0$ needs no ratio ($^\ast$: seed 1126). TW-LAVA entries are single runs (seed 1126) that use all 10{,}000 test images as the validation set; CAVA entries are means over seeds 1126, 42, and 2023 with a fixed validation set of 2{,}000 test images (largest standard deviation 7.40, FWD at $p=0.35$; at most 1.20 elsewhere except FWD at $p=0.45$, 4.55), and its selection statistics are from seed 1126;  the best accuracy in each column is in \textbf{bold}.
}
\label{tab:mnist-ablation}
\small
\begin{tabular}{l|ccc|c|cc}
\toprule
 & \multicolumn{3}{c|}{Branch Composition}  & {Combined}  & \multicolumn{2}{c}{CLL-only} \\
\textbf{Setting} & \textbf{CLL Size} & \textbf{CLL Contam.} & \textbf{CE Contam.} & \textbf{CE+SCL-NL} & \textbf{SCL-NL} & \textbf{FWD} \\
\midrule
\multicolumn{7}{l}{\emph{TW-LAVA (noisy-group label cost)}} \\
Static $p{=}0.25$     & 45{,}000     & 0.41\%     & 1.24\%      & 99.53     & 99.32     & 99.22 \\
Static $p{=}0.35$     & 39{,}000     & 0.09\%     & 28.73\%     & 99.15     & 96.39     & 98.14 \\
Static $p{=}0.45$     & 33{,}000     & 0.08\%     & 44.54\%     & 99.20     & 55.90     & 88.99 \\
\midrule
\multicolumn{7}{l}{\emph{CAVA (ours; complementary label cost, label margin)}} \\
Static $p{=}0.25$  & 45{,}000  & 0.23\%  & 0.70\%  & 99.57  & 99.35  & 99.41 \\
Static $p{=}0.35$  & 39{,}000  & 0.00\%  & 28.57\%  & 98.47  & 97.54  & 89.36 \\
Static $p{=}0.45$  & 33{,}000  & 0.00\%  & 44.44\%  & 98.60  & 79.50  & 93.33 \\
Adaptive cut $s_i\le0$  & 45{,}015  & 0.05\%  & 0.05\%  & \textbf{99.90}$^\ast$  & \textbf{99.71}$^\ast$  & \textbf{99.61}$^\ast$ \\
\bottomrule
\end{tabular}
\end{table}

\begin{table}[hth]
\centering
\caption{
CIFAR-20 ablation under noise ratio $\eta=25\%$ and AutoAug. Column semantics follow Table~\ref{tab:mnist-ablation}. CLL Contam. denotes the fraction of the retained complementary-label pool that still consists of diagonal-noise samples $(\bar y_i=y_i)$, while CE Contam. denotes the fraction of the recycled CE branch that consists of valid complementary samples $(\bar y_i\neq y_i)$. The corruption record is used only for post-hoc evaluation of branch composition and is never used by either method during OT valuation, thresholding, pruning, recycling, or training. The upper block is the baseline TW-LAVA and the lower block is CAVA with the same feature extractor. Static $p$ denotes the fixed cut. For CAVA, the adaptive cut $s_i\le0$ needs no ratio ($^\ast$: seed 1126). TW-LAVA entries are single runs (seed 1126) that use all 10{,}000 test images as the validation set; CAVA entries are means over seeds 1126, 42, and 2023 with a fixed validation set of 2{,}000 test images (largest standard deviation 2.36), and its selection statistics are from seed 1126;  the best accuracy in each column is in \textbf{bold}.
}
\label{tab:cifar20-ablation}
\small
\begin{tabular}{l|rrr|c|cc}
\toprule
 & \multicolumn{3}{c|}{Branch Composition}  & Combined  & \multicolumn{2}{c}{CLL-only} \\
\textbf{Setting} & \textbf{CLL Size} & \textbf{CLL Contam.} & \textbf{CE Contam.} & \textbf{CE+SCL-NL} & \textbf{SCL-NL} & \textbf{FWD} \\
\midrule
\multicolumn{7}{l}{\emph{TW-LAVA (noisy-group label cost)}} \\
Static $p{=}0.25$     & 37{,}500     & 8.11\%     & 24.34\%     & 64.10     & 12.74     & 15.53 \\
Static $p{=}0.35$     & 32{,}500     & 4.44\%     & 36.81\%     & 63.78     & 16.31     & 19.34 \\
Static $p{=}0.45$     & 27{,}500     & 2.66\%     & 47.70\%     & 61.87     & 17.77     & 27.16 \\
\midrule
\multicolumn{7}{l}{\emph{CAVA (ours; complementary label cost, label margin)}} \\
Static $p{=}0.25$  & 37{,}500  & 1.48\%  & 4.44\%  & \textbf{76.45}  & 17.23  & 18.34 \\
Static $p{=}0.35$  & 32{,}500  & 0.06\%  & 28.69\%  & 70.44  & \textbf{28.11}  & 35.57 \\
Static $p{=}0.45$  & 27{,}500  & 0.02\%  & 44.47\%  & 64.15  & 23.24  & \textbf{40.17} \\
Adaptive cut $s_i\le0$  & 38{,}975  & 4.12\%  & 1.19\%  & 73.63$^\ast$  & 14.70$^\ast$  & 15.82$^\ast$ \\
\bottomrule
\end{tabular}
\end{table}

\begin{figure}[htb]
\centering
\includegraphics{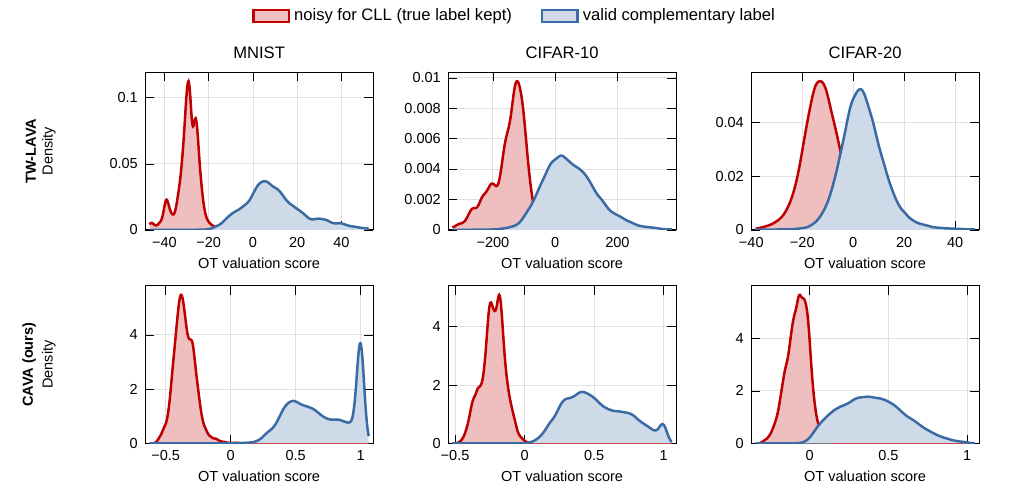}
\caption{OT valuation score distributions for diagonal-noise samples ($\bar y_i=y_i$) and valid complementary samples ($\bar y_i\neq y_i$) on MNIST, CIFAR-10, and CIFAR-20 under noise ratio $\eta=25\%$ and seed 1126, with the same feature extractor in both rows (validation set: 10{,}000 test images for TW-LAVA, a fixed subset of 2{,}000 for CAVA). \emph{Top:} TW-LAVA, the baseline with the noisy-group label cost. The two groups become less separable as the dataset becomes harder, with AUC decreasing from $0.9994$ on MNIST to $0.9833$ on CIFAR-10 and $0.9087$ on CIFAR-20, which explains the increasing branch contamination reported in Table~\ref{tab:detector-compare}. \emph{Bottom:} CAVA with the complementary label cost and the label margin. Diagonal noise stays on the low side and the separation is close to complete on all three datasets (AUC $1.000$, $1.000$, and $0.998$).
}
\label{fig:ot-scores}
\end{figure}

\begin{figure}[htb]
\centering
\includegraphics{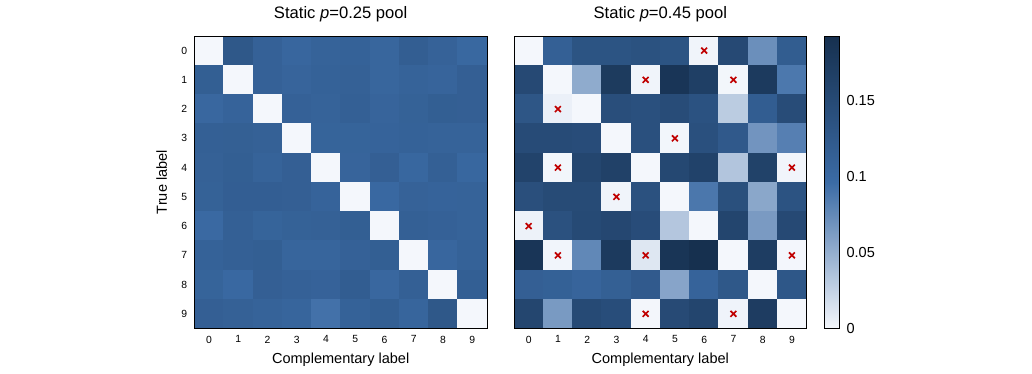}
\caption{Pruning-induced complementary-label bias (MNIST, $\eta = 25\%$; pools of the original TW-LAVA protocol). Row-normalized transition matrices $P(\bar{Y}\mid Y)$ of the CLL pool at $p{=}0.25$
(left, near-uniform) and $p{=}0.45$ (right); {\color{red}{$\times$}} marks near-empty transitions ($<2\%$ row mass). Aggressive pruning empties the transitions of visually confusable pairs, violating the uniform assumption of SCL-NL while the estimated $Q$ of FWD absorbs the bias.}
\label{fig:cll-bias}
\end{figure}

\begin{figure}[htb]
\centering
\includegraphics{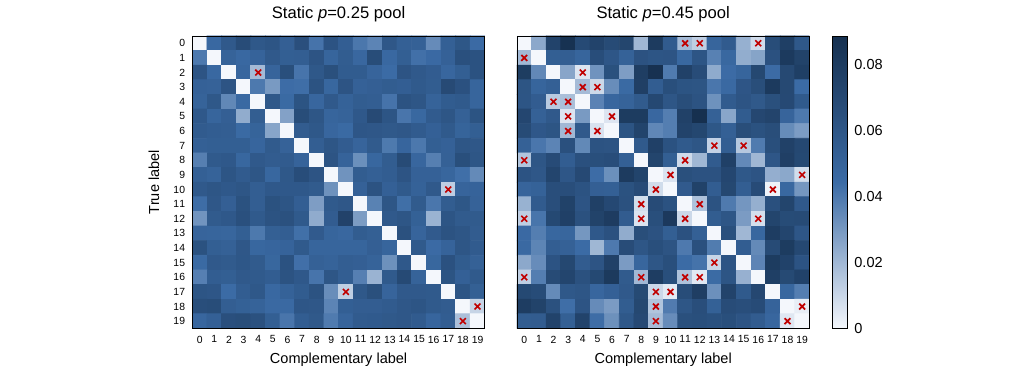}
\caption{Pruning-induced complementary-label bias on CIFAR-20 ($\eta=25\%$), as in Figure~\ref{fig:cll-bias}: row-normalized transitions $P(\bar{Y}\mid Y)$ of the CLL pool at $p{=}0.25$ (left) and $p{=}0.45$ (right); {\color{red}{$\times$}} marks near-empty transitions ($<2\%$ row mass; $5/380$ versus $38/380$ cells).}
\label{fig:cll-bias-cifar20}
\end{figure}

\subsection{Real-world Complementary Labels: CLCIFAR-10 and CLCIFAR-20}
\label{sec:ablation-clcifar}

CLCIFAR-10 and CLCIFAR-20 replace our synthetic corruption with \emph{human-annotated} complementary labels: each of the 50{,}000 training images carries one crowdsourced complementary label, of which $3.98\%$ (CLCIFAR-10; $1{,}990$ samples) and $2.32\%$ (CLCIFAR-20; $1{,}161$ samples) are annotation errors, in which the worker selected the \emph{true} class of the image as complementary. Two properties distinguish this regime from the synthetic sections. First, both error rates lie below $1/M$, so under the noisy-group label cost of TW-LAVA the effective label weight $\lambda_y\alpha(\eta)$ is negative (Theorem~\ref{thm:flip}) and annotation errors sit on the high-score side; the TW-LAVA rules therefore flag the high-score tail here, the reverse of the synthetic benchmarks, where diagonal noise concentrates at the validation-consistent low end. Under the label margin of CAVA the errors sit on the low side on every dataset (Theorem~\ref{thm:margin}), and the fixed cut flags the low tail as usual. Second, for TW-LAVA the errors are far harder to detect than synthetic flips (error-detection AUC $0.8790$ and $0.7243$ versus $0.9087$--$0.9994$ in the synthetic regimes), because human mistakes concentrate on genuinely ambiguous images, whereas the margin separates them almost completely (AUC $1.000$ and $0.999$, Figure~\ref{fig:realworld_compare}); the pruning ratios are matched to the small error rates ($p \in \{3\%, 5\%, 8\%, 10\%\}$).

For the real-world CLImage datasets, we report two CE-branch diagnostics.  Err. Prec. is the fraction of the recycled CE branch that consists of true annotation errors $(\bar y_i=y_i)$; it measures whether recycling is reliable.  Err. Rec. is the fraction of all annotation errors captured by the CE branch; it measures how much annotation noise is removed from the retained CLL pool. CLL Contam. is the fraction of annotation errors remaining in the retained complementary-label pool.

\begin{table*}[htb]
\centering
\caption{
CLCIFAR-10 ablation with real-world human complementary labels under AutoAug, for the baseline TW-LAVA and for CAVA with the same feature extractor. Static $p$ denotes the fixed cut; for CAVA, the adaptive cut $s_i\le0$ needs no ratio ($^\ast$: seed 1126). TW-LAVA rows use the original TW-LAVA protocol with $10{,}000$ validation images. Err. Prec. is the fraction of the recycled CE branch that consists of true annotation errors $(\bar y_i=y_i)$, while Err. Rec. is the fraction of all annotation errors captured by the CE branch. CLL Contam. is the fraction of annotation errors remaining in the retained complementary-label pool. The annotation record is used only for post-hoc evaluation of branch composition and, where indicated, to construct known-$Q$ FWD; neither method uses true labels for selection. TW-LAVA entries are single runs (seed 1126) that use all 10{,}000 test images as the validation set; CAVA entries are means over seeds 1126, 42, and 2023 with a fixed validation set of 2{,}000 test images (largest standard deviation 2.04), and its selection statistics are from seed 1126; the best accuracy in each column is in \textbf{bold}; ``--'': not applicable.
}
\label{tab:clcifar10-ablation}
\small
\setlength{\tabcolsep}{3.6pt}
\renewcommand{\arraystretch}{0.96}
\begin{tabular}{@{}lrrr|rr|cc|cc@{}}
\toprule
\multirow{2}{*}{Setting} 
& \multicolumn{3}{c}{CE branch} 
& \multicolumn{2}{c}{CLL pool} 
& \multicolumn{2}{c}{Combined} 
& \multicolumn{2}{c}{CLL-only} \\
\cmidrule(lr){2-4} \cmidrule(lr){5-6} \cmidrule(lr){7-8} \cmidrule(l){9-10}
& \textbf{Size} & \textbf{Err. Prec.} & \textbf{Err. Rec.} 
& \textbf{Size} & \textbf{Contam.} 
& \textbf{CE+SCL-NL} & \textbf{CE+FWD} 
& \textbf{SCL-NL} & \textbf{FWD} \\
\midrule
No pruning $(p=0)$   & --   & --   & --   & 50,000   & 3.98\%   & --   & --   & 50.68   & 51.76 \\
\midrule
\multicolumn{10}{@{}l}{\emph{TW-LAVA (noisy-group label cost)}} \\
Static $p=0.03$     & 1,500     & 32.1\%     & 24.2\%     & 48,500     & 3.11\%     & 46.78     & 46.19     & 53.25     & 59.38 \\
Static $p=0.05$     & 2,500     & 26.7\%     & 33.5\%     & 47,500     & 2.79\%     & 46.48     & 52.44     & 53.81     & 63.81 \\
Static $p=0.08$     & 4,000     & 21.7\%     & 43.7\%     & 46,000     & 2.44\%     & 43.65     & 46.78     & 50.68     & 67.29 \\
Static $p=0.10$     & 5,000     & 19.9\%     & 49.9\%     & 45,000     & 2.22\%     & 41.21     & 45.41     & 49.41     & 67.21 \\
\midrule
\multicolumn{10}{@{}l}{\emph{CAVA (ours; complementary label cost, label margin)}} \\
Static $p=0.03$  & 1,500  & 100.0\%  & 75.4\%  & 48,500  & 1.01\%  & 82.42  & 82.71  & 74.28  & 74.05 \\
Static $p=0.05$  & 2,500  & 79.6\%  & 100.0\%  & 47,500  & 0.00\%  & 83.46  & 83.20  & 80.45  & \textbf{75.99} \\
Static $p=0.08$  & 4,000  & 49.8\%  & 100.0\%  & 46,000  & 0.00\%  & 71.16  & 72.37  & 75.68  & 73.57 \\
Static $p=0.10$  & 5,000  & 39.8\%  & 100.0\%  & 45,000  & 0.00\%  & 63.54  & 65.46  & 65.76  & 67.14 \\
Adaptive cut $s_i\le0$  & 1,990  & 99.8\%  & 99.8\%  & 48,010  & 0.01\%  & \textbf{87.11}$^\ast$  & \textbf{85.25}$^\ast$  & \textbf{81.45}$^\ast$  & 75.10$^\ast$ \\
\bottomrule
\end{tabular}
\end{table*}


\begin{table*}[htb]
\centering
\caption{
CLCIFAR-20 ablation with real-world human complementary labels under AutoAug, for the baseline TW-LAVA and for CAVA with the same feature extractor. Static $p$ denotes the fixed cut; for CAVA, the adaptive cut $s_i\le0$ needs no ratio ($^\ast$: seed 1126). TW-LAVA rows use the original TW-LAVA protocol with $10{,}000$ validation images. Err. Prec. is the fraction of the recycled CE branch that consists of true annotation errors $(\bar y_i=y_i)$, while Err. Rec. is the fraction of all annotation errors captured by the CE branch. CLL Contam. is the fraction of annotation errors remaining in the retained complementary-label pool. The annotation record is used only for post-hoc evaluation of branch composition and, where indicated, to construct known-$Q$ FWD; neither method uses true labels for selection. TW-LAVA entries are single runs (seed 1126) that use all 10{,}000 test images as the validation set; CAVA entries are means over seeds 1126, 42, and 2023 with a fixed validation set of 2{,}000 test images (largest standard deviation 2.38), and its selection statistics are from seed 1126; the best accuracy in each column is in \textbf{bold}; ``--'': not applicable.
}
\label{tab:clcifar20-ablation}
\small
\setlength{\tabcolsep}{3.6pt}
\renewcommand{\arraystretch}{0.96}
\begin{tabular}{@{}lrrr|rr|cc|cc@{}}
\toprule
\multirow{2}{*}{Setting} 
& \multicolumn{3}{c}{CE branch} 
& \multicolumn{2}{c}{CLL pool} 
& \multicolumn{2}{c}{Combined} 
& \multicolumn{2}{c}{CLL-only} \\
\cmidrule(lr){2-4} \cmidrule(lr){5-6} \cmidrule(lr){7-8} \cmidrule(l){9-10}
& \textbf{Size} & \textbf{Err. Prec.} & \textbf{Err. Rec.} 
& \textbf{Size} & \textbf{Contam.} 
& \textbf{CE+SCL-NL} & \textbf{CE+FWD} 
& \textbf{SCL-NL} & \textbf{FWD} \\
\midrule
No pruning $(p=0)$   & --   & --   & --   & 50,000   & 2.32\%   & --   & --   & 8.98   & 23.73 \\
\midrule
\multicolumn{10}{@{}l}{\emph{TW-LAVA (noisy-group label cost)}} \\
Static $p=0.03$     & 1,500     & 7.8\%     & 10.1\%     & 48,500     & 2.15\%     & 9.77     & 16.75     & 8.98     & 25.15 \\
Static $p=0.05$     & 2,500     & 7.2\%     & 15.5\%     & 47,500     & 2.07\%     & 9.08     & 14.10     & 9.08     & 24.95 \\
Static $p=0.08$     & 4,000     & 7.7\%     & 26.5\%     & 46,000     & 1.85\%     & 10.30     & 16.02     & 9.38     & 25.44 \\
Static $p=0.10$     & 5,000     & 7.6\%     & 32.9\%     & 45,000     & 1.73\%     & 11.07     & 15.62     & 9.65     & 26.29 \\
\midrule
\multicolumn{10}{@{}l}{\emph{CAVA (ours; complementary label cost, label margin)}} \\
Static $p=0.03$  & 1,500  & 75.9\%  & 98.0\%  & 48,500  & 0.05\%  & 51.28  & \textbf{53.74}  & 12.11  & 32.88 \\
Static $p=0.05$  & 2,500  & 46.3\%  & 99.7\%  & 47,500  & 0.01\%  & 48.09  & 52.65  & 13.41  & 34.92 \\
Static $p=0.08$  & 4,000  & 29.0\%  & 99.9\%  & 46,000  & 0.00\%  & 44.50  & 48.88  & 14.95  & 37.71 \\
Static $p=0.10$  & 5,000  & 23.2\%  & 99.9\%  & 45,000  & 0.00\%  & 42.25  & 46.29  & \textbf{15.79}  & \textbf{41.78} \\
Adaptive cut $s_i\le0$  & 1,062  & 96.0\%  & 87.9\%  & 48,938  & 0.29\%  & \textbf{51.66}$^\ast$  & 53.56$^\ast$  & 11.72$^\ast$  & 30.54$^\ast$ \\
\bottomrule
\end{tabular}
\end{table*}

\paragraph{Detection quality drives everything downstream.}
Figure~\ref{fig:realworld_compare}(a) in the main text shows the score distributions. Under TW-LAVA, annotation errors ($\bar y_i=y_i$) lie on the \emph{high} side, the reverse of the synthetic setting (Figure~\ref{fig:ot-scores}), and overlap strongly with valid complementary labels (detection AUC $0.879$ on CLCIFAR-10 and $0.724$ on CLCIFAR-20); under CAVA they lie on the low side and are almost disjoint from the valid labels (AUC $1.000$ and $0.999$). For TW-LAVA, the two datasets place the detector at different operating points. On CLCIFAR-10, of the $1{,}500$ highest-scoring samples (the top $3\%$ of the OT ranking, i.e.\ the CE branch at $p{=}0.03$), $481$ samples ($32.1\%$) are actual annotation errors. Since errors make up only $3.98\%$ of the dataset, a random cut of the same size would contain ${\sim}60$; the detector concentrates them $8.1\times$ above the base rate, so even shallow static cuts remove a meaningful fraction of the errors (residual $3.98 \to 2.22\%$ at $p{=}0.10$) and every pruned FWD run beats the unpruned baseline ($51.76 \to 59.4$--$67.3$). On CLCIFAR-20 the detector is far weaker: the same $1{,}500$-sample cut contains only $117$ errors ($7.8\%$, a $3.4\times$ enrichment over the $2.32\%$ base rate), the residual barely moves ($2.32 \to 1.73\%$), and static pruning yields only $+1.4$ to $+2.6$ points over the baseline.

\paragraph{Recycling fails when detection precision is low.}
On synthetic data, the combined objective performs well because the recycled CE branch has sufficiently high precision. On real-world complementary labels, however, the base rate of annotation errors is only 2--4\%, so even a detector with useful recall can produce a low-precision recycled branch. Tables~\ref{tab:clcifar10-ablation} and~\ref{tab:clcifar20-ablation} show this clearly. For CLCIFAR-10, TW-LAVA at $p=10\%$ captures $49.9\%$ of all annotation errors, but only $19.9\%$ of its CE branch consists of true errors. For CLCIFAR-20 at the same cut, it captures $32.9\%$ of the errors with a CE precision of only $7.6\%$. Most of its recycled samples are thus valid complementary labels falsely treated as ordinary labels. Therefore, with the noisy-group score, real-world complementary labels call for discarding flagged samples rather than recycling them. The label margin changes this picture: at $p=3\%$ its recycled branch consists of $100\%$ true errors on CLCIFAR-10 and $75.9\%$ on CLCIFAR-20 (Figure~\ref{fig:realworld_compare}), which is why the recycling objectives gain the most in Table~\ref{tab:main_results}.

\paragraph{SCL-NL degrades with annotator bias.} 
The uniform-transition assumption of SCL-NL is violated here: the human complementary-label marginal is non-uniform: per-class counts range over $4{,}097$--$6{,}445$ (uniform $5{,}000$) on CLCIFAR-10 but $1{,}676$--$3{,}538$ (uniform $2{,}500$) on CLCIFAR-20. The milder CLCIFAR-10 bias leaves SCL-NL functional ($50.7\%$ unpruned, $53.8\%$ at the best static cut) while the stronger CLCIFAR-20 bias breaks it outright (${\sim}9\%$ at chance $5\%$, in every configuration). FWD, whose $Q$ models the biased transition explicitly, is unaffected by this violation and is the
only complementary loss that remains usable across both datasets.
\FloatBarrier

\subsection{Further Domains}
\label{app:newds}
Table~\ref{tab:newds} lists the results plotted in Figure~\ref{fig:newds}; the datasets are described in Appendix~\ref{clarifying_cifar10}.

\begin{table*}[htb]
\centering
\caption{Test accuracy (\%) on three further domains with synthetic diagonal noise (Figure~\ref{fig:newds}; datasets in Appendix~\ref{clarifying_cifar10}). The validation set holds $50$ test images per class ($200$ for Places365-10), and the fixed cut is $p=\eta$. Seed 1126; best per column and block in \textbf{bold}.}
\label{tab:newds}
\small
\setlength{\tabcolsep}{4.2pt}
\begin{tabular}{@{}ll ccc ccc ccc@{}}
\toprule
& & \multicolumn{3}{c}{\textbf{FOOD15}} & \multicolumn{3}{c}{\textbf{RESISC45-20}} & \multicolumn{3}{c}{\textbf{Places365-10}} \\
\cmidrule(lr){3-5} \cmidrule(lr){6-8} \cmidrule(l){9-11}
\textbf{Method} & \textbf{Selection} & $5\%$ & $15\%$ & $25\%$ & $5\%$ & $15\%$ & $25\%$ & $5\%$ & $15\%$ & $25\%$ \\
\midrule
\multicolumn{11}{@{}l}{\textit{(a) Selection only: SCL-NL on the retained pool}} \\
\midrule
No pruning & -- & 8.33 & 3.17 & 1.90 & 2.77 & 0.70 & 0.17 & 44.73 & 3.09 & 0.84 \\
TW-LAVA & Fixed $p$ & 7.60 & 7.20 & 13.00 & 3.23 & 13.57 & 24.33 & 48.39 & 46.04 & 56.94 \\
\rowcolor{green!12} CAVA (ours) & Fixed $p$ & \textbf{22.50} & \textbf{21.10} & \textbf{18.90} & 23.17 & 23.43 & \textbf{24.70} & \textbf{75.08} & \textbf{74.03} & \textbf{73.90} \\
\rowcolor{green!12} CAVA (ours) & Adaptive cut & 18.67 & 17.57 & 11.60 & \textbf{25.17} & \textbf{26.23} & 22.93 & 71.60 & 73.86 & 71.22 \\
\midrule
\multicolumn{11}{@{}l}{\textit{(b) Selection with recycling: CE+SCL-NL}} \\
\midrule
TW-LAVA & Fixed $p$ & 8.27 & 22.93 & 42.17 & 3.00 & 44.43 & 69.10 & 38.95 & 53.35 & 70.40 \\
\rowcolor{green!12} CAVA (ours) & Fixed $p$ & \textbf{47.33} & \textbf{60.43} & \textbf{63.83} & \textbf{71.93} & \textbf{82.93} & \textbf{87.30} & 80.39 & \textbf{82.16} & \textbf{84.14} \\
\rowcolor{green!12} CAVA (ours) & Adaptive cut & 44.17 & 58.20 & 60.20 & 68.70 & 82.83 & 86.03 & \textbf{80.49} & \textbf{82.16} & 83.45 \\
\bottomrule
\end{tabular}
\end{table*}

\subsection{Other Noise Rates}
\label{sec:noise_ratio_robustness}

The flip law makes the noise rate the decisive variable for TW-LAVA, so we vary it on CIFAR-10. Figure~\ref{fig:noise_ratio} reports CIFAR-10 results under $\eta\in\{0.15,0.35,0.45\}$ using AutoAugment, against the fixed cut $p$. 
Comparing the best cut of each method, CAVA is higher for both objectives at every noise ratio, by $2.5$ to $10.2$ points. At $\eta=35\%$ and $45\%$ the best cut of both methods is $p=\eta$; at $\eta=15\%$ it is $p=25\%$, larger than $\eta$. Table~\ref{tab:smallnoise} reports the complementary regime $\eta\le10\%$, where TW-LAVA is reversed or blind and CAVA is best at every noise ratio.

\begin{figure*}[htb]
\centering
\includegraphics{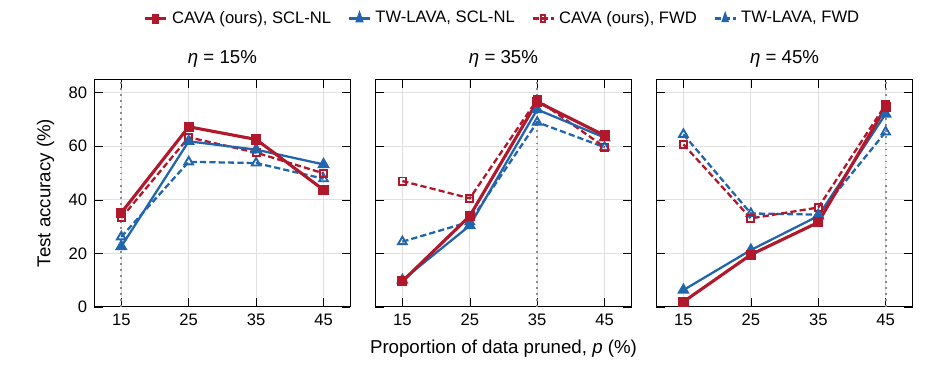}
\caption{\textbf{Other noise ratios.} CIFAR-10 test accuracy (\%) of TW-LAVA and CAVA against the fixed cut $p\in\{15\%,25\%,35\%,45\%\}$ at diagonal-noise ratios $\eta\in\{15\%,35\%,45\%\}$, above the crossover $1/M$, with the selection-only objective (SCL-NL solid, FWD dashed). Same feature extractor and validation set for both methods; AutoAugment; single runs (seed 1126) under the protocol of Table~\ref{tab:main_results}, accuracy after the final epoch on the held-out test images. The dotted line marks $p=\eta$.}
\label{fig:noise_ratio}
\end{figure*}

\begin{table*}[htb]
\centering
\caption{CIFAR-10 test accuracy (\%) at small noise rates, $\eta\le1/M=10\%$, where the label cost of TW-LAVA is reversed or blind (Theorem~\ref{thm:flip}). Both methods use the same fixed cuts; TW-LAVA keeps its original OT settings here (Appendix~\ref{sec:full_exp_details}). SCL-NL and No pruning: mean$\pm$std over three seeds; other objectives: seed 1126. Best per noise rate and objective in \textbf{bold}; ``--'': no recycled branch without pruning.}
\label{tab:smallnoise}
\small
\setlength{\tabcolsep}{4.2pt}
\begin{tabular}{@{}cc|cccc|cccc@{}}
\toprule
& & \multicolumn{4}{c|}{TW-LAVA (noisy-group label cost)} & \multicolumn{4}{c}{CAVA (ours, complementary label cost)} \\
$\eta$ & Cut & CE+SCL-NL & CE+FWD & SCL-NL & FWD & CE+SCL-NL & CE+FWD & SCL-NL & FWD \\
\midrule
\multirow{5}{*}{2\%} & none & -- & -- & \ms{68.53}{2.28} & 69.30 & -- & -- & \ms{68.53}{2.28} & 69.30 \\
& $p=2\%$ & 48.38 & 29.57 & \ms{67.22}{3.19} & 68.36 & \textbf{84.51} & \textbf{84.41} & \ms{82.26}{0.25} & \textbf{82.16} \\
& $p=5\%$ & 24.57 & 22.64 & \ms{62.38}{1.80} & 66.03 & 77.31 & 75.89 & \textbf{\ms{82.64}{0.11}} & 81.58 \\
& $p=10\%$ & 9.95 & 16.25 & \ms{46.66}{1.27} & 64.78 & 56.74 & 53.64 & \ms{70.78}{1.37} & 67.61 \\
& $p=15\%$ & 7.70 & 14.18 & \ms{11.25}{1.30} & 31.19 & 42.59 & 40.69 & \ms{47.84}{9.72} & 44.61 \\
\midrule
\multirow{5}{*}{5\%} & none & -- & -- & \ms{38.69}{0.85} & 38.94 & -- & -- & \ms{38.69}{0.85} & 38.94 \\
& $p=2\%$ & 18.71 & 2.84 & \ms{31.65}{1.67} & 33.51 & 66.12 & 56.27 & \ms{61.02}{1.02} & 55.23 \\
& $p=5\%$ & 6.33 & 1.86 & \ms{30.10}{1.38} & 31.91 & \textbf{86.46} & \textbf{86.25} & \textbf{\ms{82.15}{0.23}} & \textbf{82.05} \\
& $p=10\%$ & 2.95 & 2.51 & \ms{23.16}{3.82} & 37.92 & 72.53 & 70.53 & \ms{81.11}{0.19} & 71.50 \\
& $p=15\%$ & 3.41 & 2.59 & \ms{6.84}{0.72} & 18.15 & 58.26 & 54.84 & \ms{38.63}{5.02} & 54.52 \\
\midrule
\multirow{5}{*}{10\%} & none & -- & -- & \ms{7.85}{1.66} & 10.00 & -- & -- & \ms{7.85}{1.66} & 10.00 \\
& $p=2\%$ & 18.32 & 18.85 & \ms{11.20}{2.04} & 10.00 & 30.52 & 31.11 & \ms{25.39}{0.62} & 27.10 \\
& $p=5\%$ & 21.65 & 18.36 & \ms{10.72}{0.74} & 10.00 & 59.31 & 57.55 & \ms{47.27}{1.63} & 44.46 \\
& $p=10\%$ & 14.89 & 14.28 & \ms{12.42}{4.39} & 18.11 & \textbf{89.04} & \textbf{88.42} & \textbf{\ms{81.35}{0.48}} & \textbf{80.85} \\
& $p=15\%$ & 12.51 & 13.32 & \ms{11.40}{1.88} & 10.00 & 77.31 & 76.42 & \ms{36.52}{2.18} & 35.25 \\
\bottomrule
\end{tabular}
\end{table*}

\subsection{Other CLL Objectives}
\label{sec:cll_algorithms}
We examine whether the benefit of CAVA is tied to a specific CLL objective or carries over as a plug-in data-selection module. To isolate selection, recycling is disabled: each objective is trained only on the retained complementary-labeled pool (Eq.~\eqref{eq:cava_selection_objective}). We use two settings that cover both regimes of the flip law. On synthetic CIFAR-10 with $\eta=25\%$, above the crossover $1/M$, CAVA uses the adaptive cut $s_i\le0$, which needs no hyperparameter, and TW-LAVA the fixed cut $p=25\%$. On human-annotated CLCIFAR-10 with $\eta=3.98\%$, below the crossover, both TW-LAVA and CAVA use the fixed cut $p=5\%$, the smallest grid cut that covers the error rate. Both methods are scored with the same feature extractor and the same $2{,}000$ validation images, and No selection trains on the full pool.

\begin{figure*}[htb]
\centering
\includegraphics{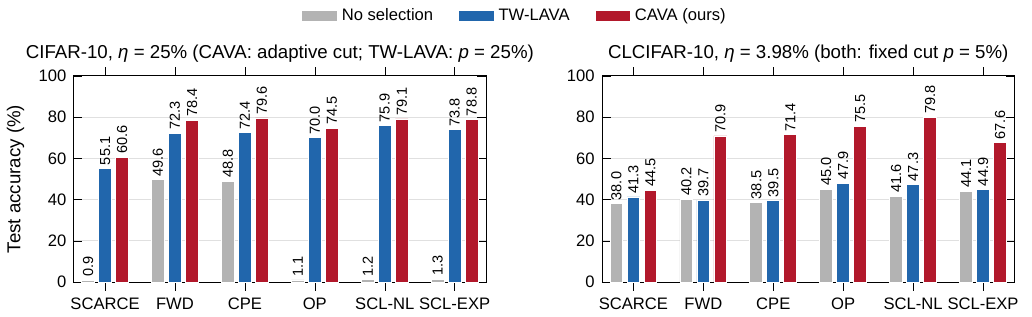}
\caption{\textbf{CAVA as a plug-in selection module across CLL objectives.} Test accuracy (\%) of six objectives without selection, with TW-LAVA, and with CAVA (selection only, no recycling; AutoAugment; seed 1126; final-epoch accuracy on the held-out test images, protocol of Table~\ref{tab:main_results}). All objectives use SGD as in the main experiments except OP, which uses Adam with learning rate $10^{-4}$ as in its official implementation. On synthetic CIFAR-10 ($\eta=25\%$), CAVA uses the adaptive cut and TW-LAVA, which has no adaptive cut, the fixed cut $p=25\%$; on human-annotated CLCIFAR-10 ($\eta=3.98\%$), both methods use the fixed cut $p=5\%$. TW-LAVA and CAVA use the same $2{,}000$ validation images.}
\label{fig:plugin}
\end{figure*}

Figure~\ref{fig:plugin} reports the test accuracy of each objective without selection, after TW-LAVA selection, and after CAVA selection. On CIFAR-10, every objective needs selection: without it, SCL-NL, SCL-EXP, OP, and SCARCE collapse ($1.18\%$, $1.26\%$, $1.06\%$, and $0.90\%$), and FWD and CPE reach only $49.64\%$ and $48.81\%$. With the adaptive cut, CAVA lifts SCL-NL, SCL-EXP, OP, and SCARCE to $79.09\%$, $78.81\%$, $74.54\%$, and $60.58\%$, and FWD and CPE to $78.56\%$ and $79.59\%$. TW-LAVA also helps here, where $\eta$ lies above the crossover and its score has the correct direction, but CAVA is higher for all six objectives by $3.2$ to $7.2$ points, without choosing a ratio. On CLCIFAR-10, below the crossover, TW-LAVA stays within $6$ points of No selection, whereas CAVA improves FWD, CPE, OP, SCL-NL, and SCL-EXP by $23.5$ to $38.2$ points and SCARCE by $6.5$ points.
Because recycling depends on detector precision, this analysis uses selection-only training; Appendix~\ref{combined_loss_function} studies when removed samples can be reused as ordinary-labeled data.

\subsection{Data Augmentation}
\label{app:augmentation}
AutoAugment is the default augmentation in all experiments. Figure~\ref{fig:aug} repeats the CIFAR-10 setting ($\eta=25\%$) with Cutout and RandAugment for TW-LAVA and for CAVA under the protocol of Table~\ref{tab:main_results}. The conclusions do not depend on the augmentation policy: among the fixed cuts both methods are best at $p=\eta$ for every objective and policy, and the adaptive cut of CAVA is within one point of that cut or above it; with recycling, CAVA reaches $91.4$ to $92.3\%$ under the three policies (best setting per policy), $3.3$ to $4.2$ points above the best TW-LAVA setting under the same policy and objective.

\begin{figure*}[htb]
\centering
\includegraphics{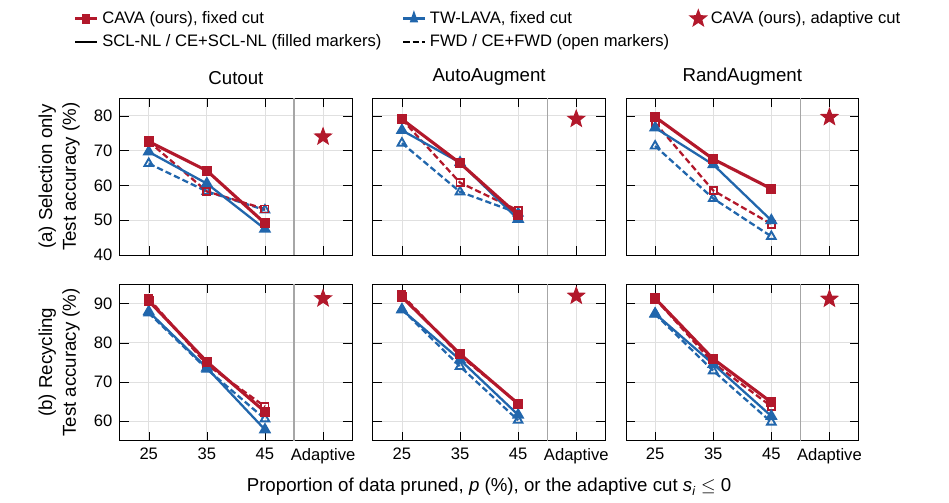}
\caption{\textbf{Data augmentation.} CIFAR-10 test accuracy (\%) at $\eta=25\%$ of TW-LAVA and CAVA against the fixed cut $p\in\{25\%,35\%,45\%\}$, and of CAVA with the adaptive cut $s_i\le0$ (stars right of the gray line, with SCL-NL and CE+SCL-NL), under Cutout, AutoAugment, and RandAugment. Row~(a) uses the selection-only objective in Eq.~\eqref{eq:cava_selection_objective} with SCL-NL (solid, filled markers) and FWD (dashed, open markers); row~(b) the recycling objective in Eq.~\eqref{eq:cava_mixed_objective} with CE+SCL-NL and CE+FWD. Same feature extractor, validation set and OT settings for both methods; single runs (seed 1126), accuracy after the final epoch on the held-out test images, protocol of Table~\ref{tab:main_results}.}
\label{fig:aug}
\end{figure*}

\subsection{Recycling and Branch Quality}
\label{combined_loss_function}
\label{post-hoc_analysis}
Recycling helps only when the flagged set is pure (Theorem~\ref{thm:recycle_truelabel}). This subsection measures that purity with the metrics of Table~\ref{tab:branch_metrics}.

\paragraph{CAVA.} Wherever the cut does not exceed the noise rate, the recycled branch of CAVA is nearly pure: it contains $0.54\%$ valid labels on CIFAR-10 at $p=\eta$ (Table~\ref{tab:detector-compare}), and the adaptive cut reaches $99.7\%$ and $96.0\%$ precision on CLCIFAR-10 and CLCIFAR-20 (Tables~\ref{tab:clcifar10-ablation} and~\ref{tab:clcifar20-ablation}). Recycling then turns the detected noise into supervision: on CIFAR-10 it raises SCL-NL from $79.26\%$ to $91.57\%$, and on CLCIFAR-20 from $10.78\%$ to $47.71\%$ with the fixed cut (Table~\ref{tab:main_results}).

\paragraph{TW-LAVA.} The baseline shows the failure mode that the bound predicts. On CIFAR-10 (Table~\ref{tab:branch_quality}), fixed cuts trade precision for recall: at $p=25\%$ the recycled branch is $95.4\%$ pure, and larger cuts recover more noise but recycle more valid labels, so accuracy falls from $88.46\%$ to $61.58\%$. On human labels its recycled branches contain only $9.3$ to $43.5\%$ true errors, and recycling lowers its accuracy on CLCIFAR-10 by $6$ to $13$ points; on CLCIFAR-20, where SCL-NL alone is near chance, it adds $3$ to $5$ points (Appendix~\ref{sec:ablation-clcifar}).

\paragraph{Pruning and recycling tolerate errors differently.} An extreme score is not proof of noise: atypical valid images, distribution shift between pool and validation set, or limited coverage of rare modes can also produce one. The two uses of a flagged sample are asymmetric. For pruning, a false positive costs one weak exclusion out of tens of thousands, whereas a retained diagonal label teaches a wrong exclusion, so pruning remains safe under imperfect detection. For recycling, a false positive injects a wrong ordinary label. Recycling is therefore an option gated by the measured precision, and selection only is the default.

\begin{table}[htb]
\centering
\caption{Branch quality of recycling with TW-LAVA on synthetic CIFAR-10 ($\eta=25\%$, CE+SCL-NL, seed 1126). The metrics of Table~\ref{tab:branch_metrics} use the true labels for evaluation only; the last column is the test accuracy.}
\label{tab:branch_quality}
\small
\begin{tabular}{l|cccccc|c}
\toprule
\textbf{Setting}
& \textbf{CE Precision}
& \textbf{CE Recall}
& \textbf{F1}
& $\mathbf{|I^{\mathrm{rec}}|}$
& \textbf{CE Contam.}
& \textbf{CLL Contam.}
& \textbf{Test Accuracy} \\
\midrule
Fixed cut $p=25\%$ 
& 95.37\% & 95.37\% & 95.37\% & 12,500 & 4.63\% & 1.54\% & 88.46\% \\
Fixed cut $p=35\%$ 
& 71.12\% & 99.57\% & 82.97\% & 17,500 & 28.88\% & 0.17\% & 75.67\% \\
Fixed cut $p=45\%$ 
& 55.46\% & 99.82\% & 71.30\% & 22,500 & 44.54\% & 0.08\% & 61.58\% \\
\bottomrule
\end{tabular}
\end{table}

\subsection{The Validation Set}
\label{app:reference_assumption}
This subsection examines the clean validation set from two sides: how much its labels are worth when spent on direct supervision instead, and how small it can be.

\paragraph{Valuation versus direct supervision.}

CAVA follows the standard protocol of OT-based data valuation~\citep{just2023lava,kesslersava}: a small clean validation set defines the target distribution against which training samples are valued. Every valuation baseline in our comparison (LAVA, SAVA) assumes the same validation set, so this assumption grants no advantage specific to CAVA. Moreover, the validation set enters only the OT scores; it is never used to train the classifier. Its cost is also one-time and independent of the pool size: the CIFAR-10-1M study in Appendix~\ref{large-scale_dataset} keeps the same $2{,}000$ validation examples while the training pool grows twentyfold.

A natural question remains: would the $2{,}000$ clean labels be more valuable if they were spent directly on supervised training rather than on valuation? To quantify this, we train a classifier with the standard CE loss on $n$ clean labels, stratified per class, for budgets $n\in\{500, 1{,}000, 2{,}000, 5{,}000, 10{,}000\}$ on CIFAR-10 and $n\in\{500, 1{,}000, 2{,}000, 4{,}000, 7{,}000\}$ on CIFAR-20. The protocol matches the main experiments exactly (ResNet-18, AutoAugment, SGD with the same schedule and weight decay, $200$ epochs, accuracy after the final epoch on the held-out test images, three runs per budget with seeds $1126$, $42$, and $2023$), except that the batch size is $\min(128, n)$ so that small budgets receive a comparable number of updates. This baseline faces no label noise at all, whereas CAVA operates under $25\%$ complementary-label noise.

\begin{table}[htb]
\centering
\caption{What $2{,}000$ clean labels are worth ($\eta=25\%$; mean $\pm$ std over three seeds). CE from scratch trains on these labels alone. CAVA uses them only as its validation set, together with the $50{,}000$ noisy complementary labels (CE+SCL-NL, $p=25\%$), on three fixed representations.}
\label{tab:reference_budget}
\small
\setlength{\tabcolsep}{5.0pt}
\begin{tabular}{@{}llcc@{}}
\toprule
\textbf{Representation} & \textbf{Method with $2{,}000$ clean labels} & CIFAR-10 & CIFAR-20 \\
\midrule
none & CE from scratch & \ms{71.52}{0.66} & \ms{43.08}{2.24} \\
\midrule
supervised embedder & CAVA (ours), recycling & \ms{91.57}{0.14} & \ms{75.34}{0.17} \\
CLIP ViT-B/32 & CAVA (ours), recycling & \ms{87.91}{0.55} & \ms{67.37}{0.24} \\
validation images only & CAVA (ours), recycling & \ms{75.22}{0.67} & \ms{49.21}{0.85} \\
\bottomrule
\end{tabular}
\end{table}

Table~\ref{tab:reference_budget} and Figure~\ref{fig:scale_budget}(b) give the answer. Against training from scratch, the validation set is worth more as an anchor: with $2{,}000$ labels direct supervision reaches $71.52\%$ on CIFAR-10 and $43.08\%$ on CIFAR-20, CAVA $91.57\%$ and $75.34\%$, and direct supervision does not reach this level even with $10{,}000$ and $7{,}000$ labels ($88.79\%$ and $68.09\%$). The table also checks that CAVA does not depend on one particular fixed representation. The supervised embedder of the main tables has itself been trained on the $50{,}000$ clean training labels (Appendix~\ref{sec:full_exp_details}); with CLIP features, which have seen no label of the pool, CAVA reaches $87.91\%$ and $67.37\%$, still far above training from scratch, and with the network trained on the $2{,}000$ validation images only, the weakest of the three representations, it reaches $75.22\%$ and $49.21\%$, still above training from scratch on the same labels (Appendix~\ref{app:representations}). Table~\ref{tab:nv_coverage} adds that, with the supervised embedder, CAVA keeps its level with as few as $200$ validation images; on label-free features the detection AUC on CIFAR-10 falls from $0.994$ at $N_v=2{,}000$ to $0.978$ at $N_v=50$ (Appendix~\ref{app:representations}).

\paragraph{Scope and fairness of the comparison.}
CAVA is a valuation and denoising method for CLL: the question it answers is which samples of a given noisy complementary-label pool should be kept, recycled, or removed. Every compared method answers this same question under the same validation set, so the comparison is information-equal across the table. Holding out a small clean set for scores, thresholds, and hyperparameter selection, without using it for training, is standard practice throughout the learning-with-noisy-labels literature~\citep{han2018co,li2020dividemix,frenay2013classification} and is inherited from the data valuation line~\citep{ghorbani2019data,just2023lava,kesslersava}; every baseline with tunable hyperparameters implicitly relies on such a set. Methods that instead consume the validation set as training data, whether by direct supervision or by semi-supervised learning over the pool images, solve a different problem and discard the complementary labels that are the object of study; the budget study above therefore measures a fixed clean-annotation budget rather than information equivalence, and Table~\ref{tab:reference_budget} bounds what direct supervised use achieves at the same budget. Such methods are pipeline alternatives rather than competing valuation methods: any consumer of the validation set can be composed with the CAVA output, which supplies a cleaned complementary-label pool and a recycled ordinary-label set on top, and the validation set itself can also be added to the CE branch of the recycling objective in Eq.~\eqref{eq:cava_mixed_objective}.

\paragraph{Validation-set size.}
We shrink the validation set on both synthetic benchmarks. For each $N_v\in\{200, 500, 1{,}000, 1{,}500, 2{,}000\}$ we draw a class-stratified subset of the test split ($N_v/10$ per class on CIFAR-10, $N_v/20$ per class on CIFAR-20, seed $1126$), recompute the task-wise OT scores against it, and rerun the full fixed-cut pipeline ($p=0.25$, combined CE+SCL-NL with $\kappa=0.5$, AutoAugment; three full-pipeline seeds per setting). Table~\ref{tab:nv_coverage} reports the detector AUC, the precision of the recycled branch, and the final test accuracy for TW-LAVA and CAVA. The two methods react differently. TW-LAVA degrades as the validation set shrinks: from $N_v=2{,}000$ to $200$ the precision of its recycled branch falls by $4$ points on CIFAR-10 and $6$ points on CIFAR-20, and its accuracy by $3.8$ and $6.7$ points. CAVA is almost insensitive to $N_v$: its precision stays above $99\%$ and $93\%$, and its accuracy within $0.2$ points on CIFAR-10 and $2.1$ points on CIFAR-20 of the $N_v=2{,}000$ result. With only $200$ validation images, CAVA reaches $91.47\%$ on CIFAR-10 and $73.21\%$ on CIFAR-20, above TW-LAVA with ten times more validation data ($88.75\%$ and $59.03\%$). These numbers use the supervised embedder of the main tables; on CLIP features the detection AUC of CAVA falls from $0.994$ at $N_v=2{,}000$ to $0.978$ at $N_v=50$ (Appendix~\ref{app:representations}). The label margin needs only enough validation images to place each class, not to estimate the feature distribution of every noisy label group.

\begin{table}[htb]
\centering
\caption{Effect of the validation-set size $N_v$ on CIFAR-10 and CIFAR-20 ($\eta=25\%$, fixed cut $p=25\%$, CE+SCL-NL). Prec.\ is the share of true diagonal noise in the recycled branch. Mean $\pm$ std over three seeds; best per row in \textbf{bold}, ties broken by the unrounded value.}
\label{tab:nv_coverage}
\small
\setlength{\tabcolsep}{5pt}
\begin{tabular}{@{}lllccccc@{}}
\toprule
Dataset & Method & $N_v$ & 200 & 500 & 1{,}000 & 1{,}500 & 2{,}000 \\
\midrule
\multirow{6}{*}{CIFAR-10} & \multirow{3}{*}{TW-LAVA} & Detector AUC & \ms{0.98}{0.00} & \ms{0.99}{0.00} & \ms{0.99}{0.00} & \ms{1.00}{0.00} & \msbest{1.00}{0.00} \\
 & & CE-branch Prec.\ (\%) & \ms{91.52}{0.58} & \ms{92.99}{0.29} & \ms{94.21}{0.23} & \ms{95.27}{0.18} & \msbest{95.46}{0.09} \\
 & & Accuracy (\%) & \ms{84.94}{0.06} & \ms{86.29}{0.31} & \ms{87.82}{0.25} & \ms{88.64}{0.07} & \msbest{88.75}{0.40} \\
\cmidrule(l){2-8}
 & \multirow{3}{*}{CAVA (ours)} & Detector AUC & \ms{1.00}{0.00} & \ms{1.00}{0.00} & \ms{1.00}{0.00} & \ms{1.00}{0.00} & \msbest{1.00}{0.00} \\
 & & CE-branch Prec.\ (\%) & \ms{99.23}{0.11} & \ms{99.29}{0.14} & \ms{99.32}{0.13} & \msbest{99.34}{0.14} & \ms{99.33}{0.15} \\
 & & Accuracy (\%) & \ms{91.47}{0.16} & \msbest{91.69}{0.17} & \ms{91.51}{0.17} & \ms{91.69}{0.46} & \ms{91.57}{0.14} \\
\multirow{6}{*}{CIFAR-20} & \multirow{3}{*}{TW-LAVA} & Detector AUC & \ms{0.90}{0.00} & \ms{0.92}{0.00} & \ms{0.93}{0.00} & \ms{0.94}{0.00} & \msbest{0.94}{0.00} \\
 & & CE-branch Prec.\ (\%) & \ms{72.32}{0.31} & \ms{75.04}{0.18} & \ms{76.77}{0.22} & \ms{78.00}{0.04} & \msbest{78.54}{0.18} \\
 & & Accuracy (\%) & \ms{52.29}{0.56} & \ms{55.67}{1.07} & \ms{57.49}{0.75} & \ms{58.18}{0.79} & \msbest{59.03}{0.77} \\
\cmidrule(l){2-8}
 & \multirow{3}{*}{CAVA (ours)} & Detector AUC & \ms{0.99}{0.00} & \ms{1.00}{0.00} & \ms{1.00}{0.00} & \ms{1.00}{0.00} & \msbest{1.00}{0.00} \\
 & & CE-branch Prec.\ (\%) & \ms{93.11}{0.04} & \ms{94.30}{0.17} & \ms{94.62}{0.14} & \ms{95.31}{0.09} & \msbest{95.52}{0.15} \\
 & & Accuracy (\%) & \ms{73.21}{0.49} & \ms{74.13}{0.56} & \ms{74.32}{0.54} & \ms{74.79}{0.34} & \msbest{75.34}{0.17} \\
\bottomrule
\end{tabular}
\end{table}


\subsection{Scaling to One Million Samples}
\label{large-scale_dataset}

To further evaluate the scalability of CAVA with respect to the number of training samples, we conduct an additional large-sample ablation study on CIFAR-10 under $25\%$ diagonal complementary-label noise ($\eta=25\%$). Specifically, we expand the original CIFAR-10 training set from 50{,}000 images to 1{,}000{,}000 training samples using label-preserving data augmentation, denoted as CIFAR-10-1M.\footnote{CIFAR-10-1M is an augmented benchmark generated from CIFAR-10 using label-preserving augmentation, increasing the training set from 50K to 1M samples. Therefore, it evaluates scalability with respect to training-set size rather than naturally collected semantic diversity.} After augmentation, complementary labels are generated following the same diagonal-noise protocol as in the main experiments, with noise ratio $\eta=25\%$. The original CIFAR-10 test set is kept unchanged for evaluation. This design allows us to assess whether CAVA remains effective under a substantially enlarged noisy complementary-label training regime without introducing test-set leakage or changing the evaluation protocol.

Figure~\ref{fig:scale_budget}(a) in the main text visualizes the accuracy--efficiency trade-off on CIFAR-10-1M, and Table~\ref{tab:cifar10_1m_scalability} reports the corresponding numerical results. Without pruning, SCL-NL fails on the enlarged noisy pool and reaches only $0.12\%$. TW-LAVA peaks at the cut $p=35\%$, above the noise rate, with $71.05\%$ after 200 epochs ($68.26\%$ at $p=25\%$ and $70.28\%$ at $p=45\%$) and $67.76\%$ under the 20-epoch schedule of CAVA ($61.01\%$ and $55.95\%$ at $p=25\%$ and $45\%$), so the gap to CAVA does not come from the training budget. CAVA ranks the noise at the low end of each task (detection AUC $0.996$) and peaks at the cut $p=25\%$ that equals the noise rate, where it reaches $86.04\%$ with SCL-NL after only 20 epochs; deeper cuts remove valid samples and lower accuracy ($78.05\%$ and $67.30\%$ at $p=35\%$ and $45\%$ after 20 epochs, $84.95\%$ and $73.25\%$ after 200 epochs, against $90.51\%$ at $p=25\%$). Cuts below the noise rate leave diagonal noise in the retained pool and fail for both methods: at $p=15\%$ TW-LAVA reaches $27.79\%$ and $42.90\%$ after 20 and 200 epochs and CAVA $24.99\%$ after 20 epochs, while its 200-epoch run collapses to a single class ($10.00\%$) from epoch 38 on and does not recover.

The computational results show that the task-wise decomposition keeps both methods feasible at this scale, and that the label margin makes valuation nearly free. The valuation is decomposed into five tasks of 200{,}000 samples each: TW-LAVA needs about 30 minutes per task, about 2.5 hours in total, whereas CAVA scores all one million samples in about 4 minutes, mainly because the complementary label cost needs no label-to-label transport between noisy groups. Training is the dominant cost, and it depends only on the size of the retained pool: on one dedicated GPU an epoch takes 566 seconds without pruning and 436 and 432 seconds at $p=25\%$ for TW-LAVA and CAVA (Table~\ref{tab:cifar10_1m_scalability}). These are the times of the second epoch, without the one-time start-up of about 20 seconds; the two methods agree within $4\%$ at the same cut, and the CAVA times at $p=15\%$ come from a linear fit in the pool size whose residuals at the measured cuts are below $2\%$. TW-LAVA at its best cut ($p=35\%$) needs $20.7$ hours for 200 epochs, whereas CAVA reaches a higher accuracy, $86.04\%$, with the 20-epoch schedule in $2.4$ hours, a tenfold smaller training budget; with 200 epochs CAVA reaches $90.51\%$ ($24.0$ hours).

\begin{table}[htb]
\centering
\caption{
Scaling to CIFAR-10-1M ($\eta=25\%$, SCL-NL, single run). Training time is the measured time of one epoch on a dedicated NVIDIA RTX A6000, multiplied by 20 or 200; it depends only on the size of the retained pool. Valuation adds about $2.5$ hours for TW-LAVA and $4$ minutes for CAVA. The CAVA times at $p=15\%$ are estimated by a linear fit in the pool size. The best-epoch columns are for reference only; the text uses the final epoch. Best per column in \textbf{bold}; ``--'': not run.
}
\label{tab:cifar10_1m_scalability}
\small
\setlength{\tabcolsep}{5pt}
\begin{tabular}{@{}lc cc cc cc@{}}
\toprule
& & \multicolumn{2}{c}{Training time} & \multicolumn{2}{c}{Acc. (\%), final epoch} & \multicolumn{2}{c}{Acc. (\%), best epoch} \\
\cmidrule(lr){3-4} \cmidrule(lr){5-6} \cmidrule(l){7-8}
\textbf{Method} & $p$ (\%) & 20 ep. & 200 ep. & 20 ep. & 200 ep. & 20 ep. & 200 ep. \\
\midrule
No pruning & 0 & 3.1 h & 31.5 h & 0.12 & -- & 2.52 & -- \\
\midrule
\multirow{4}{*}{TW-LAVA} & 15 & 2.7 h & 27.4 h & 27.79 & 42.90 & 27.96 & 43.66 \\
 & 25 & 2.4 h & 24.2 h & 61.01 & 68.26 & 61.01 & 69.64 \\
 & 35 & 2.1 h & 20.7 h & 67.76 & 71.05 & 67.80 & 72.40 \\
 & 45 & 1.8 h & 17.7 h & 55.95 & 70.28 & 56.00 & 71.14 \\
\midrule
\multirow{4}{*}{CAVA (ours)} & \cellcolor{green!12}15 & \cellcolor{green!12}2.7 h & \cellcolor{green!12}27.2 h & \cellcolor{green!12}24.99 & \cellcolor{green!12}10.00 & \cellcolor{green!12}24.99 & \cellcolor{green!12}21.02 \\
 & \cellcolor{green!12}25 & \cellcolor{green!12}2.4 h & \cellcolor{green!12}24.0 h & \cellcolor{green!12}\textbf{86.04} & \cellcolor{green!12}\textbf{90.51} & \cellcolor{green!12}\textbf{86.04} & \cellcolor{green!12}\textbf{91.19} \\
 & \cellcolor{green!12}35 & \cellcolor{green!12}2.0 h & \cellcolor{green!12}20.0 h & \cellcolor{green!12}78.05 & \cellcolor{green!12}84.95 & \cellcolor{green!12}78.09 & \cellcolor{green!12}86.70 \\
 & \cellcolor{green!12}45 & \cellcolor{green!12}1.7 h & \cellcolor{green!12}17.2 h & \cellcolor{green!12}67.30 & \cellcolor{green!12}73.25 & \cellcolor{green!12}67.30 & \cellcolor{green!12}73.42 \\
\bottomrule
\end{tabular}
\end{table}

Overall, these results indicate that CAVA scales favorably with increased training-set size through task-wise decomposition. However, because CIFAR-10-1M is generated through label-preserving augmentation rather than collected as a naturally large and semantically diverse dataset, we view evaluation on naturally large-scale CLL datasets as an important direction for future work.

\subsection{Computational Efficiency of Task-Wise Decomposition}
\label{app:computational_efficiency}

The Sinkhorn algorithm can become computationally expensive when OT is computed over the full training and validation sets. To examine whether the proposed task-wise decomposition improves computational feasibility, we conduct an additional efficiency comparison among LAVA, SAVA, and CAVA on CIFAR-10. Specifically, we use 10K training samples and 10K validation samples and report the wall-clock runtime required for one OT-based data-selection step.

\begin{table}[htb]
\centering
\caption{Runtime of one OT-based selection step on CIFAR-10 with 10K training and 10K validation samples. The CAVA time is the mean of three runs on one GPU, with features computed beforehand.}
\label{tab:runtime_comparison}
\small
\resizebox{\textwidth}{!}{
\begin{tabular}{lcccl}
\toprule
\textbf{Method} 
& \textbf{Training Size} 
& \textbf{Validation Size} 
& \textbf{Data-Selection Runtime} 
& \textbf{Remark} \\
\midrule
LAVA 
& 10K 
& 10K 
& 45 seconds 
& Full-set OT computation \\

SAVA 
& 10K 
& 10K 
& 300 seconds 
& Batch-wise OT computation \\

TW-LAVA (task-wise) 
& 10K 
& 10K 
& 45 seconds 
& One task-wise OT selection step \\

CAVA (ours) 
& 10K 
& 10K 
& 0.2 seconds 
& OT solve and label margin on precomputed features \\
\bottomrule
\end{tabular}
}
\end{table}

As shown in Table~\ref{tab:runtime_comparison}, LAVA and task-wise selection with one solve per task require a comparable runtime in the 10K-sample setting because both solve an OT problem of similar size. CAVA solves a problem of the same size in $0.2$ seconds on precomputed features; its label cost is a fixed $0/1$ table and needs no OT between label groups. However, the key difference lies in their scalability. LAVA performs OT over the full data selection problem, and therefore its computational and memory costs increase rapidly as the training set becomes larger. In contrast, CAVA keeps the OT problem size fixed by decomposing the enlarged training set into task-level subproblems. We further compare the peak memory usage of different methods as the training-set size increases. The results are reported in Table~\ref{tab:memory_comparison}.

\begin{table}[htb]
\centering
\caption{Peak GPU memory of OT-based selection as the training set grows. LAVA, SAVA, and TW-LAVA run their original pipelines with $10{,}000$ validation images (TW-LAVA with tasks of 10K samples). For CAVA we report its valuation step, without feature extraction, with the $2{,}000$ validation images of the paper and with $10{,}000$ for a direct comparison. OOM: out of memory; ``--'': not measured.}
\small
\label{tab:memory_comparison}
\setlength{\tabcolsep}{4pt}
\begin{tabular}{@{}lccccc@{}}
\toprule
& & & & \multicolumn{2}{c}{\textbf{CAVA (ours)}} \\
\cmidrule(l){5-6}
\textbf{Training set} & \textbf{LAVA} & \textbf{SAVA} & \textbf{TW-LAVA} & $N_v=2{,}000$ & $N_v=10{,}000$ \\
\midrule
10K & 21.65 GB & 3.00 GB & 21.65 GB & \textbf{0.10 GB} & 0.35 GB \\
20K & 43.00 GB & 3.00 GB & 21.65 GB & \textbf{0.17 GB} & 0.65 GB \\
30K & OOM      & 3.00 GB & 21.65 GB & \textbf{0.23 GB} & 0.95 GB \\
40K & OOM      & 3.00 GB & 21.65 GB & \textbf{0.30 GB} & 1.25 GB \\
50K & OOM      & 3.00 GB & 21.65 GB & \textbf{0.35 GB} & 1.55 GB \\
1M  & OOM      & --      & --       & \textbf{0.42 GB} & 1.86 GB \\
\bottomrule
\end{tabular}
\end{table}

Table~\ref{tab:memory_comparison} shows that LAVA becomes increasingly memory-intensive as the training-set size grows. In particular, its memory usage reaches 43.00 GB at 20K training samples and runs out of memory from 30K samples onward in our environment. This behavior is expected because full-set OT requires constructing and optimizing over a large transport problem whose size grows with the number of training and validation samples.

SAVA, by contrast, maintains low memory usage of approximately 3.00 GB because it performs OT at the batch level. However, this memory advantage comes with substantially longer runtime in our experiment, as shown in Table~\ref{tab:runtime_comparison}. Moreover, batch-wise selection may restrict the comparison scope of OT to small local batches.

TW-LAVA, the task-wise transfer of LAVA, avoids the out-of-memory failure by fixing the task size, but its original pipeline still needs about 21.65 GB. CAVA needs far less: its complementary label cost is a lookup table, so the valuation step stores only the task-by-validation cost matrix and a few vectors. At the paper protocol of $2{,}000$ validation images it uses $0.10$ to $0.35$ GB for 10K to 50K training samples, and $0.42$ GB for one million samples, where pools larger than $12{,}000$ samples are scored in chunks; even with $10{,}000$ validation images it stays below $2$ GB. Memory therefore grows with the validation set and the chunk size, and with the training set only until the pools are scored in chunks, and the task-wise subproblems remain independent, so they can be run in parallel across devices without changing the method.

\FloatBarrier
\section{Visualizing the Label Margin}
\label{app:margin_viz}
Figures~\ref{fig:margin_samples_clcifar10} and~\ref{fig:margin_samples_places10} show how CAVA decides on individual training samples, in the style of the retrieval visualization of~\citet{long2022retrieval}. Each figure shows eight samples drawn at random from the first task, half of them annotation errors ($\bar y_i=y_i$) and half valid complementary labels, with the OT solution of the main experiments. Each row gives the image (red border: annotation error), the labels of the six validation images with the largest transport weight from this sample, with these weights (green: the true class), the label potentials $f_i(c)$ relative to their minimum for the six cheapest classes (the given label $\bar y_i$ in red), the margins of all samples in the task with this sample marked and the threshold $s_i=0$ of the adaptive cut drawn as a dashed line, and the decision of the adaptive cut.

The rows make the mechanism of Theorem~\ref{thm:margin} visible. The validation images that receive the transport mass of a sample almost always carry its true class. For an annotation error, the given label is that class, so $f_i(\bar y_i)$ is the smallest potential, the margin $s_i=f_i(\bar y_i)-\min_{c\ne\bar y_i}f_i(c)$ is negative, and the sample is flagged. For a valid complementary label, the given label names another class, its potential lies above the minimum, and the sample is kept. The examples are not selected for success: on Places365-10 (Figure~\ref{fig:margin_samples_places10}), one annotation error labeled ``bridge'' shows mostly a forest stream, its transport mass goes to forest-path images, and the adaptive cut keeps it. Such failures are rare, as the margin histograms show: the two groups overlap only in a thin band around the threshold $s_i=0$.

\begin{figure*}[p]
\centering
\includegraphics[width=\textwidth, height=0.88\textheight, keepaspectratio]{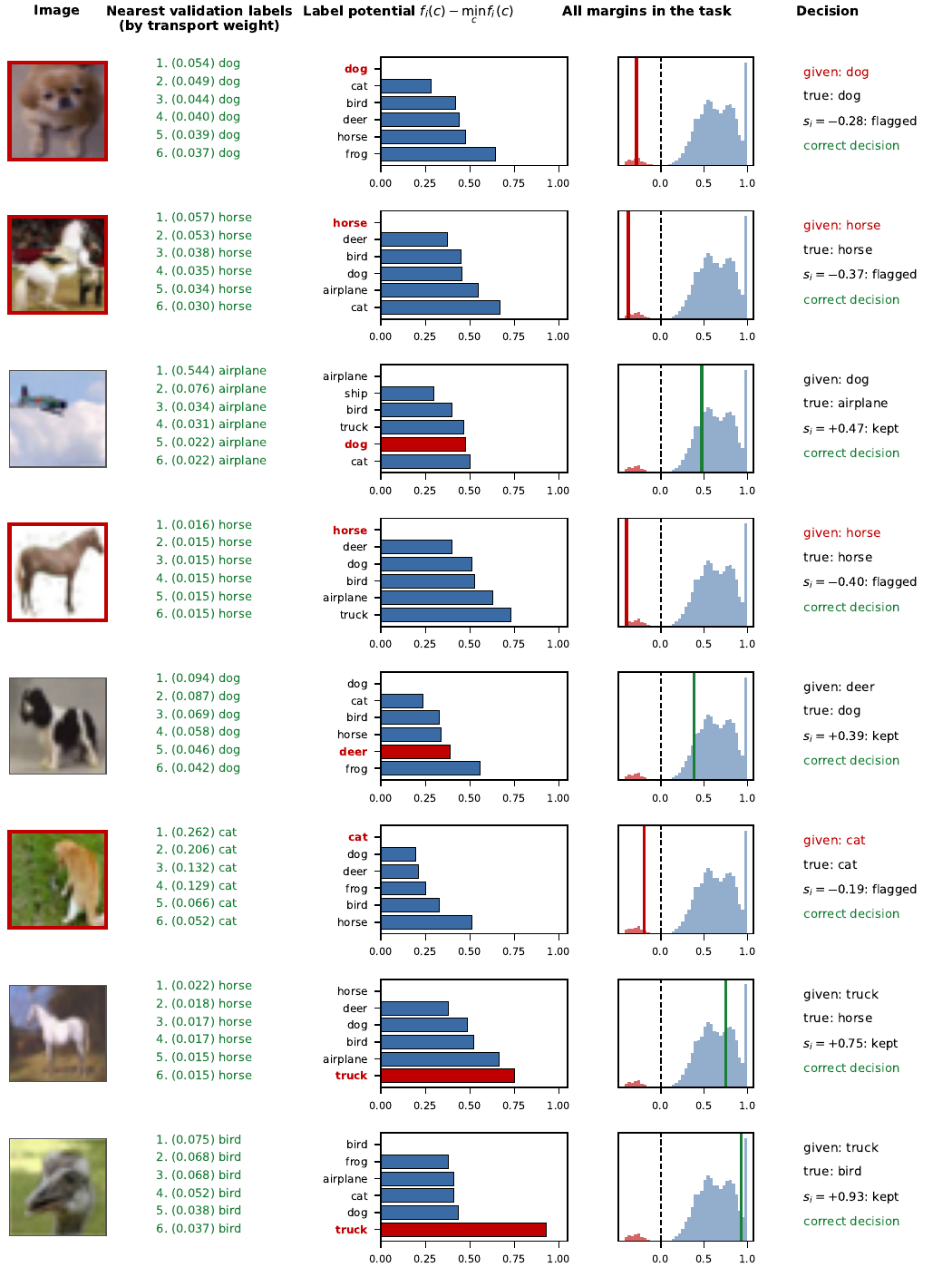}
\caption{Label-margin visualization for randomly drawn samples from CLCIFAR-10 (human complementary labels, $3.98\%$ annotation errors; first task, $2{,}000$ validation images). Columns: image (red border: annotation error), labels of the validation images with the largest transport weight, with the weights (green: true class), label potentials $f_i(c)-\min_c f_i(c)$ with the given label in red, margins of all samples in the task (blue: valid, red: annotation error; dashed: adaptive-cut threshold $s_i=0$) with this sample marked, and the decision of the adaptive cut.}
\label{fig:margin_samples_clcifar10}
\end{figure*}

\begin{figure*}[p]
\centering
\includegraphics[width=\textwidth, height=0.88\textheight, keepaspectratio]{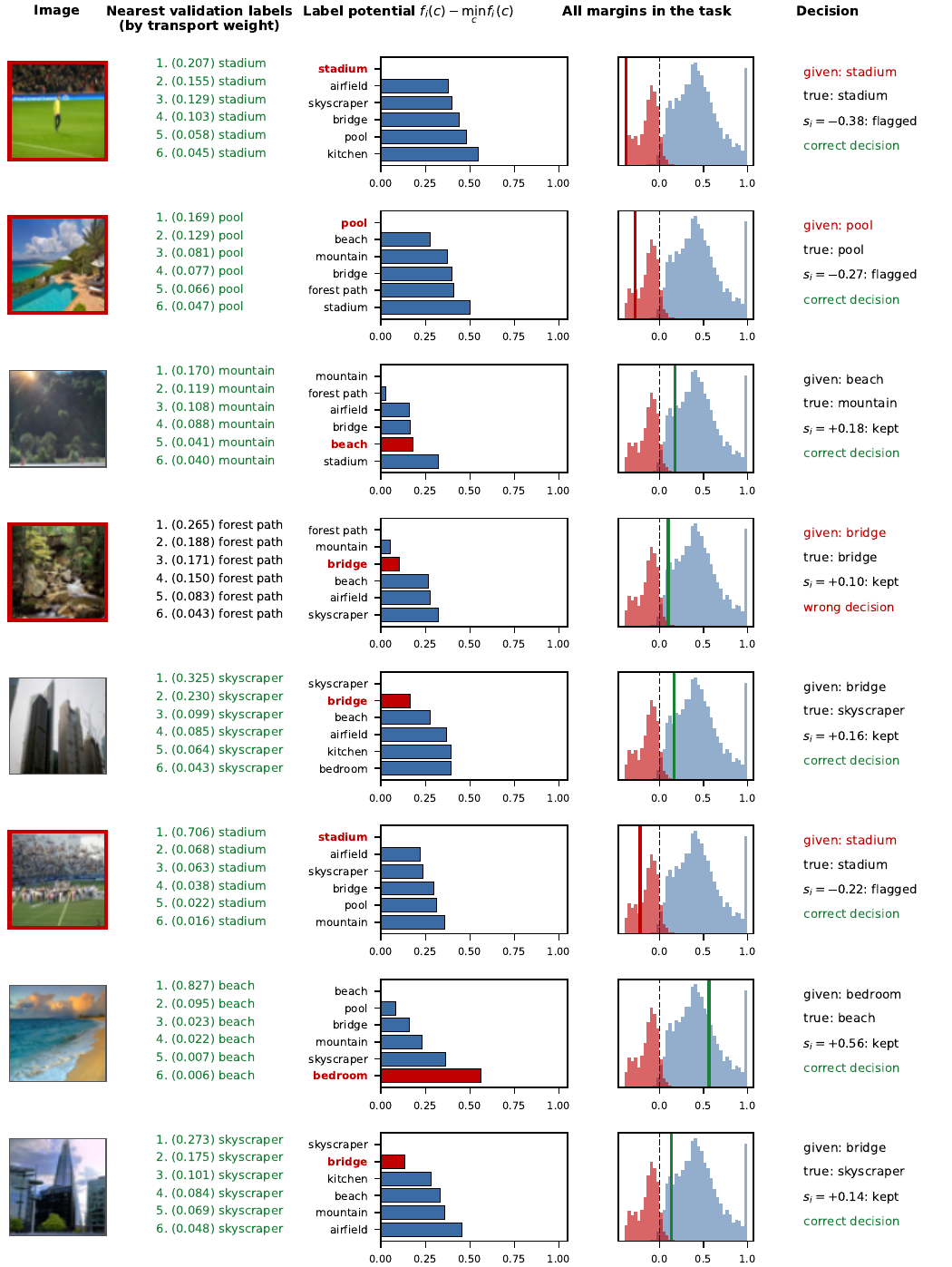}
\caption{Label-margin visualization for randomly drawn samples from Places365-10, ten scene classes of Places365~\citep{zhou2018places} with synthetic diagonal noise ($\eta=25\%$; first task, $2{,}000$ validation images). Columns as in Figure~\ref{fig:margin_samples_clcifar10}. The fourth row is a failure: an annotation error whose image resembles a forest path is kept.}
\label{fig:margin_samples_places10}
\end{figure*}

\end{document}